\documentclass[twoside]{article}
\usepackage[accepted]{aistats2025}

\usepackage[dvipsnames]{xcolor}
\usepackage{comment}
\usepackage[authoryear]{natbib}

\allowdisplaybreaks[1]

\usepackage{eqnarray,amsmath}
\usepackage{booktabs}       
\usepackage{nicefrac}       

\usepackage{ragged2e}
\usepackage{subfig}
\usepackage{epsf}
\usepackage{epsfig}
\usepackage{fancyhdr}
\usepackage{graphics}
\usepackage{graphicx}
\usepackage{psfrag}

\usepackage{enumerate}   
\usepackage{multirow}
\usepackage{bm}

\usepackage{mathtools}

\usepackage{url}
\usepackage[colorlinks,linkcolor=magenta,citecolor=blue, pagebackref=true]{hyperref}
\renewcommand*{\backrefalt}[4]{%
    \ifcase #1 \footnotesize{(Not cited.)}%
    \or        \footnotesize{(Cited on page~#2.)}%
    \else      \footnotesize{(Cited on pages~#2.)}%
    \fi}

\usepackage{color}

\usepackage{amsthm}
\usepackage{amsmath}
\usepackage{amssymb,bbm}
\usepackage{caption}
\usepackage{algorithmic}
\usepackage{algorithm}
\usepackage{textcomp}
\usepackage{siunitx}
\usepackage{wrapfig}
\usepackage{algorithmic}
\usepackage{algorithm}
\usepackage{multirow}
\usepackage{multicol}
\usepackage{mathtools}

\def\1{\bm{1}}

\DeclareMathAlphabet{\mathsfit}{\encodingdefault}{\sfdefault}{m}{sl}
\SetMathAlphabet{\mathsfit}{bold}{\encodingdefault}{\sfdefault}{bx}{n}

\newcommand{\E}{\mathbb{E}}

\newcommand{\lambdan}{\lambda_n}
\newcommand{\lambdasn}{\lambda_n^*}
\newcommand{\lambdaone}{\lambda_1}
\newcommand{\lambdatwo}{\lambda_2}
\newcommand{\lambdaonen}{\lambda_{1,n}}
\newcommand{\lambdatwon}{\lambda_{2,n}}

\newcommand{\Gn}{G_{n}}
\newcommand{\Gsn}{G_{*,n}}
\newcommand{\bigO}{\mathcal{O}}
\newcommand{\an}{a_{n}}
\newcommand{\asn}{a^{*}_{n}}
\newcommand{\bn}{b_{n}}
\newcommand{\bsn}{b^{*}_{n}}
\newcommand{\nun}{\nu_{n}}
\newcommand{\nusn}{\nu^{*}_{n}}
\newcommand{\dan}{\Delta a_{n}}
\newcommand{\dasn}{\Delta a^{*}_{n}}
\newcommand{\dbn}{\Delta b_{n}}
\newcommand{\dbsn}{\Delta b^{*}_{n}}
\newcommand{\dnun}{\Delta \nu_{n}}
\newcommand{\dnusn}{\Delta \nu^{*}_{n}}

\newcommand{\da}{\Delta a_{}}
\newcommand{\das}{\Delta a^{*}_{}}
\newcommand{\db}{\Delta b_{}}
\newcommand{\dbs}{\Delta b^{*}_{}}
\newcommand{\dnu}{\Delta \nu_{}}
\newcommand{\dnus}{\Delta \nu^{*}_{}}

\newcommand{\dhan}{\Delta \widehat{a}_{n}}
\newcommand{\dhbn}{\Delta \widehat{b}_{n}}
\newcommand{\dhnun}{\Delta \widehat{\nu}_{n}}

\newcommand{\han}{\widehat{a}_{n}}
\newcommand{\hbn}{\widehat{b}_{n}}
\newcommand{\hnun}{\widehat{\nu}_{n}}
\newcommand{\hGn}{\widehat{G}_{n}}
\newcommand{\hlambdan}{\widehat{\lambda}_{n}}

\newcommand{\lan}{\overline{a}_{n}}
\newcommand{\lbn}{\overline{b}_{n}}
\newcommand{\lnun}{\overline{\nu}_{n}}
\newcommand{\lGn}{\overline{G}_{n}}
\newcommand{\llambdan}{\overline{\lambda}_{n}}

\newcommand{\G}{G}
\newcommand{\Gs}{G_{*}}
\newcommand{\lambdas}{\lambda^*}
\newcommand{\as}{a^{*}}
\newcommand{\bs}{b^{*}}
\newcommand{\nus}{\nu^{*}}

\newcommand{\mnp}{M^\prime_n}

\newcommand{\dfourbar}{{
\overline{D_4}((\lambdan,\Gn),(\lambdasn, \Gsn))
}}

\newcommand{\dtwo}{{
{D_2}((\lambdan,\Gn),(\lambdasn, \Gsn))
}}
\newcommand{\dbtwo}{D_{2}((\lambdan,\Gn),(\overline{\lambda},\overline{G}))}
\newcommand{\dbstwo}{D_{2}((\lambdasn,\Gsn),(\overline{\lambda},\overline{G}))}
\newcommand{\dptwo}{D^{\prime}_{2}((\lambdan,\Gn),(\lambdasn,\Gsn))}
\newcommand{\dtwobarl}{{
\overline{D_2}((\lambda,G),(\lambdas, \Gs))
}}

\newcommand{\done}{D_1((\lambdan,\Gn),(\lambdasn,\Gsn))}

\newcommand{\ysigma}{ Y|\sigma((a_n^*)^{\top}X+b_n^*),\nusn }
\newcommand{\ysigman}{ Y|\sigma((a_n)^{\top}X+b_n),\nun }
\newcommand{\ysigmai}{ Y|\sigma((a_i)^{\top}X+b_i),\nu_i }
\newcommand{\ysigmazero}{ Y|\sigma((a_0)^{\top}X+b_0),\nu_0 }

\newcommand{\ysigmap}{ Y|\sigma((a')^{\top}X+b'),\nu' }
\newcommand{\ysigmaa}{ Y|\sigma(a^{\top}X+b),\nu}
\newcommand{\ysigmaonen}{ Y|\sigma((a_{1,n})^{\top}X+b_{1,n}),\nu_{1,n} }
\newcommand{\ysigmatwon}{ Y|\sigma((a_{2,n})^{\top}X+b_{2,n}),\nu_{2,n} }

\newcommand{\ylinearsigmai}{ Y|a_i^{\top}X+b_i,\nu_i }

\newcommand{\yphizero}{ Y|\varphi((a_0)^{\top}X+b_0),\nu_0 }

\newcommand{\xabs}{ X,\asn,\bsn }

\newcommand{\xabzero}{ X,a_0,b_0 }

\newcommand{\pminus}{p_{\lambdan,\Gn}(X,Y)-p_{\lambdasn,\Gsn}(X,Y)}

\newcommand{\lbg}{\lambda, G}
\newcommand{\lbgs}{\lambdas, \Gs}
\newcommand{\hlbgn}{\hlambdan, \hGn}
\newcommand{\llbgn}{\llambdan, \lGn}
\newcommand{\plbg}{p_{\lbg}}
\newcommand{\plbgs}{p_{\lbgs}}
\newcommand{\phlbgn}{p_{\hlambdan,\hGn}}
\newcommand{\barplbg}{\Bar{p}_{\lbg}}
\newcommand{\barphlbgn}{\Bar{p}_{\hlbgn}}

\newcommand{\linepxi}{\overline{\mathcal{P}}(\Xi)}
\newcommand{\linephalfxi}{\overline{\mathcal{P}}^{1/2}(\Xi)}
\newcommand{\linephalf}{\overline{\mathcal{P}}^{1/2}}

\newcommand{\tlb}{\widetilde{\lambda}}
\newcommand{\tg}{\widetilde{G}}

\DeclareMathOperator*{\argmax}{arg\,max}

\newtheorem{lemma}{Lemma}

\newtheorem{theorem}{Theorem}
\newtheorem{proposition}{Proposition}
\newtheorem{definition}{Definition}

\begin{document}

%
\runningtitle{Minimax Parameter Estimation in Contaminated Mixture of Experts}

%

\twocolumn[

    \aistatstitle{Understanding Expert Structures on Minimax Parameter Estimation in Contaminated Mixture of Experts}

\aistatsauthor{Fanqi Yan$^{\star,\diamond}$ \And Huy Nguyen$^{\star, \dagger}$ \And  Dung Le$^{\star,\dagger}$ \And Pedram Akabrian$^{\ddagger}$ \And Nhat Ho$^{\dagger}$}

\aistatsaddress{ $^{\diamond}$Department of Computer Science \\
$^{\dagger}$Department of Statistics and Data Sciences \\ $^{\ddagger}$Department of Electrical and Computer Engineering \\ The University of Texas at Austin} ]

\begin{abstract}
We conduct the convergence analysis of parameter estimation in the contaminated mixture of experts. 
This model is motivated from the prompt learning problem where ones utilize prompts, which can be formulated as experts, to fine-tune a large-scale pre-trained model for learning downstream tasks. 
There are two fundamental challenges emerging from the analysis: (i) the proportion in the mixture of the pre-trained model and the prompt may converge to zero during the training, leading to the prompt vanishing issue; 
(ii) the algebraic interaction among parameters of the pre-trained model and the prompt can occur via some partial differential equations and decelerate the prompt learning. 
In response, we introduce a distinguishability condition to control the previous parameter interaction. 
Additionally, we also investigate various types of expert structure to understand their effects on the convergence behavior of parameter estimation. In each scenario, we provide comprehensive convergence rates of parameter estimation along with the corresponding minimax lower bounds. Finally, we run several numerical experiments to empirically justify our theoretical findings. 

\end{abstract}

\section{INTRODUCTION}
\label{sec:introduction}
Originally introduced by \cite{Jacob_Jordan-1991}, mixture of experts (MoE) has been widely used as a statistical machine learning framework to integrate the power of several sub-models based on the principle of divide and conquer. In particular, it consists of two main components, namely multiple experts formulated as either regression functions \citep{ho2022gaussian} or classifiers \citep{nguyen2024general}, each of which is responsible for handling a few specific tasks; and a gating function softly partitioning the input space into regions where corresponding specialized experts are assigned larger weights than the others. Thanks to its flexibility, there is a surge of interests in utilizing MoE model for various applications, including natural language processing \citep{jiang2024mixtral,zhou2023brainformers,fedus2021switch, puigcerver2024sparse,Du_Glam_MoE,pham2024competesmoe}, computer vision \citep{Riquelme2021scalingvision,liang_m3vit_2022,Ruiz_Vision_MoE}, autonomous driving \citep{pini2023safe}, multimodal fusion~\citep{fleximodal_fusion}, reinforcement learning \citep{chow_mixture_expert_2023,ceron2024rl} and others \citep{li2023sparse,hazimeh_dselect_k_2021,gupta2022sparsely,ng2023botbuster}. 
Recently, practitioners have also incorporated MoE models into parameter-efficient fine-tuning methods for large-scale pre-trained models \citep{zadouri2024pushing,shen2024instruction,le2024mixtureexpertsmeetspromptbased}. More specifically, they attach small trainable experts called prompts to a frozen pre-trained model, which we refer to as a \emph{contaminated MoE} in the sequel. By doing so, they are able to learn downstream tasks efficiently without having to train the whole model expensively. Nevertheless, the potentials for further enhancing this fine-tuning technique have been restricted owing to the absence of the theoretical foundation for the contaminated MoE model. Thus, our main objective in this work is to provide a comprehensive convergence analysis of parameter estimation under the following probabilistic setting of the contaminated MoE to potentially shed some light on both the prompt convergence behavior and the prompt design.

\textbf{Problem setting.} To begin with, we assume that $(X_1 , Y_1 ), (X_2,Y_2), \ldots , (X_n , Y_n ) \in \mathcal{X}\times\mathcal{Y} \subset\mathbb{R}^d\times\mathbb{R} $
are i.i.d. samples of size $n$ drawn from a contaminated MoE model whose conditional density function is given by
\begin{align}
\label{eq:deviated-model}
    p_{\lambda^*, G_*}(Y|X):=&~(1-\lambda^*)f_0(Y|\varphi(a_0^{\top}X+b_0),\nu_0) \nonumber
    \\
    +&~\lambda^{*}f(Y|\sigma((a^*)^{\top}X+b^*),\nu^*).
\end{align}
Above, $f_0(\cdot|\mu_0,\nu_0)$ is a known probability density function with mean $\mu_0$ and variance $\nu_0$, which represents for the pre-trained model. Meanwhile, $f(\cdot|\mu, \nu)$ is the univariate Gaussian density function with mean $\mu$ and variance $\nu$, standing for the prompt. Next, $\lambdas \in [0, 1]$ denotes a mixing proportion, $G_0:=(a_0,b_0,\nu_0)$ is a set of known parameters of the pre-trained model belonging to the parameter space $\Theta\subset \mathbb{R}^d\times\mathbb{R}\times\mathbb{R}^{+}$,
whereas $\Gs:=(\as,\bs,\nus)\in\Theta$ is a true yet unknown parameter set of the prompt. For the sake of presentation, we suppress the dependence of $G_*$ on the sample size $n$. 
In addition, $X_1,X_2,\ldots,X_n$ 
are i.i.d. samples from some continuous probability distribution with the density function denoted by $\bar{f}(X)$. Lastly, note that the expectation of the response variable $Y$ conditioning on the covariate (input) $X$ is a mixture of experts $\varphi(a_0^{\top}X+b_0)$ and $\sigma((a^*)^{\top}X+b^*)$, that is,
\begin{align*}
    \E[Y|X]=(1-\lambdas)\varphi(a_0^{\top}X+b_0)+\lambdas\sigma((a^*)^{\top}X+b^*),
\end{align*}
where $\varphi,\sigma:\mathbb{R}\to\mathbb{R}$ are two given real-valued activation functions. To ensure the dependence of expert values on the input $X$ as in practice, we assume that $a^*,a_0\neq 0_d$ and the first derivative of $\sigma$ is different from zero in a non-zero measure set in $\mathbb{R}$ to eliminate the trivial case when the expert function $\sigma$ is a constant.

\textbf{Maximum likelihood estimation (MLE).} Subsequently, to estimate the unknown parameters $\lambdas$ and $G_*=(a^*,b^*,\nu^*)$ of model~\eqref{eq:deviated-model}, we leverage the maximum likelihood method \citep{Vandegeer-2000} as follows:
\begin{align}
\label{eq:MLE}
    (\widehat{\lambda}_n,\widehat{G}_n)\in
    \argmax_{(\lambda,G)\in\Xi}
    \sum_{i=1}^n
    \log(p_{\lambda,G}(Y_i|X_i)),
\end{align}
where $\widehat{G}_n :=(\widehat{a}_n, \widehat{b}_n,\widehat{\nu}_n)$ and $\Xi:=[0,1]\times\Theta$. In order to guarantee the convergence of parameter estimation, we assume that the parameter space $\Theta$ is compact, whereas the input space $\mathcal{X}$ is bounded. Note that these assumptions are mild and commonly employed in previous works \citep{ho2022gaussian,do_strong_2022,nguyen2024deviated}.

\textbf{Related works.} There are early attempts to comprehend the properties of MoE models from the perspective of the parameter estimation problem. Firstly, \cite{ho2022gaussian} investigated the convergence behavior of MLE under the covariate-free gating Gaussian MoE model. They discovered that the convergence of parameter estimation might be negatively affected by parameter interactions (via some partial differential equations), depending on the algebraic independence between expert functions. 
In particular, they showed that parameter estimation rates were determined by the solvability of a system of polynomial equations, which became significantly slow when the number of fitted parameters increased. Then,  \cite{do_strong_2022} generalized these results to the settings where the conditional distribution of the response variable given the covariate was a mixture of distributions other than Gaussians. Lastly, \cite{nguyen2024deviated} explored a deviated Gaussian MoE where they attached a mixture of Gaussian distribution to the pre-trained model $f_0$ rather than a single Gaussian as in equation~\eqref{eq:deviated-model}. However, they took only linear experts into account, i.e. the activation $\sigma$ was an identity function. Furthermore, they also let the ground-truth parameters be independent of the sample size $n$, which simplified the problem considerably and induced only point-wise parameter estimation rates. Therefore, we aim to examine the impacts of various expert structures on the uniform convergence rates of parameter estimation as well as establish the corresponding minimax lower bounds for those rates in this work.

\textbf{Challenges.} There are two fundamental challenges arising from the dependence of ground-truth parameters $\lambdas$ and $G_*$ on the sample size $n$:

\emph{(i) Prompt vanishing.} The mixing proportion $\lambdas$ may converge to zero as $n$ goes to infinity, making the prompt disappear from the contaminated MoE model~\eqref{eq:deviated-model}. Consequently, the prompt's effects on fine-tuning the pre-trained model would become minor. Thus, it would be difficult to characterize the convergence behavior of estimation for the prompt parameters $G_*$.

\emph{(ii) Prompt merging.} When $f_0$ belongs to the family of Gaussian densities and the two expert functions $\varphi$ and $\sigma$ are identical, that is, $\varphi=\sigma$, if the prompt parameters $G_*$ converge to those of the pre-trained model $G_0$ as $n$ goes to infinity, then the prompt in equation~\eqref{eq:deviated-model} will merge into the pre-trained model. This phenomenon not only hinders the ability of the contaminated MoE to learn downstream tasks but also induces several obstacles in capturing the convergence behavior of estimation for the mixing proportion $\lambdas$.
\begin{table*}[!ht]
\caption{Summary of the convergence rates of parameter estimation in the contaminated MoE model. Below, we denote $\Delta \Gs:=(\as-a_0,\bs-b_0,\nus-\nu_0)$ and $\vartheta_{\Delta \Gs}^{\alpha,\beta,\gamma}:=\Vert \as-a_0 \Vert^{\alpha}+|\bs-b_0 |^{\beta}+|\nus -\nu_0|^{\gamma}$. 
Additionally, $f_0=f$ means that $f_0$ is a Gaussian density, while $f_0\neq f$ indicates the opposite.
}
\centering
\begin{tabular}{ |c|c|c|c|c|c|c|c|c|c|c| } 
\hline
$\boldsymbol{\sigma}$ & $\boldsymbol{f_0}$ & $\boldsymbol{\varphi}$  & \textbf{Thm.} & $\boldsymbol{|\hlambdan-\lambdas|}$ &$\boldsymbol{\|\han-\as\|}$, $\boldsymbol{|\hnun-\nus|}$ &$\boldsymbol{|\hbn-\bs|}$ \\
\hline
\multirow{3}{2.5em}{linear} & $f_0\neq f$ &$-$& \ref{theorem:sigma-linear-f0-notGaussian}& $\widetilde{\mathcal{O}}(n^{-\frac{1}{2}})$ & \multicolumn{2}{c|}{$\widetilde{\mathcal{O}}(n^{-\frac{1}{2}}(\lambdas)^{-1})$}\\ 
\cline{2-7}
   & \multirow{2}{3em}{$f_0=f$} & linear&\ref{theorem:sigma-linear-f0-Gaussian-varphi-linear} & $\widetilde{\mathcal{O}}((n\cdot\vartheta_{\Delta \Gs}^{4,8,4})^{-\frac{1}{2}})$ & {$\widetilde{\mathcal{O}}((n\cdot\vartheta_{\Delta \Gs}^{2,4,2})^{-\frac{1}{2}}(\lambdas)^{-1})$} & $\widetilde{\mathcal{O}}((n\cdot\vartheta_{\Delta \Gs}^{2,4,2})^{-\frac{1}{4}}(\lambdas)^{-\frac{1}{2}})$\\ 
\cline{3-7}
   && non-linear&\ref{theorem:sigma-linear-f0-Gaussian-varphi-nonlinear}& $\widetilde{\mathcal{O}}(n^{-\frac{1}{2}})$ & \multicolumn{2}{c|}{$\widetilde{\mathcal{O}}(n^{-\frac{1}{2}}(\lambdas)^{-1})$}\\ \hline
\multirow{3}{2.5em}{non-linear} & $f_0\neq f$ &$-$ & \ref{theorem:sigma-nonlinear-f0-notGaussian}& $\widetilde{\mathcal{O}}(n^{-\frac{1}{2}})$ & \multicolumn{2}{c|}{$\widetilde{\mathcal{O}}(n^{-\frac{1}{2}}(\lambdas)^{-1})$}\\ 
\cline{2-7}
    & \multirow{2}{3em}{$f_0=f$} & linear& \ref{theorem:sigma-nonlinear-f0-Gaussian-varphi-linear}& $\widetilde{\mathcal{O}}(n^{-\frac{1}{2}})$ & \multicolumn{2}{c|}{$\widetilde{\mathcal{O}}(n^{-\frac{1}{2}}(\lambdas)^{-1})$}\\ 
\cline{3-7}
    && non-linear&\ref{theorem:sigma-nonlinear-f0-Gaussian-varphi-nonlinear}& $\widetilde{\mathcal{O}}(n^{-\frac{1}{2}}\cdot\Vert\Delta\Gs\Vert^{-2})$ & \multicolumn{2}{c|}{$\widetilde{\mathcal{O}}(n^{-\frac{1}{2}}\cdot(\Vert\Delta\Gs\Vert\lambdas)^{-1})$} \\ \hline
\end{tabular}
\label{table:parameter_rates}
\end{table*}

\textbf{Contributions.} To address these challenges, we consider various structures of the expert functions $\varphi$ and $\sigma$ (whether they are linear or non-linear) and different families of the pre-trained model $f_0$ (whether it belongs to the Gaussian density family). Our contributions are three-fold and can be summarized as follows:

\emph{1. Uniform convergence rates:} In each of the aforementioned scenarios, we establish the corresponding convergence rates of the mixing proportion ($\widehat{\lambda}_n$ to $\lambdas$) and the prompt parameters ($(\widehat{a}_n,\widehat{b}_n,\widehat{\nu}_n)$ to $(\as,\bs,\nus)$). Notably, we find that whether the proportion rate depends on the prompt parameter rates or not varies with the expert structures and the pre-trained model family. By contrast, the latter rates always hinge upon the former rate under every scenario (see also Table~\ref{table:parameter_rates}).

\emph{2. Minimax lower bounds:} To demonstrate that the above uniform convergence rates of parameter estimation are optimal, we also derive the associated minimax lower bound with each of those rates. 

\emph{3. Theoretical insights into fine-tuning method:} Finally, from the results in our theoretical analysis, we provide several useful insights into designing prompts to improve the efficiency of fine-tuning techniques.

\textbf{Organization.} The rest of the paper proceeds as follows. In Section~\ref{sec:preliminaries}, we introduce some necessary concepts for our analysis and capture the density estimation rates. Then, we present the convergence rates of parameter estimation along with their corresponding minimax lower bounds when the expert $\sigma$ is linear and non-linear in Section~\ref{section:linear-sigma} and Section~\ref{section:non-linear-sigma}, respectively. Next, we conduct numerical experiments to justify our theoretical results in Section~\ref{sec:experiments}. Lastly, we discuss and conclude the paper in Section~\ref{sec:discussion}. Additional experiments and full proofs are defer to the Appendices.

\textbf{Notations.}
Firstly, we let $[n]: = \{1, 2, \ldots, n\}$ for any $n\in\mathbb{N}$. Next, for any vector of real numbers $u$ we denote $\|u\|$ as its Euclidean norm value. Given any two positive sequences $(a_n)_{n\geq 1}$ and $(b_n)_{n\geq 1}$, we write $a_n = \mathcal{O}(b_n)$ or $a_{n} \lesssim b_{n}$ if $a_n \leq C b_n$ for all $ n\in\mathbb{N}$, where $C > 0$ is some universal constant. Besides, the notation $a_n=\widetilde{\mathcal{O}}(b_n)$ indicates the previous inequality may depend on some logarithmic term. Lastly, for any two densities $p$ and $q$ (dominated by the Lebesgue measure $m$), their squared Hellinger distance is computed as $h^2(p, q):=\frac{1}{2} \int[\sqrt{p(x)}-\sqrt{q(x)}]^2 d m(x)$.

\section{PRELIMINARIES}
\label{sec:preliminaries}
We start this section with the following \emph{distinguishability condition} to control the merging of the prompt into the pre-trained model mentioned in Section~\ref{sec:introduction}. 

\begin{definition}[Distinguishability]
\label{def:distinguish}
We say that $f$ is distinguishable from $f_0$ if the following holds:\\
    Given any two distinct components
    $(a_i,b_i,\nu_i)\in\Theta$, for $i=1,2$,
    if we have real coefficients 
    $\eta_i$ for $0\leq i \leq2$
    such that $\eta_1\eta_2 \leq 0$ and
    \begin{align*}
        \eta_0 f_0(Y|\varphi (a_0^{\top}X&+b_0 ),\nu_0)\\
        +
        \sum_{i=1}^{2}&\eta_if(\ysigmai)=0,
    \end{align*}
for almost surely $(X,Y)\in \mathcal{X}\times\mathcal{Y}$, then $\eta_i=0,\forall i$.
\end{definition}
\textbf{Example.} If $f_0$ belongs to the family of Student's t-distribution, we can verify that $f$ is distinguishable from $f_0$. By contrast, if $f_0$ is a Gaussian density and $\varphi=\sigma$, then $f$ is not distinguishable from $f_0$. 

Note that if $f$ is distinguishable from $f_0$, the merging $f_0(Y|\varphi (a_0^{\top}X+b_0 ),\nu_0)=f(\ysigmai)$, for almost surely $(X,Y)\in \mathcal{X}\times\mathcal{Y}$, cannot occur as it violates the distinguishability condition. Furthermore, according to Proposition~\ref{prop:identifiability}, whose proof is in Appendix~\ref{appendix:identifiability}, this condition also ensures that the contaminated MoE model~\eqref{eq:deviated-model} is identifiable.
\begin{proposition}[Identifiability]
\label{prop:identifiability}
    Let $(\lambda,G), (\lambda',\G')$ be two components in $\Xi$. Suppose that $f$ is distinguishable from $f_0$, then if the identifiability equation 
    $p_{\lambda, G}(Y|X) =p_{\lambda^\prime , G^\prime}(Y|X) $
    holds for almost surely $(X,Y)\in \mathcal{X}\times\mathcal{Y}$, we obtain that $(\lambda, G) = (\lambda^\prime , G^\prime )$.
\end{proposition}
Subsequently, we proceed to characterize the convergence behavior of density estimation in Theorem~\ref{theorem:ConvergenceRateofDensityEstimation}.

\begin{theorem}
\label{theorem:ConvergenceRateofDensityEstimation}
Assume that the density $f_0$ is bounded with tail 
$\mathbb{E}_X\left[-\log f_0(Y|\varphi(a_0^{\top}X+b_0 ),\nu_0)\right]\gtrsim Y^q$,
for almost surely $Y\in\mathcal{Y}$
for some $q>0$.
Then, there exists a constant $C > 0$ depending on $f_0,\lambdas,\Gs$ and $\Xi$ such that the following holds true:
\begin{align*}
    \sup_{(\lambdas,\Gs)\in\Xi}
    \mathbb{E}_{p_{\lambdas,\Gs}}
    h(p_{\widehat{\lambda}_n,\widehat{G}_n},p_{\lambdas,\Gs})
    \leq
    C\sqrt{\log n/n}.
\end{align*}
\end{theorem}
Proof of Theorem~\ref{theorem:ConvergenceRateofDensityEstimation} is in Appendix \ref{appendix:ConvergenceRateofDensityEstimation}.
The above bound shows that the density estimator $p_{\hlambdan, \widehat{G}_n}$ converges to the true density $p_{\lambda^*, \Gs}$ at the parametric rate of order $\widetilde{\mathcal{O}}(n^{-1/2})$. Moreover, if we can construct some parameter loss function $D((\lambda,G),(\lambdas,\Gs))$ that satisfies $D((\hlambdan,\hGn),(\lambdas,\Gs))\lesssim h(p_{\widehat{\lambda}_n,\widehat{G}_n},p_{\lambdas,\Gs})$, we achieve our desired parameter estimation rates (see Appendix~\ref{appendix:Proofs for the MLE rate} for further details). Finally, to utilize Theorem~\ref{theorem:ConvergenceRateofDensityEstimation}, we assume that the density $f_0$ is bounded with the tail  
$\mathbb{E}_X
\left[
-\log f_0(Y|\varphi (a_0^{\top}X+b_0 ),\nu_0)
\right]
\gtrsim
Y^q
$,
for almost surely $Y \in \mathcal{Y}$ for some $q > 0$, throughout this paper unless stating otherwise.

\section{LINEAR EXPERT \texorpdfstring{$\sigma$}{a}}
\label{section:linear-sigma}
In this section, we determine the convergence rates of parameter estimation when the expert $\sigma$ is linear, that is, it takes the form $\sigma(z)=a'z+b'$ for any $z\in\mathbb{R}$, where $a',b'$ are known constants. Since the rates vary with the family of the density $f_0$, we further consider two scenarios when $f_0$ is and is not a Gaussian density.

\subsection{When \texorpdfstring{$f_0$}{ } is not a Gaussian density}
\label{section:sigma-linear-f0-notGaussian}
We begin with the simpler scenario when $f_0$ is not a Gaussian density as we can verify that $f$ is distinguishable from $f_0$ under this scenario in Proposition~\ref{lemma:distinguish-linear sigma not Gaussian}.
\begin{proposition}
    \label{lemma:distinguish-linear sigma not Gaussian}
    If $f_0$ does not belong to the family of Gaussian densities, then $f$ is distinguishable from $f_0$. 
\end{proposition}
Proof of Proposition~\ref{lemma:distinguish-linear sigma not Gaussian} is in Appendix~\ref{appendix:lemma:distinguish-linear sigma not Gaussian}. Given this result, we are now ready to capture the convergence of parameter estimation in the following theorem:

\begin{theorem}[MLE rates]
    \label{theorem:sigma-linear-f0-notGaussian}
When the expert function $\sigma$ is of linear form and $f_0$ is not a Gaussian density, we obtain the MLE convergence rates as follows:
\begin{align*}
    \sup_{(\lambdas,G_*)\in\Xi}\mathbb{E}_{p_{\lambdas,\Gs}} 
    \Big[
    |\widehat{\lambda}_n
    -\lambdas|^2 
    \Big] 
    \lesssim \frac{\log n}{n},\\
    \sup_{(\lambdas,G_*)\in\Xi}\mathbb{E}_{p_{\lambdas,\Gs}} 
    \Big[
    (\lambdas)^2 
    \Vert 
    \hGn-\Gs
    \Vert^2 
    \Big] 
    \lesssim 
    \frac{\log n}{n}.
\end{align*}
\end{theorem}
The proof of Theorem~\ref{theorem:sigma-linear-f0-notGaussian} is in Appendix~\ref{appendixproof:sigma-linear-f0-notGaussian}.
The first bound indicates that the mixing proportion estimator $\hlambdan$ converges to the ground-truth value $\lambdas$ at the rate of order $\widetilde{\mathcal{O}}(n^{-1/2})$, which is parametric on the sample size $n$. On the other hand, it follows from the second bound that the convergence rate of the prompt parameter estimation $\hGn=(\widehat{a}_n,\widehat{b}_n,\widehat{\nu}_n)$ to its true counterpart $\Gs=(\as,\bs,\nus)$ is slower than $\widetilde{\mathcal{O}}(n^{-1/2})$ as it hinges upon the vanishing rate of $\lambdas$. This rate dependence reflects the prompt vanishing issue mentioned in the ``Challenges'' paragraph in Section~\ref{sec:introduction}.

Next, we establish the minimax lower bounds for estimating the true parameters $(\lambdas,\Gs)$ in Theorem~\ref{thm:lower-distinguish}.
\begin{theorem}[Minimax lower bounds]
\label{thm:lower-distinguish}
Under the setting of Theorem~\ref{theorem:sigma-linear-f0-notGaussian}, we have for any $0<r < 1$ that
\begin{align*}
    \inf_{(\llambdan,\overline{G}_n)\in \Xi }\sup_{(\lambda,G)\in \Xi }
    \mathbb{E}_{p_{\lambda,G}} \Big( |\overline{\lambda}_n
    -\lambda|^2 \Big) 
    \gtrsim n^{-1/r},
    \\
    \inf_{(\llambdan,\overline{G}_n)\in \Xi }\sup_{(\lambda,G)\in \Xi }
    \mathbb{E}_{p_{\lambda,G}} 
    \Big( \lambda^2 \Vert \lGn-G \Vert^2 \Big) 
    \gtrsim n^{-1/r}.
\end{align*}
\end{theorem}

Proof of Theorem \ref{thm:lower-distinguish} is in Appendix \ref{proof:lower-distinguish}.
The results of this theorem imply that the convergence rates of the MLE $(\hlambdan,\hGn)$ provided in Theorem~\ref{theorem:sigma-linear-f0-notGaussian} are optimal.

\subsection{When \texorpdfstring{$f_0$}{} is a Gaussian density}
Next, we consider the scenario when $f_0$ belongs to the family of Gaussian densities. Under this scenario, it can be checked that $f$ might not be distinguishable from $f_0$, depending on the structure of the expert $\varphi$ in the pre-trained model. To this end, we continue to divide the analysis into two smaller scenarios where the function $\varphi$ is linear and non-linear, respectively.
\subsubsection{Linear expert \texorpdfstring{$\varphi$}{}}
\label{section:sigma-linear-f0-Gaussian-varphi-linear}
Recall that in this case, both the experts $\varphi$ and $\sigma$ are of linear forms, and the density $f_0=f$ is a Gaussian density. Notably, this setting induces two potential obstacles in our analysis. 

Firstly, the prompt might merge into the pre-trained model. In particular, let us assume that $\varphi(a_0^{\top}X+b_0)=\alpha_0(a_0^{\top}X+b_0)+\beta_0$ and $\sigma((\as)^{\top}X+\bs)=\alpha^*((\as)^{\top}X+\bs)+\beta^*$, where $\alpha_0,\alpha^*\neq0$ and $\beta_0,\beta^*$ are some known constants. Then, if  $\as\to\frac{\alpha_0}{\alpha^*}a_0$ and $\bs\to\frac{\alpha_0b_0+\beta_0-\beta^*}{\alpha^*}$ and $\nus\to\nu_0$ as $n\to\infty$, it follows that the expert $\sigma((\as)^{\top}X+\bs)$ will converge to its counterpart $\varphi(a_0^{\top}X+b_0)$. Consequently, we can justify that $f$ is not distinguishable from $f_0$. Furthermore, since the prompt
$f(Y|\sigma((\as)^{\top}X+\bs),\nus)$ will also converge to the pre-trained model $f_0(Y|\varphi(a_0^{\top}X+b_0),\nu_0)$, we have to cope with the prompt merging issue mentioned in the ``Challenges'' paragraph in Section~\ref{sec:introduction} as well. 
This issue will be demonstrated to make the convergence behavior of the mixing proportion estimator become complicated in Theorem~\ref{theorem:sigma-linear-f0-Gaussian-varphi-linear}. 
For simplicity, we will focus only on the case when $\varphi$ and $\sigma$ are identity functions, that is, when $\alpha_0=\alpha^*=1$ and $\beta_0=\beta^*=0$. Other cases can be naturally generalized based on the aforementioned limits of $a^*$ and $b^*$.

Secondly, there is an interaction between parameters $b$ and $\nu$ via the following heat equation:
\begin{align}
\label{eq:heat_equation_1}
        \frac{\partial^2 f}{\partial b^2} (Y|a^{\top}X+b ,\nu ) &= 2 \frac{\partial f}{\partial \nu}(Y|a^{\top}X+b ,\nu).
\end{align}
The derivation of this equation can be seen in Appendix~\ref{appendix:ProofsforAuxiliaryResults}.
Note that such interaction has been captured in prior works on Gaussian MoE \citep{ho2022gaussian,nguyen2024gaussian,nguyen2024temperature} where it is shown to decelerate the convergence rates of involved parameters significantly. However, to the best of our knowledge, the effects of that interaction has never been explored when the ground-truth parameter values $\lambdas,\Gs$ depend on the sample size $n$. Thus, we will demistify this problem in Theorem~\ref{theorem:sigma-linear-f0-Gaussian-varphi-linear}.

For simplicity, let us denote $\Delta G:=(\Delta a,\Delta b,\Delta\nu) := (a - a_0,b - b_0,\nu-\nu_0 )$ for any component $(a,b,\nu)\in \Theta$.


\begin{theorem}[MLE rates]
    \label{theorem:sigma-linear-f0-Gaussian-varphi-linear} 
Suppose that $f_0$ is a Gaussian density and both the experts $\varphi$ and $\sigma$ are identity functions.
For any sequence $(l_n)_{n\geq 1}$, we denote
\begin{align*}
    \Xi_1(l_n):=
\Big\{ &
(\lambda,G)={(\lambda,a,b,\nu)}\in\Xi:
\\&
\frac{l_n}{\min_{1\leq i\leq d}\{ 
|(\da)_i|^2,|\db|^4,|\dnu|^2
\}\sqrt{n}}\leq\lambda
\Big\}.
\end{align*}
Then, the followings hold for any sequence $(l_n )_{n\geq 1}$ such that $l_n/\log n\to\infty$ as $n\to\infty$:
\begin{align*}
    \sup_{(\lambdas,\Gs)\in \Xi_1(l_n)  }
    \mathbb{E}_{p_{\lambdas,\Gs}} \Big[ 
    (\Vert\das\Vert^4+|\dbs|^8+|\dnus|^4)
    \\
    \times |\widehat{\lambda}_n
    -\lambdas|^2 \Big] 
    \lesssim \frac{\log n}{n},
    \\
    \sup_{(\lambdas,\Gs)\in \Xi_1(l_n) }
    \mathbb{E}_{p_{\lambdas,\Gs}} 
    \Big[ (\lambdas)^2 
    (\Vert\das\Vert^2+|\dbs|^4+|\dnus|^2)
    \\
    \times(\Vert\han-\as\Vert^2+|\hbn-\bs|^4+|\hnun-\nus|^2)
    \Big] 
    \lesssim \frac{\log n}{n}.
\end{align*}
\end{theorem}
The proof of Theorem~\ref{theorem:sigma-linear-f0-Gaussian-varphi-linear} is in Appendix~\ref{appendixproof:sigma-linear-f0-Gaussian-varphi-linear}. Above, we take the set $\Xi_1(l_n)$ into account to guarantee the consistency of the MLE estimators under this scenario. Next, there are two main implications from the results of Theorem~\ref{theorem:sigma-linear-f0-Gaussian-varphi-linear}:

(i) In contrast to Theorem~\ref{theorem:sigma-linear-f0-notGaussian}, the first bound in Theorem \ref{theorem:sigma-linear-f0-Gaussian-varphi-linear} indicates that 
the convergence rate of the mixing proportion estimator $\hlambdan$ is now slower than $\widetilde{\mathcal{O}}(n^{-1/2})$ and is determined by that of the quantity $(\Vert\Delta\as\Vert^4+|\Delta\bs|^8+|\Delta\nus|^4)$. This rate deceleration occurs as a result of the prompt merging issue brought up at the beginning of Section~\ref{section:sigma-linear-f0-Gaussian-varphi-linear}.

(ii) Again, different from Theorem~\ref{theorem:sigma-linear-f0-notGaussian} where we observe that the prompt parameter estimators $\han,\hbn,\hnun$ share the same convergence rate, the second bound in Theorem~\ref{theorem:sigma-linear-f0-Gaussian-varphi-linear} points out a mismatch among their rates under this scenario. In particular, the convergence rate of $\hbn$ is the squared root of those of $\han$ and $\hnun$. Moreover, all these rates depend on the vanishing rate of not only $\lambdas$ but also that of the term $(\Vert\das\Vert^2+|\dbs|^4+|\dnus|^2)$. This intriguing phenomenon primarily arises from the parameter interaction through the heat equation~\eqref{eq:heat_equation_1}.

Similar to Section~\ref{section:sigma-linear-f0-notGaussian}, we also derive the corresponding minimax lower bounds for estimating the ground-truth parameters $(\lambdas,\Gs)$ in Theorem~\ref{theorem:sigma-linear-f0-Gaussian-varphi-linear}.
\begin{theorem}[Minimax lower bounds]
\label{appthm:lower-mle-Gaussian-linear}
    Under the setting of Theorem~\ref{theorem:sigma-linear-f0-Gaussian-varphi-linear}, we have for any $0<r < 1$ and sequence $(l_n)_{n\geq 1}$ that
\begin{align*}
    &\inf_{(\llambdan,\overline{G}_n)\in \Xi }\sup_{(\lambda,G)\in \Xi_1(l_n)  }
    \mathbb{E}_{p_{\lambda,G}} \Big[ 
    (\Vert\da\Vert^4+\Vert\db\Vert^8+\Vert\dnu\Vert^4)\\
    &\hspace{4.5cm}\times |\overline{\lambda}_n
    -\lambda|^2 \Big]\gtrsim n^{-1/r},\\
    &\inf_{(\llambdan,\overline{G}_n)\in \Xi }\sup_{(\lambda,G)\in \Xi_1(l_n) }
    \mathbb{E}_{p_{\lambda,G}} 
    \Big[\lambda^2 
    (\Vert\da\Vert^2+\Vert\db\Vert^4+\Vert\dnu\Vert^2)\\
    &\quad \times(\Vert\lan-a\Vert^2+\Vert\lbn-b\Vert^4+\Vert\lnun-\nu\Vert^2)
    \Big] \gtrsim n^{-1/r}.
\end{align*}
\end{theorem}
Proof of Theorem \ref{appthm:lower-mle-Gaussian-linear} is in Appendix \ref{appendix:appthm:lower-mle-Gaussian-linear}. The above minimax lower bounds indicate that the convergence rates of the MLE presented in Theorem~\ref{theorem:sigma-linear-f0-Gaussian-varphi-linear} are optimal.

\subsubsection{Non-linear expert \texorpdfstring{$\varphi$}{}}
\label{section:sigma-linear-f0-Gaussian-varphi-nonlinear}
In this section, we draw our attention to the scenario where $f_0$ is a Gaussian density, the expert $\sigma$ is of linear form, while $\varphi$ is a non-linear expert. For instance, $\varphi$ can be one among the activation functions $\mathrm{sigmoid}$, $\mathrm{ReLU}$ and $\mathrm{tanh}$, which are commonly used in practice.

Under this scenario, since the expert $\sigma((\as)^{\top}X+\bs)$ cannot converge to its counterpart $\varphi(a_0^{\top}X+b_0)$ due to their structure distinction, the prompt $f(Y|\sigma((\as)^{\top}X+\bs),\nus)$ does not converge to the pre-trained model $f_0(Y|\varphi(a_0^{\top}X+b_0),\nu_0)$. Therefore, the prompt merging phenomenon does not take place as in Section~\ref{section:sigma-linear-f0-Gaussian-varphi-linear} when $\varphi$ is a linear expert. Additionally, we can also validate that $f$ is distinguishable from $f_0$. As a consequence, we illustrate in the following theorem that the convergence behavior of the MLE in this case is identical to that in Section~\ref{section:sigma-linear-f0-notGaussian}. 

\begin{theorem}[MLE rates]
    \label{theorem:sigma-linear-f0-Gaussian-varphi-nonlinear} 
    Suppose that $f_0$ is a Gaussian density, the expert function $\sigma$ is linear, while its counterpart $\varphi$ is a non-linear expert.
Then, we obtain the following MLE convergence rates:
\begin{align*}
    \sup_{(\lambdas,G_*)\in\Xi}\mathbb{E}_{p_{\lambdas,\Gs}} 
    \Big[
    |\widehat{\lambda}_n
    -\lambdas|^2 
    \Big] 
    \lesssim \frac{\log n}{n},\\
    \sup_{(\lambdas,G_*)\in\Xi}\mathbb{E}_{p_{\lambdas,\Gs}} 
    \Big[
    (\lambdas)^2 
    \Vert 
    \hGn-\Gs
    \Vert^2 
    \Big] 
    \lesssim 
    \frac{\log n}{n}.
\end{align*}
\end{theorem}
The proof of Theorem~\ref{theorem:sigma-linear-f0-Gaussian-varphi-nonlinear} is in Appendix~\ref{appendixproof:sigma-linear-f0-Gaussian-varphi-nonlinear}. As mentioned above, since the two experts $\sigma$ and $\varphi$ have different structures (linear vs. non-linear), the issue that the prompt merges into the pre-trained model does not exist. Therefore, $\hlambdan$ converges to $\lambdas$ at the rate of order $\widetilde{\mathcal{O}}(n^{-1/2})$, which does not depend on any other factors. On the other hand, note that the prompt vanishing problem still holds when $\lambdas$ approaches zero. Thus, the convergence rate of $\hGn=(\han,\hbn,\hnun)$ is slower than $\widetilde{\mathcal{O}}(n^{-1/2})$ as it is governed by the vanishing rate of $\lambdas$. Lastly, we can also arrive at the minimax lower bounds in Theorem~\ref{thm:lower-distinguish} by arguing similarly. Consequently, the above MLE convergence rates are optimal.  

\section{NON-LINEAR EXPERT \texorpdfstring{$\sigma$}{a}}
\label{section:non-linear-sigma}
Moving to this section, we will look into the convergence behavior of the MLE when the expert $\sigma$ in the prompt takes a non-linear form such as $\mathrm{sigmoid}$, $\mathrm{ReLU}$ and $\mathrm{tanh}$ functions, etc.  
Similar to Section~\ref{section:linear-sigma}, we also separate the analysis into two scenarios when $f_0$ is excluded from and included in the family of Gaussian densities, respectively.

\subsection{When \texorpdfstring{$f_0$}{} is not a Gaussian density}
\label{section:sigma-nonlinear-f0-notGaussian}
Firstly, when $f_0$ does not belong to the family of Gaussian densities, we can verify that $f$ is distinguishable from $f_0$ by employing similar arguments for Proposition~\ref{lemma:distinguish-linear sigma not Gaussian}. Thus, the MLE admits the same convergence behavior as that in Section~\ref{section:sigma-linear-f0-notGaussian}, which is exhibited in Theorem~\ref{theorem:sigma-nonlinear-f0-notGaussian} whose proof is in Appendix~\ref{appendixproof:sigma-nonlinear-f0-notGaussian}.

\begin{theorem}[MLE rates]
    \label{theorem:sigma-nonlinear-f0-notGaussian}
    When the expert function $\sigma$ is of non-linear form and $f_0$ is not a Gaussian density, we obtain the MLE convergence rates as follows:
\begin{align*}
    \sup_{(\lambdas,G_*)\in\Xi}\mathbb{E}_{p_{\lambdas,\Gs}} 
    \Big[
    |\widehat{\lambda}_n
    -\lambdas|^2 
    \Big] 
    \lesssim \frac{\log n}{n},\\
    \sup_{(\lambdas,G_*)\in\Xi}\mathbb{E}_{p_{\lambdas,\Gs}} 
    \Big[
    (\lambdas)^2 
    \Vert 
    \hGn-\Gs
    \Vert^2 
    \Big] 
    \lesssim 
    \frac{\log n}{n}.
\end{align*}
\end{theorem}
Notably, it can be validated that the minimax lower bounds in Theorem~\ref{thm:lower-distinguish} also hold under the setting of this section by employing the same arguments. Thus, we can claim that the above MLE rates are optimal.
\subsection{When \texorpdfstring{$f_0$}{} is a Gaussian density}
\label{section:sigma-nonlinear-f0-Gaussian}
Subsequently, we examine the convergence behavior of parameter estimation when $f_0$ belongs to the family of Gaussian densities. Since the distinguishability of the prompt $f$ from the pre-trained model $f_0$ hinges upon the structure of the expert function $\varphi$, we proceed to divide the analysis into two smaller cases: when $\varphi$ is a linear and non-linear expert function, respectively.
\subsubsection{Linear expert \texorpdfstring{$\varphi$}{}}
\label{section:sigma-nonlinear-f0-Gaussian-varphi-linear}
In this section, we concentrate on the scenario where the expert $\sigma$ takes a non-linear form, while $\varphi$ is a linear expert function. 

Although this scenario is opposite to that in Section~\ref{section:sigma-linear-f0-Gaussian-varphi-nonlinear} where the expert $\sigma$ is linear and its counterpart $\varphi$ is non-linear, these two scenarios have one thing in common, which is the expert structure distinction. Therefore, it can be validated that they share several convergence properties. In particular, since the structures of the two experts $\sigma$ and $\varphi$ are different from each other, it is obvious that $\sigma((\as)^{\top}X+\bs)$ cannot converge to $\varphi(a_0^{\top}X+b_0)$ as $n\to\infty$ for almost surely $X\in\mathcal{X}$. As a result, the prompt $f(Y|\sigma((\as)^{\top}X+\bs),\nus)$ does not converge to the pre-trained model $f_0(Y|\varphi(a_0^{\top}X+b_0),\nu_0)$, either. This means that the expert structure difference helps tackle the prompt merging issue. Furthermore, it also guarantees that $f$ is distinguishable from $f_0$. Consequently, we demonstrate in the following theorem that the convergence behavior of the MLE in this scenario resembles that in Section~\ref{section:sigma-linear-f0-Gaussian-varphi-nonlinear}. 

\begin{theorem}[MLE rates]
    \label{theorem:sigma-nonlinear-f0-Gaussian-varphi-linear} 
    Suppose that $f_0$ is a Gaussian density, the expert function $\sigma$ is non-linear, while its counterpart $\varphi$ is a linear expert. Then, we obtain the following MLE convergence rates:
\begin{align*}
    \sup_{(\lambdas,G_*)\in\Xi}\mathbb{E}_{p_{\lambdas,\Gs}} 
    \Big[
    |\widehat{\lambda}_n
    -\lambdas|^2 
    \Big] 
    \lesssim \frac{\log n}{n},\\
    \sup_{(\lambdas,G_*)\in\Xi}\mathbb{E}_{p_{\lambdas,\Gs}} 
    \Big[
    (\lambdas)^2 
    \Vert 
    \hGn-\Gs
    \Vert^2 
    \Big] 
    \lesssim 
    \frac{\log n}{n}.
\end{align*}
\end{theorem}
The proof of Theorem~\ref{theorem:sigma-nonlinear-f0-Gaussian-varphi-linear} is in Appendix~\ref{appendixproof:sigma-nonlinear-f0-Gaussian-varphi-linear}. Note that we can also derive the minimax lower bounds in Theorem~\ref{thm:lower-distinguish} under the setting of Theorem~\ref{theorem:sigma-nonlinear-f0-Gaussian-varphi-linear} by arguing in a similar fashion and conclude that the above MLE convergence rates are optimal.


\subsubsection{Non-linear expert \texorpdfstring{$\varphi$}{}}
\label{section:sigma-nonlinear-f0-Gaussian-varphi-nonlinear}

In this section, we take into account the last scenario when both the expert functions $\varphi$ and $\sigma$ are non-linear. 

Since both $\varphi$ and $\sigma$ are non-linear expert functions, we observe that if $\sigma$ is not equal to $\varphi$ almost everywhere, then the fact $\sigma((\as)^{\top}X+\bs)$ converges to $\varphi(a_0^{\top}X+b_0)$ as $n\to\infty$ for almost surely $X\in\mathcal{X}$ does not hold. Consequently, the prompt $f$ become distinguishable from the the pre-trained model $f_0$, which allows us to demonstrate that the MLE convergence behavior in this case is totally similar to that in Theorem~\ref{theorem:sigma-nonlinear-f0-Gaussian-varphi-linear} by using the same arguments. 
Therefore, we will focus only on the challenging case when $\varphi=\sigma$ almost everywhere as the prompt might not be distinguishable from the pre-trained model in this situation, inducing the prompt merging issue if $(\as,\bs,\nus)\to (a_0,b_0,\nu_0)$ when $n\to\infty$ as in Section~\ref{section:sigma-linear-f0-Gaussian-varphi-linear}. 

On the other hand, thanks to the non-linear structures of the experts $\varphi$ and $\sigma$, we exhibit in Proposition~\ref{prop:sigma-nonlinear-f0-Gaussian-varphi-nonlinear} a property of the prompt $f$ that wipes out the interaction between parameters $b$ and $\nu$ as in equation~\eqref{eq:heat_equation_1}.

\begin{proposition}
\label{prop:sigma-nonlinear-f0-Gaussian-varphi-nonlinear}
     Suppose that the expert function $\sigma$ is twice differentiable for almost everywhere.
     For any $(a,b,\nu)\in\Theta$, if we have a set of scalars $\alpha_{\eta}$, for $\eta=(\eta_1,\eta_2,\eta_3)\in\mathbb{N}^3$, such that 
     \begin{align*}
         \sum_{\ell=0}^2\sum_{|\eta|=\ell}
         \alpha_{\eta}\cdot
         \frac{\partial^{|\eta|}f}{\partial a^{\eta_1}\partial b^{\eta_2}\partial \nu^{\eta_3}}
         ( Y|\sigma(a^{\top}X+b),\nu)=0
     \end{align*}
     for almost surely $(X,Y)\in\mathcal{X}\times\mathcal{Y}$, then we get that $\alpha_{\eta}=0$
     for all $\eta\in\mathbb{N}^3$ such that $|\eta|:=\eta_1+\eta_2+\eta_3\leq 2$.
\end{proposition}
The proof of Proposition~\ref{prop:sigma-nonlinear-f0-Gaussian-varphi-nonlinear} is in Appendix~\ref{appendix:prop:sigma-nonlinear-f0-Gaussian-varphi-nonlinear}. Given this result, we then illustrate in Theorem~\ref{theorem:sigma-nonlinear-f0-Gaussian-varphi-nonlinear} that due to the parameter interaction disappearance, the convergence rate of parameter estimation is  improved compared to that in Theorem~\ref{theorem:sigma-linear-f0-Gaussian-varphi-linear} where the expert functions $\varphi$ and $\sigma$ also share the same structure but of linear form rather than non-linear form. Before presenting the theorem statement, let us recall that we denote $\Delta G:=(\Delta a,\Delta b,\Delta\nu) := (a - a_0,b - b_0,\nu-\nu_0 )$. 


\begin{theorem}[MLE rates]
    \label{theorem:sigma-nonlinear-f0-Gaussian-varphi-nonlinear}
Suppose that $f_0$ is a Gaussian density, the non-linear expert $\sigma$ is twice differentiable and $\sigma=\varphi$ almost everywhere.
For any sequence $(l_n)_{n\geq 1}$, we denote
\begin{align*}
    \Xi_{2}(l_n):=
\Big\{ &
(\lambda,G)={(\lambda,a,b,\nu)}\in\Xi:
\\
&\frac{l_n}{\min_{1\leq i\leq d}\{ 
|(\da)_i|^2,|\db|^2,|\dnu|^2
\}\sqrt{n}}\leq\lambda
\Big\}.
\end{align*} 
Then, the followings hold for any sequence $(l_n)_{n\geq 1}$ such that $l_n/\log n\to\infty$ as $n\to\infty$:
\begin{align*}
    \sup_{(\lambdas,G_*)\in\Xi_{2}(l_n)}&\mathbb{E}_{p_{\lambdas,\Gs}} 
    \Big[
    \Vert (\das,\dbs,\dnus) \Vert^4
    |\widehat{\lambda}_n
    -\lambdas|^2 
    \Big] \\&\hspace{4.5cm}
    \lesssim \frac{\log n}{n},\\
    \sup_{(\lambdas,G_*)\in\Xi_{2}(l_n)}&\mathbb{E}_{p_{\lambdas,\Gs}} 
    \Big[
    (\lambdas)^2 
    \Vert (\das,\dbs,\dnus) \Vert^2
    \\&
    \times\Vert (\widehat{a}_n, \widehat{b}_n, \widehat{\nu}_n)-(\as,\bs,\nus) \Vert^2 
    \Big] 
    \lesssim 
    \frac{\log n}{n}.
\end{align*}
\end{theorem}
The proof of Theorem~\ref{theorem:sigma-nonlinear-f0-Gaussian-varphi-nonlinear} is in Appendix~\ref{appendixproof:sigma-nonlinear-f0-Gaussian-varphi-nonlinear}. Above, we involve the set $ \Xi_{2}(l_n) $ to ensure that the MLE $(\hlambdan,\hGn)=(\hlambdan,\han,\hbn,\hnun)$ is a consistent estimator of its ground-truth value $(\lambdas,\Gs)=(\lambdas,\as,\bs,\nus)$. A few comments regarding Theorem~\ref{theorem:sigma-nonlinear-f0-Gaussian-varphi-nonlinear} are in order.

(i) The first bound in Theorem \ref{theorem:sigma-nonlinear-f0-Gaussian-varphi-nonlinear} reveals that 
the convergence rate of $\hlambdan$ becomes slower than $\widetilde{\mathcal{O}}(n^{-1/2})$ and is subject to the vanishing rate of the quantity $\Vert (\das,\dbs,\dnus) \Vert^4$ as a result of the prompt merging issue. Note that $\Vert (\das,\dbs,\dnus) \Vert^4$ cannot approach zero faster than the term $(\Vert\das\Vert^4+|\dbs|^8+|\dnus|^4)$ in Theorem~\ref{theorem:sigma-linear-f0-Gaussian-varphi-linear}. Thus, the convergence rate of $\hlambdan$ in this theorem is technically faster than that in Theorem~\ref{theorem:sigma-linear-f0-Gaussian-varphi-linear}.

(ii) The second bound in Theorem~\ref{theorem:sigma-nonlinear-f0-Gaussian-varphi-nonlinear} suggests that the prompt parameter estimators $\han,\hbn,\hnun$ share the same convergence rate which depend on the vanishing rates of both $\lambdas$ and the term $\Vert (\das,\dbs,\dnus) \Vert^2$. Such convergence behavior is different from that in Theorem~\ref{theorem:sigma-linear-f0-Gaussian-varphi-linear} where the rate of $\hbn$ is only the squared root of those of $\han$ and $\hnun$. Moreover, since $\Vert (\das,\dbs,\dnus) \Vert^2$ cannot vanish faster than its counterpart $(\Vert\das\Vert^2+|\dbs|^4+|\dnus|^2)$ in Theorem~\ref{theorem:sigma-linear-f0-Gaussian-varphi-linear}, we deduce that the convergence rates of $\han,\hbn,\hnun$ are faster than those in Theorem~\ref{theorem:sigma-linear-f0-Gaussian-varphi-linear}. These rate improvements are owing to the disappearance of the parameter interaction in equation~\eqref{eq:heat_equation_1}.

Finally, we establish the minimax lower bounds for estimating true parameters $(\lambdas,\Gs)$ under this section setting to demonstrate that the above MLE convergence rates are optimal.

\begin{theorem}[Minimax lower bounds]
\label{appthm:lower-mle-Gaussian-nonlinear}
    Under the setting of Theorem~\ref{theorem:sigma-nonlinear-f0-Gaussian-varphi-nonlinear}, we have for any $0<r < 1$ and sequence $(l_n)_{n\geq 1}$ that
\begin{align*}
    &\inf_{(\llambdan,\overline{G}_n)\in \Xi }\sup_{(\lambda,G)\in \Xi_{2}(l_n)  }
    \mathbb{E}_{p_{\lambda,G}} \Big[ 
    \Vert (\da,\db,\dnu) \Vert^4\\
    &\hspace{4.4cm}\times|\overline{\lambda}_n
    -\lambda|^2 \Big] 
    \gtrsim n^{-1/r},
    \\
    &\inf_{(\llambdan,\overline{G}_n)\in \Xi }\sup_{(\lambda,G)\in \Xi_{2}(l_n) }
    \mathbb{E}_{p_{\lambda,G}} 
    \Big[ \lambda^2 
    \Vert(\da,\db,\dnu) \Vert^2\\
    &\hspace{2cm}\times\Vert (\overline{a}_n, \overline{b}_n, \overline{\nu}_n)-(a,b,\nu) \Vert^2 \Big] 
    \gtrsim n^{-1/r}.
\end{align*}
\end{theorem}
Proof of Theorem \ref{appthm:lower-mle-Gaussian-nonlinear} is in Appendix \ref{appendix:appthm:lower-mle-Gaussian-nonlinear}.

\begin{figure*}[t]
    \centering
    \subfloat[\textbf{Case (i):} $\lambda^* = 0.5$.\label{fig:thm2-fixed}]{
        \includegraphics[width=1\textwidth]{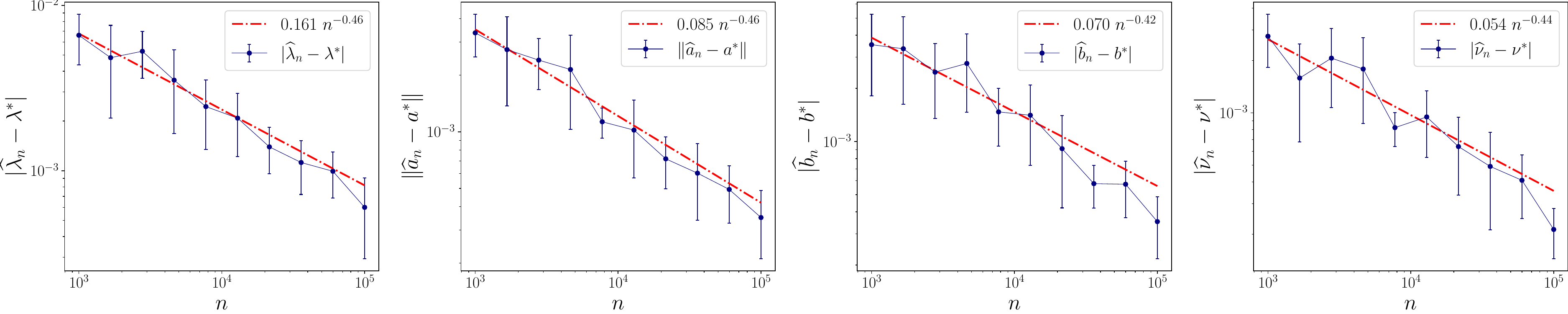}
    }
    
    \subfloat[\textbf{Case (ii):} $\lambda^* = 0.5~n^{-1/4}$.\label{fig:thm2-var}]{
        \includegraphics[width=1\textwidth]{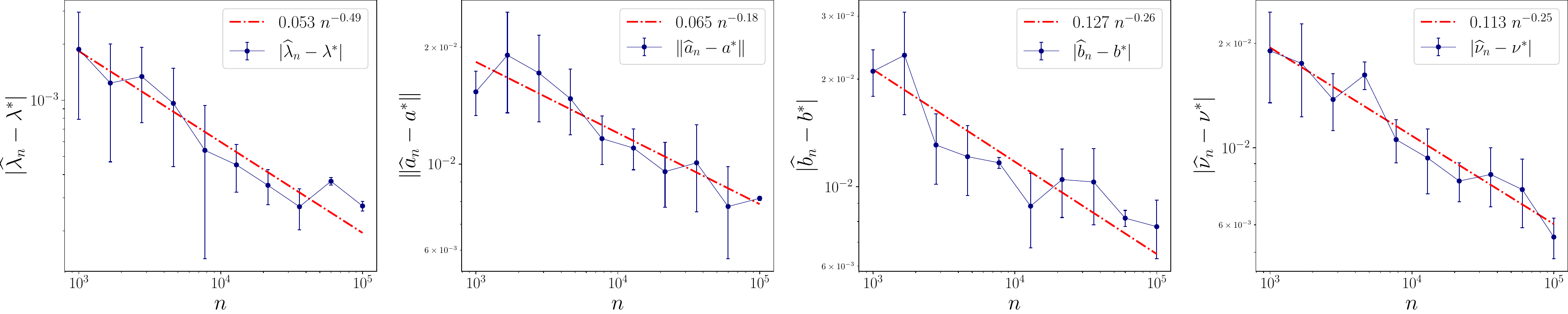}
    }
    \caption{Log-log graphs depicting the empirical convergence rates of the MLE $(\hlambdan,\han,\hbn,\hnun)$ to the ground-truth values $(\lambdas,\as,\bs,\nus)$. 
    The blue lines display the parameter estimation errors, while the red dashed dotted lines are the fitted lines, highlighting the empirical MLE convergence rates. Figure~\ref{fig:thm2-fixed} and Figure~\ref{fig:thm2-var} illustrates results for the cases when $\lambdas = 0.5$ and $\lambdas = 0.5~n^{-1/4}$, respectively.
    }
    \label{fig:both_plots}
\end{figure*}


\section{NUMERICAL EXPERIMENTS}
\label{sec:experiments}
In this section, we conduct numerical experiments to verify the theoretical results presented in Theorem~\ref{theorem:sigma-linear-f0-notGaussian}. Meanwhile, we defer additional experiments for other theorems to Appendix \ref{appendix:DiscussionandAdditionalExperiments}.

\textbf{Experimental setup.} 
In Theorem~\ref{theorem:sigma-linear-f0-notGaussian}, since the
pre-trained model $f_0$ does not belong to the Gaussian density family, we let $f_0$ be the density of a Student's t-distribution, characterized by mean $\varphi(a_0^{\top} X + b_0)=a_0^{\top} X + b_0$ and degrees of freedom $\nu_0 = 4$.
Here, $a_0$ is a $d$-dimension vector $(1, 0, \ldots, 0)$
and $b_0 = 0$. 
Meanwhile, the prompt $f$ is formulated as a Gaussian density with mean $\sigma((\as)^\top X + \bs) = (\as)^\top X + \bs$ and variance $\nu^*$. 


\textbf{Synthetic data generation.} We create synthetic datasets following the model outlined in equation~\eqref{eq:deviated-model}. Specifically, we generate data pairs $\{(X_i, Y_i)\}_{i=1}^n\in\mathcal{X}\times\mathcal{Y} \subset \mathbb{R}^d \times \mathbb{R}$ by initially drawing $X_i$ from a uniform distribution on the interval $[-1, 1]^d$ for $i = 1, \ldots, n$, consistently setting $d=8$ across all trials. Subsequently, $Y_i$ is produced according to the true density $p_{\lambdas,\Gs}(Y|X)$. Here, we consider two following cases of the true parameters $(\lambdas,\Gs) =(\lambdas,\as,\bs,\nus)$ to observe the difference in the MLE convergence behavior when $\lambdas$ is fixed versus when it varies with $n$: 

(i) $\lambdas=0.5,\as=\mathbbm{1}_d,\bs=1,\nus=0.01$;

(ii) $\lambdas=0.5~n^{-1/4},\as=\mathbbm{1}_d,\bs=1,\nus=0.01$.

\textbf{Training procedure.} We conducted $20$ experiments for each sample size $n$, ranging from $10^3$ to $10^5$. 
Prior to each experiment, parameter initialization was carefully set close to the true parameter values to mitigate potential optimization instabilities. 
Each experiment employed an EM algorithm~\citep{Jordan-1994} to derive the MLE, utilizing an SGD optimizer in PyTorch for the maximization step due to the absence of a universal closed-form solution.

\textbf{Results.}
Note that the convergence rates of $\hlambdan$ under both cases are close to $\mathcal{O}(n^{-1/2})$. However, there are differences among those of $(\han,\hbn,\hnun)$ in the two cases. In particular, in Figure~\ref{fig:thm2-fixed} when $\lambdas$ is fixed, they share similar rates of orders roughly $\mathcal{O}(n^{-1/2})$. By contrast, when $\lambdas$ is set to vanish at the rate  $\mathcal{O}(n^{-1/4})$, their rates become substantially slower as shown in Figure~\ref{fig:thm2-var}, which are of orders around $\mathcal{O}(n^{-1/4})$. Those empirical results totally align with Theorem~\ref{theorem:sigma-linear-f0-notGaussian}.


\section{DISCUSSION}
\label{sec:discussion}
In this paper, we carry out the convergence analysis of maximum likelihood parameter estimation under the contaminated MoE model which is the mixture of a frozen pre-trained model and a trainable prompt for learning downstream tasks. Due to the dependence of ground-truth parameters on the sample size, there are several challenges for the theoretical analysis, namely, the prompt vanishing when the mixing proportion tends to zero; the prompt merging when it converges to the pre-trained model; and the parameter interaction via the heat equation. For better understanding, we divide our analysis into multiple scenarios based upon the linear and non-linear structures of the expert functions as well as the distribution family of the pre-trained model. In each scenario, we provide corresponding MLE convergence rates along with their minimax lower bounds to show that those rates are optimal. Additionally, we also conduct experiments to verify our theoretical results. 

\textbf{Practical implications.} From our convergence analysis whose result is summarized in Table~\ref{table:parameter_rates}, we observe that the MLE convergence behaves better when the expert functions associated with the pre-trained model and the prompt do not share the same structure. Since several large-scale pre-trained models often adopt non-linear experts in practice, our theory implies that we can achieve a good performance on downstream tasks even when attaching the prompts with simple linear experts to those models. However, if ones still would like to employ non-linear experts in the prompt, then our analysis suggests that it is better to use different expert functions from those in the pre-trained model.

\textbf{Limitations and Future directions.} A few natural directions arise from our work. Firstly, the current contaminated MoE model considers only one prompt. It is of practical importance to extend the current analysis to the setting where several prompts are incorporated. The convergence rates for parameter estimation under that setting will provide us new insights into reduce the prompt length in fine-tuning large-scale AI models based on merging the close experts. Secondly, the current model assumes that the density $f$ belongs to the Gaussian family. We plan to generalize the current theoretical analysis for general family of density functions $f$. Finally, the gating function (mixing proportion) considered in this paper is input-independent. Thus, we can examine popular input-dependent gating functions, namely softmax gating \citep{shazeer2017topk,nguyen2023demystifying} and sigmoid gating \citep{csordas2023approximating,nguyen2024sigmoid}, in future works.

\section*{Acknowledgements}
NH acknowledges support from the NSF IFML 2019844 and the NSF AI Institute for Foundations of Machine Learning.

\bibliographystyle{apalike}
\bibliography{references}

\newpage
\appendix
\onecolumn
\centering
\textbf{\Large{Supplement to
``Understanding Expert Structures on Minimax Parameter Estimation in Contaminated Mixture of Experts''}}

\justifying
\setlength{\parindent}{0pt}
\vspace{.2in}


In this supplementary material, we provide additional results and proofs. 
At first, Appendix~\ref{appendix:Proofs for the MLE rate} will provide the proofs for convergence rates of parameter estimation.
Appendix \ref{appendix:lower-bounds-proof} focuses on the proofs for the minimax lower bounds for parameter estimation. 
Appendix \ref{appendix:ConvergenceRateofDensityEstimation} contains the general theory for convergence rates of densities (Theorem \ref{theorem:ConvergenceRateofDensityEstimation}) and their corresponding proofs. 
Appendix \ref{appendix:ProofsforAuxiliaryResults} includes auxiliary results and their proof. 
Finally, Appendix \ref{appendix:DiscussionandAdditionalExperiments} illustrates the numerical experiments for justifying the theoretical findings.

\section{PROOFS FOR PARAMETER ESTIMATION RATES}
\label{appendix:Proofs for the MLE rate}



In this section, we will provide the proofs for the MLE rate theorems discussed in Sections \ref{section:linear-sigma} and \ref{section:non-linear-sigma}, focusing on various expert structures.



\subsection{Proof of Theorem \ref{theorem:sigma-linear-f0-notGaussian}}
\label{appendixproof:sigma-linear-f0-notGaussian}

Now we consider the first situation we meet, the expert function $\sigma$ is of linear form and $f_0$ is not a Gaussian density.
For the MLE $(\widehat{\lambda}_n,\widehat{G}_n) = (\widehat{\lambda}_n,\widehat{a}_n,\widehat{b}_n,\widehat{\nu}_n)$ as defined in equation (\ref{eq:MLE}),
we obtain the MLE convergence rates as follows:
\begin{align*}
    \sup_{(\lambdas,G_*)\in\Xi}\mathbb{E}_{p_{\lambdas,\Gs}} 
    \Big[
    |\widehat{\lambda}_n
    -\lambdas|^2 
    \Big] 
    \lesssim \frac{\log n}{n},\\
    \sup_{(\lambdas,G_*)\in\Xi}\mathbb{E}_{p_{\lambdas,\Gs}} 
    \Big[
    (\lambdas)^2 
    \Vert 
    (\widehat{a}_n, \widehat{b}_n, \widehat{\nu}_n)-(\as,\bs,\nus) 
    \Vert^2 
    \Big] 
    \lesssim 
    \frac{\log n}{n}.
\end{align*}

\begin{proof}[Proof of Theorem \ref{theorem:sigma-linear-f0-notGaussian}]
In order to prove the above results, 
for any $(\lambda, G),(\lambda^*, G_*)$, we define a loss
\begin{align}
\label{applossdef:D1-loss}
    D_1((\lambda, G),(\lambda^*, G_*))
    =
    |\lambda-\lambda^*|
    +(\lambda+\lambda^*)
    \|(a,b,\nu)-(a^*,b^*,\nu^*)\|.
\end{align}
Then, we claim that there exists a positive constant $C_1$ that depends on $\Theta$, $f_0$ and satisfies the following for all $(\lambda, G)$ and $(\lambda^*, G_*)$:
\begin{align}
\label{applossineq:D1-loss}
    V(p_{\lambda, G},p_{\lambda^*, G_*})\geq C_1D_1((\lambda, G),(\lambda^*, G_*)).
\end{align}

Then it follows from Theorem \ref{theorem:ConvergenceRateofDensityEstimation}
that:
\begin{align*}
    \mathbb{E}_{\plbgs}\Big( |\widehat{\lambda}_n-\lambdas|+\lambdas \Vert (\widehat{a}_n,\widehat{b}_n,\widehat{\nu}_n)-(\as,\bs,\nus) \Vert \Big)
    &\lesssim 
    \mathbb{E}_{\plbgs}
     V(p_{\widehat{\lambda}_n,\widehat{G}_n}, p_{\lambdas,G_{*}})
     \leq
     \mathbb{E}_{\plbgs}
     h(p_{\widehat{\lambda}_n,\widehat{G}_n}, p_{\lambdas,G_{*}})
     \nonumber\\
     &\lesssim^{Thm\ref{theorem:ConvergenceRateofDensityEstimation}}
\sqrt{\frac{\log n}{n}}.
\end{align*}
Since all inequalities are uniform with respect to $(\lambdas,\Gs)$, we can conclude the theorem.    
\end{proof}


We will now prove the claim stated in equation \eqref{applossineq:D1-loss}. 
To establish the inequality in equation \eqref{applossineq:D1-loss}, 
we require a distinguishability condition between \( f_0 \) and \( f \), which involves the higher-order derivatives of the density.

\begin{lemma}
\label{applemma:sigma-linear-f0-notGaussian}
Under the setting of Theorem \ref{theorem:sigma-linear-f0-notGaussian}, i.e. the expert function $\sigma$ is of linear form and $f_0$ is not a Gaussian density, 
for any component $(a,b,\nu)\in\Theta$ ,
    if we have real coefficients $\eta,\tau_{\alpha}$ for all $\alpha=(\alpha_1,\alpha_2,\alpha_3)\in\mathbb{N}^{d}\times\mathbb{N}\times\mathbb{N}$,
    $|\alpha|=|\alpha_1|+|\alpha_2|+|\alpha_3|\leq1$ such that
    \begin{align}
    \label{eq:first-order-distinguish}
    \eta f_0(Y|\varphi (a_0^{\top}X+b_0 ),\nu_0)
    + 
    \sum_{|\alpha|\leq1} \tau_{\alpha}
    \frac{\partial^{|\alpha|}f}{\partial a^{\alpha_1}\partial b^{\alpha_2}
    \partial \nu^{\alpha_3}}
    (\ysigmaa)=0
    \end{align}
for almost surely $(X,Y)\in\mathcal{X}\times\mathcal{Y}$, then we will have $\eta=\tau_{\alpha}=0$ for all $|\alpha|\leq1$.
\end{lemma}

  

\begin{proof}[Proof of Lemma \ref{applemma:sigma-linear-f0-notGaussian}]
Under the setting of Theorem \ref{theorem:sigma-linear-f0-notGaussian}, equation \eqref{eq:first-order-distinguish} could be expressed as
\begin{align*}
    \eta f_0+\tau_{\alpha_0}f+\sum_{i=1}^d\tau_{\alpha_{1,i}}\dfrac{\partial f}{\partial a_i} +\tau_{\alpha_2}\dfrac{\partial f}{\partial b} 
    +\tau_{\alpha_3}\dfrac{\partial f}{\partial \nu}=
    \eta f_0+\tau_{\alpha_0}f
    +
    \left(
    \sum_{i=1}^d\tau_{\alpha_{1,i}}\dfrac{\partial \sigma}{\partial a_i} 
    +
    \tau_{\alpha_2}\dfrac{\partial \sigma}{\partial b} 
    \right)\dfrac{\partial f}{\partial \sigma}
    +\tau_{\alpha_3}\frac{1}{2}\dfrac{\partial^2 f}{\partial \sigma^2}=0.
\end{align*}
Since $f_0\neq f$ and $\displaystyle\dfrac{\partial^k f}{\partial \sigma^k}, k=0,1,2$ are linear independent (see Lemma \ref{lemma:independent_gaussian}).
The coefficients of $f_0$, $f$, $\displaystyle\dfrac{\partial f}{\partial \sigma}$ and $\displaystyle\dfrac{\partial^2 f}{\partial \sigma^2}$ must be zero.
Now consider the coefficient for $\displaystyle\dfrac{\partial f}{\partial \sigma}$:
\begin{align}
\label{appeq:thm2-lemma1}
    \sum_{i=1}^d\tau_{\alpha_{1,i}}\dfrac{\partial \sigma}{\partial a_i} 
    +
    \tau_{\alpha_2}\dfrac{\partial \sigma}{\partial b} 
    =0 
    \quad\Longleftrightarrow\quad
    \left(
    \tau_{\alpha_1}^{\top}X+\tau_{\alpha_2}
    \right)
    \sigma^{\prime}(a^{\top}X+b)
    =0.
\end{align}
Here we consider $\sigma$ is a linear function
and $\sigma^{\prime}\neq0$
in a non-zero measure set in $\mathbb{R}$,
so only when $\tau_{\alpha_1}=0$ and $\tau_{\alpha_2}=0$, 
equation \eqref{appeq:thm2-lemma1}  will occur, which means that
$
\eta = \tau_{\alpha} = 0 \text{ for all } |\alpha| \leq 1
$.
\end{proof}


\begin{proof}[Proof of the claim in equation \eqref{applossineq:D1-loss}]
Now we will prove the inequality in equation \eqref{applossineq:D1-loss}.
Let $\overline{G}=(\Bar{a},\Bar{b},\Bar{\nu})$, we need to demonstrate that
 \begin{align*}
     \lim_{\varepsilon\to 0}\inf_{(\lambda,G),(\lambda^*,G_*)}\left\{\frac{V(p_{\lambda, G},p_{\lambda^*, G_*})}{D_1((\lambda,G),(\lambda^*,G_*))}: 
     D_1((\lambda,G),(\overline{\lambda},\overline{G}))\vee D_1((\lambda^*, G_*),(\overline{\lambda},\overline{G}))\leq\varepsilon\right\}>0.
 \end{align*}
Using the argument with Fatou's lemma as in Theorem 3.1, \citep{Ho-Nguyen-EJS-16} , it is sufficient to show that
\begin{align*}
     \lim_{\varepsilon\to 0}\inf_{(\lambda,G),(\lambda^*,G_*)}\left\{\frac{\|p_{\lambda, G}-p_{\lambda^*, G_*}\|_{\infty}}{D_1((\lambda,G),(\lambda^*,G_*))}:
     D_1((\lambda,G),(\overline{\lambda},\overline{G}))\vee D_1((\lambda^*, G_*),(\overline{\lambda},\overline{G}))\leq\varepsilon\right\}>0.
 \end{align*}
 Assume by contrary that the above claim is not true. 
Then, there exist two sequences $G_n=(a_n,b_{n},\nu_{n})$ and $G_{*,n}=(a_n^*, b^*_{n},\nu^*_{n})$, and two sequences of mixing proportions $\lambda_n$ and $\lambda_n^* \in [0, 1]$
 such that when $n$ tends to infinity, we get
 \begin{align*}
     \begin{cases}
        D_1((\lambda_n,G_n),(\overline{\lambda},\overline{G})) \to 0 ,\\
        D_1((\lambda^*_n,G_{*,n}),(\overline{\lambda},\overline{G})) \to 0 ,\\
        {\|p_{\lambda_n,G_n}-p_{\lambda_n^*,G_{*,n}}\|_{\infty}}/D_1((\lambda_n,G_n),(\lambda_n^*,G_{*,n}))\to 0.
     \end{cases}
 \end{align*}
 In this proof, we will take into account only the most challenging setting of $G_n=(a_n, b_n, \nu_n)$ and $G^*_n=(a^*_n, b^*_n, \nu^*_n)$ when they converge to the same limit point $(a',b', \nu')$, where $(a',b', \nu')$ is not necessarily equal to $(\Bar{a},\Bar{b}, \Bar{\nu} )$ as $\lambda_n, \lambda_n^* \in [0, 1]$ can go to $0$ or $1$ in the limit. 
 Next, we consider
 \begin{align*}
     p_{\lambdan, \Gn}(X,Y)-p_{\lambdasn, \Gsn}(X,Y)
     &=(\lambdasn-\lambdan)
     [f_0(\yphizero)
     -
     f(Y|\sigma((a_n^*)^{\top}X+b_n^*),\nu_n^*)
     ]\Bar{f}(X)\\
     &~+\lambdan[
     f(Y|\sigma((a_n)^{\top}X+b_n),\nu_n)-f(Y|\sigma((a_n^*)^{\top}X+b_n^*),\nu_n^*)
     ] \Bar{f}(X).
 \end{align*}
 By applying the Taylor expansion up to the first order, 
 we get
  \begin{align*}
     &f(Y|\sigma((a_n)^{\top}X+b_n),\nu_n)-f(Y|\sigma((a_n^*)^{\top}X+b_n^*),\nu_n^*)\\
     &=\left[
     \sum_{u=1}^{d}(a_n-a_n^*)^{(u)}
     \frac{\partial f}{\partial a^{(u)}}
     +(b_n-b_n^*)^{}
     \frac{\partial f}{\partial b^{}}
     +(\nu_n-\nu_n^*)^{}
     \frac{\partial f}{\partial \nu^{}}
     \right]
     (
     Y|\sigma((a_n^*)^{\top}X+b_n^*),\nu_n^*)
     +R(X,Y),
 \end{align*}
where 
 $R(X,Y)=\bigO(\|(a_n, b_n, \nu_n)-(a^*_n, b^*_n, \nu^*_n)\|^{1+\gamma})$ 
 for some $\gamma>0$ is the remainder from the Taylor expansion.
 Since $f$ is the p.d.f of the univariate location-scale Gaussian distribution, we have the following partial differential equation (PDE):
 \begin{align*}
     \frac12
     \frac{\partial^2f}{\partial \sigma^2}
     (Y|\sigma((a_n^*)^{\top}X+b_n^*),\nu_n^*)
     =
     \frac{\partial f}{\partial \nu}
     (Y|\sigma((a_n^*)^{\top}X+b_n^*),\nu_n^*
     ).
 \end{align*}
 Thus, the difference $p_{\lambdan,\Gn}(X,Y)-p_{\lambdasn,\Gsn}(X,Y)$ is reduced to
 \begin{align}
    \label{eq:distinguishable_independent_taylor}
     &p_{\lambdan,\Gn}(X,Y)-p_{\lambdasn,\Gsn}(X,Y)
     =(\lambdasn-\lambdan)[f_0(\yphizero)-f(Y|\sigma((a_n^*)^{\top}X+b_n^*),\nu_n^*)]\Bar{f}(X)
     \nonumber\\
     &
     +\lambdan
     \Big[
     \sum_{u=1}^{d}(a_n-a_n^*)^{(u)}
     \frac{\partial \sigma}{\partial a^{(u)}}(X, a_n^*, b_n^*)
     \frac{\partial f}{\partial \sigma}
     (
     Y|\sigma((a_n^*)^{\top}X+b_n^*),\nu_n^*)
     \\&
     +(b_n-b_n^*)^{}
     \frac{\partial \sigma}{\partial b}(X, a_n^*, b_n^*)
     \frac{\partial f}{\partial \sigma }
     (
     Y|\sigma((a_n^*)^{\top}X+b_n^*),\nu_n^*)
     \\&
     +
     \frac12
     (\nu_n-\nu_n^*)^{}
     \frac{\partial^2 f}{\partial \sigma^2 }
     (
     Y|\sigma((a_n^*)^{\top}X+b_n^*),\nu_n^*)
     +R(X,Y)
     \Big]
     \Bar{f}(X).
 \end{align}
 From the formulation of $\done$, it is obvious that
\begin{align*}
     R(X,Y)\Bar{f}(X)/\done=
     \bigO(\|(a_n, b_n, \nu_n)-(a^*_n, b^*_n, \nu^*_n)\|^{\gamma})\to 0
 \end{align*}
 as $n\to\infty$. Therefore, the quantity $[p_{\lambdan,\Gn}(X,Y)-p_{\lambdasn,\Gsn}(X,Y)]/\done$ can be seen as a linear combination of elements in the following set
  \begin{align*}
     \mathcal{S}=\Big\{
     &f_0(\yphizero)\Bar{f}(X),
     f(Y|\sigma((a_n^*)^{\top}X+b_n^*,\nu_n^*))\Bar{f}(X),
     \\&
     \frac{\partial \sigma}{\partial a^{(u)}}(X, a_n^*, b_n^*)
     \frac{\partial f}{\partial \sigma}
     (
     Y|\sigma((a_n^*)^{\top}X+b_n^*,\nu_n^*))
     \Bar{f}(X),
     \\&
     \frac{\partial \sigma}{\partial b}(X, a_n^*, b_n^*)
     \frac{\partial f}{\partial \sigma }
     (
     Y|\sigma((a_n^*)^{\top}X+b_n^*,\nu_n^*))
     \Bar{f}(X),
     \\&
     \frac{\partial^2 f}{\partial \sigma^2 }
     (
     Y|\sigma((a_n^*)^{\top}X+b_n^*,\nu_n^*))
     \Bar{f}(X)
     :u\in[d] \Big\}.
 \end{align*}
 Assume that the coefficients associated with those terms go to 0 as $n\to\infty$. Thus, we obtain that
\begin{align*}
    \frac{\lambdasn-\lambdan}{\done}\to 0, 
    \quad\frac{\lambdan(a_n-a_n^*)^{(u)}}{\done}\to 0, \\
    \quad\frac{\lambdan(b_n-b_n^*)^{}}{\done}\to 0,
    \quad\frac{\lambdan(\nu_n-\nu_n^*)^{}}{\done}\to 0,
\end{align*}
as $n\to\infty$ for all $u\in[d]$. 
The last three limits imply that
\begin{align*}
    (\lambdan+\lambdasn)\|(a_n,b_n,\nu_n)-(a^*_n,b^*_n,\nu^*_n)\|/\done\to 0.
\end{align*}
Consequently, we achieve that
\begin{align*}
    1=\left[|\lambdasn-\lambdan|+(\lambdasn+\lambdan)
    \|
    (a_n,b_n,\nu_n)-(a^*_n,b^*_n,\nu^*_n)
    \|\right]/\done\to 0,
\end{align*}
which is a contradiction. Thus, not all the coefficients of elements in the set $\mathcal{S}$ tend to 0 as $n\to\infty$. Let us denote by $m_n$ the maximum of the absolute values of those coefficients. It follows from the previous result that $1/m_n\not\to\infty$. Therefore, we get
\begin{align}
    \label{eq:distinguishable_independent_fatou}
    &\frac{1}{m_n}\frac{p_{\lambdan,\Gn}(X,Y)-p_{\lambdasn,\Gsn}(X,Y)}{\done}
    \\&\to\eta 
    f_0(X|\yphizero)+
    \sum_{\ell=0}^2\alpha_{\ell}(X)
    \frac{\partial^\ell f}{\partial \sigma^{\ell}}
    (Y|
    \sigma((a'_n)^{\top}X+b'_n),\nu'_n
    )=0, 
\end{align}
where
\begin{align*}
    \eta&=\lim_{n\to\infty}\frac{1}{m_n}\frac{(\lambdasn-\lambdan)}{\done},
    \quad 
    \alpha_0(X)=\lim_{n\to\infty}\frac{1}{m_n}\frac{-(\lambdasn-\lambdan)}{\done},\\
    \alpha_1(X)&=
    \lim_{n\to\infty}
    \left[
    \sum_{u=1}^{d}\frac{1}{m_n}
    \frac{\lambdan(a_n-a_n^*)^{(u)}}{\done}
    \frac{\partial \sigma}{\partial a ^{(u)}}(X,a',b')
    +
    \frac{1}{m_n}
    \frac{\lambdan(b_n-b_n^*)^{}}{\done}
    \frac{\partial \sigma}{\partial b ^{}}(X,a',b')
    \right]
    ,\\
    \alpha_2(X)&=
    \lim_{n\to\infty}
    \frac{1}{2 m_n}\frac{\lambdan(\nu_n-\nu_n^*)^{}}{\done}
    .
\end{align*}
Recall that $f$ is distinguishable from $f_0$,
we obtain that $\eta=\alpha_0(X)=\alpha_1(X)=\alpha_2(X)=0$ for almost surely $X$. 
Since $\sigma(a^{\top} X+b)$ and $\nu$ are algebraically independent, 
this result implies that 
$$
\lim_{n\to\infty}\frac{\lambdan(a_n-a_n^*)^{(u)}}{\done} =
\lim_{n\to\infty}\frac{\lambdan(b_n-b_n^*)^{}}{\done} =
\lim_{n\to\infty}\frac{\lambdan(\nu_n-\nu_n^*)^{}}{\done} =
0$$ 
for all $u\in[d]$. 
This contradicts the fact that not all the coefficients of elements in the set $\mathcal{S}$ vanish as $n\to\infty$.
Hence, we reach the desired conclusion.
\end{proof}


\subsection{Proof of Theorem~\ref{theorem:sigma-linear-f0-Gaussian-varphi-nonlinear}}
\label{appendixproof:sigma-linear-f0-Gaussian-varphi-nonlinear}
In this section, we draw our attention to the scenario where $f_0$ is a Gaussian density, the expert $\sigma$ is of linear form, while $\varphi$ is a non-linear expert. 
\begin{proof}[Proof of Theorem~\ref{theorem:sigma-linear-f0-Gaussian-varphi-nonlinear}]
As discussed in Section \ref{section:sigma-linear-f0-Gaussian-varphi-nonlinear}, the expert model $\sigma((\as)^{\top}X+\bs)$ cannot converge to its counterpart $\varphi(a_0^{\top}X+b_0)$ due to the inherent differences in their structures. Consequently, the prompt $f(Y|\sigma((\as)^{\top}X+\bs),\nus)$ won't converge to the pre-trained model $f_0(Y|\varphi(a_0^{\top}X+b_0),\nu_0)$.
Following the same reasoning as in Lemma \ref{applemma:sigma-linear-f0-notGaussian}, we can validate that
for any component $(a,b,\nu)\in\Theta$ ,
    if we have real coefficients $\eta,\tau_{\alpha}$ for all $\alpha=(\alpha_1,\alpha_2,\alpha_3)\in\mathbb{N}^{d}\times\mathbb{N}\times\mathbb{N}$,
    $|\alpha|=|\alpha_1|+|\alpha_2|+|\alpha_3|\leq1$ such that
$$
    \eta f_0(Y|\varphi (a_0^{\top}X+b_0 ),\nu_0)
    + 
    \sum_{|\alpha|\leq1} \tau_{\alpha}
    \frac{\partial^{|\alpha|}f}{\partial a^{\alpha_1}\partial b^{\alpha_2}
    \partial \nu^{\alpha_3}}
    (\ysigmaa)=0
$$
for almost surely $(X,Y)\in\mathcal{X}\times\mathcal{Y}$, then we will have $\eta=\tau_{\alpha}=0$ for all $|\alpha|\leq1$.
Building on this, and similar to the discussion in the proof of the claim in equation \eqref{applossineq:D1-loss}, we can conclude that
there exists a positive constant $C_1$ that depends on $\Theta$, $f_0$ and satisfies the following for all $(\lambda, G)$ and $(\lambda^*, G_*)$:
\begin{align*}
    C_1
    \left(
    |\lambda-\lambda^*|
    +(\lambda+\lambda^*)
    \|(a,b,\nu)-(a^*,b^*,\nu^*)\|
    \right)
    \leq
    V(p_{\lambda, G},p_{\lambda^*, G_*}).
\end{align*}
Recall Theorem \ref{theorem:ConvergenceRateofDensityEstimation},
we will have
that:
\begin{align*}
    \mathbb{E}_{\plbgs}\Big( |\widehat{\lambda}_n-\lambdas|+\lambdas \Vert (\widehat{a}_n,\widehat{b}_n,\widehat{\nu}_n)-(\as,\bs,\nus) \Vert \Big)
     \lesssim
\sqrt{\frac{\log n}{n}}.
\end{align*}
Since all inequalities hold uniformly with respect to $(\lambdas, \Gs)$, we can thus conclude the theorem.
\end{proof}

\subsection{Proof of Theorem~\ref{theorem:sigma-nonlinear-f0-notGaussian}}
\label{appendixproof:sigma-nonlinear-f0-notGaussian}


In this section, we focus on the scenario where $f_0$ does not belong to the family of Gaussian densities and the expert $\sigma$ takes a non-linear form.

\begin{proof}[Proof of Theorem~\ref{theorem:sigma-nonlinear-f0-notGaussian}]

Firstly, when $f_0$ does not belong to the family of Gaussian densities, $f_0\neq f$ and ${\partial^k f}/{\partial \sigma^k}, k=0,1,2$ are linear independent. 
When equation \eqref{eq:first-order-distinguish} holds, we will have $\eta=\tau_{\alpha_0}=\tau_{\alpha_3}=0$ and
$\left(
    \tau_{\alpha_1}^{\top}X+\tau_{\alpha_2}
    \right)
    \sigma^{\prime}(a^{\top}X+b)
    =0$.
Since $\sigma^{\prime}(a^{\top}X + b)$ is a non-linear function, it is natural to assume that the first derivative of $\sigma$ is non-zero on a set of non-zero measure in $\mathbb{R}$, which serves as our global assumption.
Then we can verify that $f$ is distinguishable from $f_0$ as in Lemma \ref{applemma:sigma-linear-f0-notGaussian}.
Based on this distinguishability condition, following the same argument in Appendix \ref{appendixproof:sigma-linear-f0-notGaussian} we will have $\forall (\lambda, G),(\lambda^*, G_*)\in\Xi$,
\begin{align*}
    |\lambda-\lambda^*|
    +(\lambda+\lambda^*)
    \|(a,b,\nu)-(a^*,b^*,\nu^*)\|
    \lesssim
    V(p_{\lambda, G},p_{\lambda^*, G_*}).
\end{align*}
Recall Theorem \ref{theorem:ConvergenceRateofDensityEstimation},
we obtain the following:
\begin{align*}
    \mathbb{E}_{\plbgs}\Big( |\widehat{\lambda}_n-\lambdas|+\lambdas \Vert (\widehat{a}_n,\widehat{b}_n,\widehat{\nu}_n)-(\as,\bs,\nus) \Vert \Big)
     \lesssim
\sqrt{\frac{\log n}{n}}.
\end{align*}
Since all inequalities hold uniformly with respect to $(\lambdas, \Gs)$, we can thus conclude the theorem.
\end{proof}


\subsection{Proof of Theorem~\ref{theorem:sigma-nonlinear-f0-Gaussian-varphi-linear}}
\label{appendixproof:sigma-nonlinear-f0-Gaussian-varphi-linear}


In this section, we concentrate on the scenario where $f_0$ and $f$ are both Gaussian distribution and the expert $\sigma$ takes a non-linear form, while $\varphi$ is a linear expert function. 

\begin{proof}[Proof of Theorem~\ref{theorem:sigma-nonlinear-f0-Gaussian-varphi-linear}]
In fact, we can view this scenario as the opposite of Appendix \ref{appendixproof:sigma-linear-f0-Gaussian-varphi-nonlinear}, where the expert $\sigma$ is linear and its counterpart $\varphi$ is non-linear. Both scenarios share a common feature: the distinction in the structure of the experts. As a result, the prompt $f(Y|\sigma((\as)^{\top}X + \bs), \nus)$ does not converge to the pre-trained model $f_0(Y|\varphi(a_0^{\top}X + b_0), \nu_0)$.
Therefore, following the reasoning in Appendix \ref{appendixproof:sigma-linear-f0-Gaussian-varphi-nonlinear} and assuming that the first derivative of the non-linear function $\sigma^{\prime}(a^{\top}X + b)$ is non-zero on a set of non-zero measure in $\mathbb{R}$, as our global assumption, we can conclude the theorem.

A more detailed explanation is that this situation guarantees that $f$ is distinguishable from $f_0$ as in Lemma \ref{applemma:sigma-linear-f0-notGaussian}. 
Then the claim in equation \eqref{applossineq:D1-loss} will hold, i.e.
$\forall (\lambda, G),(\lambda^*, G_*)\in\Xi$,
\begin{align*}
    |\lambda-\lambda^*|
    +(\lambda+\lambda^*)
    \|(a,b,\nu)-(a^*,b^*,\nu^*)\|
    \lesssim
    V(p_{\lambda, G},p_{\lambda^*, G_*}).
\end{align*}
Thus, we can conclude the theorem, as all inequalities hold uniformly with respect to $(\lambdas, \Gs)$, in conjunction with Theorem \ref{theorem:ConvergenceRateofDensityEstimation}:
\begin{align*}
    \sup_{(\lambdas,G_*)\in\Xi}
    \mathbb{E}_{p_{\lambdas,\Gs}} 
    \Big(
    |\widehat{\lambda}_n
    -\lambdas| 
    \Big)
    \lesssim \sqrt{\frac{\log n}{n}},
    \quad\text{and}\quad
    \sup_{(\lambdas,G_*)\in\Xi}
    \mathbb{E}_{p_{\lambdas,\Gs}} 
    \Big(
    \lambdas
    \Vert 
    (\widehat{a}_n, \widehat{b}_n, \widehat{\nu}_n)-(\as,\bs,\nus) 
    \Vert
    \Big)
    \lesssim 
    \sqrt{\frac{\log n}{n}}.
\end{align*}
\end{proof}

\subsection{Proof of Theorem~\ref{theorem:sigma-nonlinear-f0-Gaussian-varphi-nonlinear}}
\label{appendixproof:sigma-nonlinear-f0-Gaussian-varphi-nonlinear}


In this section, we consider the scenario where both \( f_0 \) and \( f \) belong to the family of Gaussian densities, and both expert functions, \( \varphi \) and \( \sigma \), are non-linear. 
Building on the explanation in Section \ref{section:sigma-nonlinear-f0-Gaussian-varphi-nonlinear}, we will focus on the challenging case where \( \varphi = \sigma \)  and they are twice differentiable almost everywhere, as in this situation, the prompt might not be distinguishable from the pre-trained model.
Recall for the simplicity of the presentation in the paper, we define $\Delta G=(\Delta a,\Delta b,\Delta\nu) = (a - a_0,b - b_0,\nu-\nu_0 )$ and $\Delta \Gs=(\Delta \as,\Delta \bs,\Delta\nus) = (\as - a_0,\bs - b_0,\nus-\nu_0 )$ for any element $(a,b,\nu), (\as,\bs,\nus)\in \Theta$.
For any sequence $(l_n)_{n\geq 1}$, we denote
\begin{align*}
    \Xi_{2}(l_n):=
\Big\{ 
(\lambda,G)={(\lambda,a,b,\nu)}\in\Xi:
\frac{l_n}{\min_{1\leq i\leq d}\{ 
|(\da)_i|^2,|\db|^2,|\dnu|^2
\}\sqrt{n}}\leq\lambda
\Big\}.
\end{align*} 
Then, from Theorem~\ref{theorem:sigma-nonlinear-f0-Gaussian-varphi-nonlinear} the followings will hold for any sequence $(l_n)_{n\geq 1}$ such that $l_n/\log n\to\infty$ as $n\to\infty$:
\begin{align*}
    \sup_{(\lambdas,G_*)\in\Xi_{2}(l_n)}\mathbb{E}_{p_{\lambdas,\Gs}} 
    \Big[
    \Vert (\das,\dbs,\dnus) \Vert^4
    |\widehat{\lambda}_n
    -\lambdas|^2 
    \Big]
    \lesssim \frac{\log n}{n},\\
    \sup_{(\lambdas,G_*)\in\Xi_{2}(l_n)}\mathbb{E}_{p_{\lambdas,\Gs}} 
    \Big[
    (\lambdas)^2 
    \Vert (\das,\dbs,\dnus) \Vert^2
    \times\Vert (\widehat{a}_n, \widehat{b}_n, \widehat{\nu}_n)-(\as,\bs,\nus) \Vert^2 
    \Big] 
    \lesssim 
    \frac{\log n}{n}.
\end{align*}
To prove the above results, 
drawing from Proposition \ref{prop:sigma-nonlinear-f0-Gaussian-varphi-nonlinear}, we can establish the following result concerning the lower bound of $V(p_{\lambda, G}, p_{\lambda^*, G_*})$.
\begin{lemma}
    \label{applemma:sigma-nonlinear-f0-Gaussian-varphi-nonlinear}
Assume that $f_0$ and $f$ are both Gaussians, and the mean experts $\varphi$ in $f_0$ and $\sigma$ in $f$ are both non-linear functions.
Futhermore, if $f_0(Y|\varphi(a_0^{\top}X+b_0),\nu_0)=f(Y|\sigma(a_0^{\top}X+b_0),\nu_0)$ for some $(a_0,b_0,\nu_0)\in\Theta$.
    For any $(\lambda, G),(\lambda^*, G_*)$, we define
    \begin{align*}
        &D_2((\lambda, G),(\lambda^*, G_*))
        \\&
        =\lambda\|(\Delta a,\Delta b, \Delta\nu)\|^2+\lambda^*\|(\Delta a^*,\Delta b^*, \Delta\nu^*)\|^2
        -\min\{\lambda,\lambda^*\}
        \left(
        \|(\Delta a,\Delta b, \Delta\nu)\|^2+\|(\Delta a^*,\Delta b^*, \Delta\nu^*)\|^2
        \right)
        \\
        &+
        \left(
        \lambda\|(\Delta a,\Delta b, \Delta\nu)\|+\lambda^*\|(\Delta a^*,\Delta b^*, \Delta\nu^*)\|
        \right)
        \times\|(\Delta a-\Delta a^*,\Delta b-\Delta b^*,\Delta\nu-\Delta\nu^*)\|.
    \end{align*}
    Then, there exists a positive constant $C_2$ that depends on $\Theta$, 
    $(a_0,b_0,\nu_0)$
    and satisfies 
\begin{align*}
    V(p_{\lambda, G},p_{\lambda^*, G_*})\geq C_2D_2((\lambda, G),(\lambda^*, G_*)).
\end{align*}
    for all $(\lambda, G)$ and $(\lambda^*, G_*)$.
\end{lemma}

\begin{proof}[Proof of Theorem~\ref{theorem:sigma-nonlinear-f0-Gaussian-varphi-nonlinear}]
First we denote $\overline{D_2}((\lambda,G),(\lambdas,\Gs))$ by
\begin{align*}
    \dtwobarl &:=|\lambdas-\lambda| 
    \Vert (\da,\db,\dnu) \Vert
    \Vert (\das,\dbs,\dnus) \Vert
    \\&
    +\Vert  (\da,\db,\dnu)-(\das,\dbs,\dnus) \Vert
    \big( 
    \lambda\Vert (\da,\db,\dnu) \Vert
    +\lambdas\Vert (\das,\dbs,\dnus) \Vert
    \big)
\end{align*}
It's easily to verify that 
$D_2((\lambda,G),(\lambdas,G_{*})) \asymp 
\overline{D_2}((\lambda,G),(\lambdas,G_{*}))$.
By leveraging this result together with
Theorem \ref{theorem:ConvergenceRateofDensityEstimation} and
Lemma \ref{applemma:sigma-nonlinear-f0-Gaussian-varphi-nonlinear}
, we immediately establish the following convergence rates
\begin{align}
    \label{eq:second_theorem:sigma-nonlinear-f0-Gaussian-varphi-nonlinear}
     \sup_{{\lambdas,\Gs} \in \Xi}\mathbb{E}_{p_{{\lambdas,\Gs}}}\left(\Vert (\dhan,\dhbn,\dhnun)\Vert ^2 \Vert(\Delta a^*,\Delta b^*, \Delta \nu^*)\Vert ^2 |\hlambdan-\lambdas|^2\right) &\lesssim \dfrac{\log n}{n},\\
    \label{eq:first_theorem:sigma-nonlinear-f0-Gaussian-varphi-nonlinear}
    \sup_{{\lambdas,\Gs} \in \Xi}\mathbb{E}_{p_{{\lambdas,\Gs}}}\left((\lambdas)^2\Vert (\Delta a^*,\Delta b^*, \Delta \nu^*)\Vert ^2 \Vert (\dhan,\dhbn,\dhnun) -  (\Delta a^*,\Delta b^*, \Delta \nu^*)\Vert^2\right) &\lesssim \dfrac{\log n}{n}.
\end{align}
Although the result in equation \eqref{eq:second_theorem:sigma-nonlinear-f0-Gaussian-varphi-nonlinear} does not align with the first formulae of Theorem \ref{theorem:sigma-nonlinear-f0-Gaussian-varphi-nonlinear}, we can bypass this issue using the fact that $(\lambdas,\as,\bs,\nus) \in \Xi_{2}(l_n)$. In fact, observe that 
$(\han,\hbn,\hnun)-(\as, \bs, \nus) = (\dhan,\dhbn,\dhnun)-(\Delta\as, \Delta\bs, \Delta\nus)$, applying the Cauchy-Schwartz inequality, we have 
\begin{align*}
    \mathbb{E}_{p_{{\lambdas,\Gs}}}\left(\Vert (\das,\dbs,\dnus)\Vert ^4 |\hlambdan-\lambdas|^2\right) 
    \lesssim  \mathbb{E}_{p_{{\lambdas,\Gs}}}\left(\Vert (\das,\dbs,\dnus)\Vert ^2\Vert (\dhan,\dhbn,\dhnun)\Vert ^2 |\hlambdan-\lambdas|^2\right)\\
    +\mathbb{E}_{p_{{\lambdas,\Gs}}}\left(\Vert (\das,\dbs,\dnus)\Vert ^2\Vert (\dhan,\dhbn,\dhnun) -  (\Delta a^*,\Delta b^*, \Delta \nu^*)\Vert^2|\hlambdan-\lambdas|^2\right)
\end{align*}
The first term in the right-hand-side has been estimated in equation \eqref{eq:second_theorem:sigma-nonlinear-f0-Gaussian-varphi-nonlinear}. For the second term, if $\lambdas \neq 0$ it is in fact an implication of equation \eqref{eq:first_theorem:sigma-nonlinear-f0-Gaussian-varphi-nonlinear}, noting that 
$|\hlambdan-\lambdas|/\lambdas$ is uniformly bounded by $2/\lambdas$.
Otherwise, if $\lambdas = 0$, we have 
    \begin{align*}
    D_2((\hlambdan, \hGn),(\lambda^*, G_*))
    &=\hlambdan\|(\dhan,\dhbn, \dhnun)\|^2
    +\hlambdan\|(\dhan,\dhbn, \dhnun)\|\|(\dhan-\Delta a^*,\dhbn-\Delta b^*,\dhnun-\Delta\nu^*)\|\\
    &=\hlambdan \left( \|(\dhan,\dhbn, \dhnun)\|^2 +\|(\dhan,\dhbn, \dhnun)\|\|(\dhan-\Delta a^*,\dhbn-\Delta b^*,\dhnun-\Delta\nu^*)\|\right)
    \\&\geq \hlambdan\|(\dhan,\dhbn, \dhnun)\|\|(\dhan, \dhbn, \dhnun)-(\Delta a^*,\Delta b^*,\Delta\nu^*)\|.
    \end{align*}
The conclusion follows from direct application of Lemma \ref{applemma:sigma-nonlinear-f0-Gaussian-varphi-nonlinear} and Theorem \ref{theorem:ConvergenceRateofDensityEstimation}.
\end{proof}

Now we will prove Lemma \ref{applemma:sigma-nonlinear-f0-Gaussian-varphi-nonlinear}.
\begin{proof}[Proof of Lemma \ref{applemma:sigma-nonlinear-f0-Gaussian-varphi-nonlinear}]
Let $\overline{G}=(\Bar{a},\Bar{b},\Bar{\nu})$ such that $\Bar{\lambda}\in[0,1]$ and $(\Bar{a},\Bar{b},\Bar{\nu})$ can be identical to $({a_0},{b_0},{\nu_0})$. 
Then, we will show that
\begin{itemize}
    \item[(i)] When 
    $({a_0},{b_0},{\nu_0})\neq(\Bar{a},\Bar{b},\Bar{\nu})$ 
    and $\Bar{\lambda}>0$, 
    \begin{align*}
     \lim_{\varepsilon\to 0}\inf_{(\lambda,G),(\lambdas,\Gs)}\left\{\frac{\|p_{\lambda, G} - p_{\lambda^*, G_*}\|_{\infty}}{D_1((\lambda,G),(\lambda^*,G_*))}:D_1((\lambda,G),(\overline{\lambda},\overline{G}))\vee D_1((\lambda^*, G_*),(\overline{\lambda},\overline{G}))\leq\varepsilon\right\}>0.
    \end{align*}
    \item[(ii)] 
    When $({a_0},{b_0},{\nu_0})=(\Bar{a},\Bar{b},\Bar{\nu})$ or $({a_0},{b_0},{\nu_0})\neq(\Bar{a},\Bar{b},\Bar{\nu})$ and $\Bar{\lambda}=0$, 
    \begin{align}
    \label{eq:claim_nondistinguishable_independent}
     \lim_{\varepsilon\to 0}\inf_{(\lambda,G),(\lambdas,\Gs)}\left\{\frac{\|p_{\lambda, G} - p_{\lambda^*, G_*}\|_{\infty}}{D_2((\lambda,G),(\lambda^*,G_*))}:D_2((\lambda,G),(\overline{\lambda},\overline{G}))\vee D_2((\lambda^*, G_*),(\overline{\lambda},\overline{G}))\leq\varepsilon\right\}>0.
    \end{align}
\end{itemize}
 Part (i) can be proved by using the same arguments as in the proof of Theorem~\ref{theorem:sigma-linear-f0-notGaussian}. Thus, we will consider only part (ii) in this section, specifically the most challenging setting that $({a_0},{b_0},{\nu_0})=(\Bar{a},\Bar{b},\Bar{\nu})$. 
Under this assumption, we know that $\varphi$ and $\sigma$ are the same nonlinear function, s.t. 
$f_0(Y|\varphi (a_0^{\top}X+b_0),\nu_0)=f(Y|\sigma (a_0^{\top}X+b_0 ),\nu_0)$ for all most surly $(X,Y)\in \mathcal{X}\times\mathcal{Y}$.
Assume that the above claim in equation~\eqref{eq:claim_nondistinguishable_independent} does not hold, 
then there exist two sequences $G_n=(a_n,b_{n},\nu_{n})$ and $G_{*,n}=(a_n^*, b^*_{n},\nu^*_{n})$, and two sequences of mixing proportions $\lambda_n$ and $\lambda_n^* \in [0, 1]$
 such that
\begin{align*}
    \begin{cases}
        D_2((\lambdan,\Gn),(\overline{\lambda},\overline{G}))=\lambdan\|(\Delta\an, \Delta\bn,\Delta\nun)\|^2\to 0,\\
        D_2((\lambdasn,\Gsn),(\overline{\lambda},\overline{G}))=\lambdasn\|(\Delta\asn, \Delta\bsn,\Delta\nusn)\|^2\to 0,
        \\ {\|p_{\lambdan, \Gn} - p_{\lambdasn, \Gsn}\|_{\infty}}/{D_2((\lambdan,\Gn),(\lambdasn,\Gsn))}\to 0.
    \end{cases}
\end{align*}
To facilitate the presentation, let us denote
\begin{align*}
    A_n&=\|(\Delta\an, \Delta\bn,\Delta\nun)\|,\\
    B_n&=\|(\Delta\asn, \Delta\bsn,\Delta\nusn)\|,\\
    C_n&=\|(\Delta\an-\Delta\asn, \Delta\bn-\Delta\bsn,\Delta\nun-\Delta\nusn)\|.
\end{align*}
Now, we have three primary cases regarding the convergence behaviors between $G_n=(a_n,b_{n},\nu_{n})$ and $G_{*,n}=(a_n^*, b^*_{n},\nu^*_{n})$.

\subsubsection*{Case 1:}
Both $A_n\to 0$ and $B_n\to 0$ as $n\to\infty$, or equivalently, $(\an, \bn,\nun)$ and $(\asn, \bsn,\nusn)$ share the same limit of $(a_0,b_0,\nu_0)$. Without loss of generality,
we assume that $\lambdasn\geq\lambdan$ for all $n\in\mathbb{N}$. Then, the metric ${D_2((\lambdan,\Gn),(\lambdasn,\Gsn))}$ is reduced to
\begin{align*}
    {D_2((\lambdan,\Gn),(\lambdasn,\Gsn))}=(\lambdasn-\lambdan)B_n^2+(\lambdan A_n+\lambdasn B_n)C_n.
\end{align*}
Subsequently, we consider
\begin{align*}
    p_{\lambdan, \Gn}(X,Y)-p_{\lambdasn, \Gsn}(X,Y)
    &= 
    (\lambdasn-\lambdan)
    \left[
    f(Y|
    \sigma (a_0^{\top}X+b_0 ),\nu_0
    )
    -
    f(Y|
    \sigma \left((a_n^*)^{\top}X+b_n^* \right),\nu_n^*
    )
    \right]\Bar{f}(X)\\
    &~+\lambdan
    \left[f(Y|
    \sigma \left(a_n^{\top}X+b_n \right),\nu\textbf{}
    )-f(Y|
    \sigma \left((a_n^*)^{\top}X+b_n^* \right),\nu_n^*
    )
    \right]\Bar{f}(X)\\
    &=S_n+T_n,
\end{align*}
where $S_n$ and $T_n$ denote the first term and second term in the equation above, respectively. 

By using the Taylor expansion up to the second order, we get
\begin{align*}
    S_n&=(\lambdasn-\lambdan)
    \Big[\sum_{1\leq|\alpha|+|\beta|\leq 2}\frac{1}{\alpha!\beta!}\prod_{u=1}^{d}
    \{-(\Delta\asn)^{(u)}\}^{\alpha_{1u}}
    \{-(\Delta\bsn)^{}\}^{\alpha_2}
    \{-(\Delta\nusn)^{}\}^{\beta}
    \\
    &\hspace{3.8cm}\times\frac{\partial^{|\alpha|+|\beta|}f}{\partial a^{\alpha_1}\partial b^{\alpha_2}\partial\nu^{\beta}}(\ysigma)\Bar{f}(X)+R_1(X,Y)\Big],\\
    T_n&=\lambdan\Big[\sum_{1\leq|\alpha|+|\beta|\leq 2}\frac{1}{\alpha!\beta!}\prod_{u=1}^{d}
    \{(\Delta\an-\Delta\asn)^{(u)}\}^{\alpha_{1u}}
    \{(\Delta\bn-\Delta\bsn)^{ }\}^{\alpha_{2}}
    \{(\Delta\nun-\Delta\nusn)^{ }\}^{\beta}
    \\
    &\hspace{3.8cm}\times\frac{\partial^{|\alpha|+|\beta|}f}{\partial a^{\alpha_1}\partial b^{\alpha_2}\partial\nu^{\beta}}(\ysigma)\Bar{f}(X)+R_2(X,Y)\Big],
\end{align*}
where 
$R_1(X,Y)=\bigO(B_n^{2+\gamma})$ and $R_2(X,Y)=\bigO(C_n^{2+\gamma})$ are Taylor remainders for some $\gamma>0$. According to the triangle inequality, we get $A_n+B_n\geq C_n$. It follows from this result and the assumption both $A_n\to 0$ and $B_n\to 0$ that
\begin{align*}
    (\lambdasn-\lambdan)|R_1(X,Y)|/\dtwo\leq |R_1(X,Y)|/B_n^2=\bigO(B_n^\gamma)&\to 0,\\
    \lambdan|R_2(X,Y)|/\dtwo\leq|R_2(X,Y)|/[(A_n+B_n)C_n]=\bigO(C_n^\gamma)&\to 0,
\end{align*}
for all $(X,Y)\in\mathcal{X}\times\mathcal{Y}$. 
Then we will have 
\begin{align*}
    \frac{(\lambdan-\lambdasn)|R_1(X,Y)|+\lambdan|R_2(X,Y)|}{\dtwo}\to0 .
\end{align*}
Due to the following PDE:
\begin{align*}
    \frac{\partial^2 f}{\partial \sigma^2}(Y|\sigma((a_n^*)^{\top}X+b_n^*),\nu_n^*)=2\frac{\partial f}{\partial \nu}(Y|\sigma((a_n^*)^{\top}X+b_n^*),\nu_n^*),
\end{align*}
for all $(a,b,\nu)\in\Theta$, $S_n$ can be rewritten as follows:
\begin{align*}
    S_n&=\sum_{\tau=1}^4S_{n,\tau}(X)\frac{\partial^{\tau} f}{\partial \sigma^{\tau}}(Y|\sigma((a_n^*)^{\top}X+b_n^*),\nu_n^*)\Bar{f}(X)+R_1(X,Y)\,
\end{align*}
where 
\begin{align*}
    &S_{n,1}(X)=(\lambdasn-\lambdan)
    \Big[\sum_{u=1}^{d}\{-(\Delta\asn)^{(u)}\}\frac{\partial  \sigma}{\partial a^{(u)}}(X,\asn,\bsn)+ 
    \{-(\Delta\bsn)^{}\}\frac{\partial  \sigma}{\partial b^{}}(X,\asn,\bsn)\\ 
    &~+\sum_{1\leq u,v\leq d}\frac{\{-(\Delta \asn)^{(u)}\}\{-(\Delta\asn)^{(v)}\}}{1+\mathbf{1}_{u=v}}\frac{\partial^2 \sigma}{\partial a^{(u)}\partial a^{(v)}}(X,\asn,\bsn)
    \\&~+\sum_{1\leq u\leq d}{\{-(\Delta \asn)^{(u)}\}\{-(\Delta\bsn)^{}\}}{ }\frac{\partial^2 \sigma}{\partial a^{(u)}\partial b}(X,\asn,\bsn)
    +\frac{1}{2}{\{-(\Delta \bsn)^{}\}^2}
    \frac{\partial^2 \sigma}{\partial b^{2}}(X,\asn,\bsn)
    \Big],\\
    &S_{n,2}(X)=(\lambdasn-\lambdan)\Big[
    \frac{1}{2}
    \{-(\Delta\nusn)^{}\} 
    +\sum_{1\leq u,v\leq d}\frac{\{-(\Delta \asn)^{(u)}\}\{-(\Delta\asn)^{(v)}\}}{1+\mathbf{1}_{u=v}}\frac{\partial \sigma}{\partial a^{(u)}}(X,\asn,\bsn)
    \frac{\partial \sigma}{\partial a^{(v)}}(X,\asn,\bsn)
    \\&~+\sum_{1\leq u\leq d}{\{-(\Delta \asn)^{(u)}\}\{-(\Delta\bsn)^{}\}}{ }
    \frac{\partial \sigma}{\partial a^{(u)}}(X,\asn,\bsn)
    \frac{\partial \sigma}{\partial b}
    (X,\asn,\bsn)
    +\frac{1}{2}{\{-(\Delta \bsn)^{}\}^2}
    (\frac{\partial \sigma}{\partial b}
    (X,\asn,\bsn))^2
    \Big],\\
    &S_{n,3}(X)=\frac{\lambdasn-\lambdan}{2}
    \Big[
    \sum_{u=1}^{d}
    \{-(\Delta\asn)^{(u)}\}\{-(\Delta\nusn)^{}\}\frac{\partial  \sigma}{\partial a^{(u)}}(X,\asn,\bsn) 
    +\{-(\Delta\bsn)^{}\}\{-(\Delta\nusn)^{}\}\frac{\partial  \sigma}{\partial b }(X,\asn,\bsn) ,
    \Big]\\
    &S_{n,4}(X)=\frac{\lambdasn-\lambdan}{8}
    \{-(\Delta\nusn)^{}\}^2.
\end{align*}
Similarly, we can rewrite $T_n$ in the same fashion as follows:
\begin{align*}
    T_n&=\sum_{\tau=1}^4T_{n,\tau}(X)\frac{\partial^{\tau} f}{\partial  \sigma^{\tau}}(\ysigma)\Bar{f}(X)+R_2(X,Y),
\end{align*}
where 
\begin{align*}
    T_{n,1}(X)&=\lambdan\Big[\sum_{u=1}^{d}\{(\Delta\an-\Delta\asn)^{(u)}\}\frac{\partial  \sigma}{\partial a^{(u)}}( \xabs)
    +\{(\Delta\bn-\Delta\bsn)^{}\}\frac{\partial  \sigma}{\partial b^{}}( \xabs)
    \\
    &~+\sum_{1\leq u,v\leq d}\frac{\{(\Delta\an-\Delta\asn)^{(u)}\}\{(\Delta\an-\Delta\asn)^{(v)}\}}{1+\mathbf{1}_{u=v}}\frac{\partial^2 \sigma}{\partial a^{(u)}\partial a^{(v)}}( \xabs)
    \\&~+\sum_{1\leq u\leq d}{\{(\dan-\Delta \asn)^{(u)}\}\{(\dbn-\Delta\bsn)^{}\}}{ }\frac{\partial^2 \sigma}{\partial a^{(u)}\partial b}(X,\asn,\bsn)
    +\frac{1}{2}{\{(\dbn-\Delta \bn)^{}\}^2}
    \frac{\partial^2 \sigma}{\partial b^{2}}(X,\asn,\bsn)
    \Big],\\
    T_{n,2}(X)&=\lambdan\Big[
    \frac{1}{2}
    \{(\dnun-\Delta\nusn)^{}\} 
    +\sum_{1\leq u,v\leq d}\frac{\{(\dan-\Delta \asn)^{(u)}\}\{(\dan-\Delta\asn)^{(v)}\}}{1+\mathbf{1}_{u=v}}
    \frac{\partial \sigma}{\partial a^{(u)}}(X,\asn,\bsn)
    \frac{\partial \sigma}{\partial a^{(v)}}
    (X,\asn,\bsn)
    \\&~+\sum_{1\leq u\leq d}{\{(\dan-\Delta \asn)^{(u)}\}\{(\dbn-\Delta\bsn)^{}\}}{ }
    \frac{\partial  \sigma}{\partial a^{(u)} }(X,\asn,\bsn)
    \frac{\partial  \sigma}{ \partial b}
    (X,\asn,\bsn)
    \\&~
    +\frac{1}{2}{\{(\dbn-\Delta \bsn)^{}\}^2}
    (\frac{\partial \sigma}{\partial b^{}}
    (X,\asn,\bsn))^2
    \Big],\\
    T_{n,3}(X)&=\frac{\lambdan}{2}
    \Big[
    \sum_{u=1}^{d}
    \{(\dan-\Delta\asn)^{(u)}\}\{(\dnun-\Delta\nusn)^{}\}\frac{\partial  \sigma}{\partial a^{(u)}}(X,\asn,\bsn) 
    \\
    &~
    +\{(\dbn-\Delta\bsn)^{}\}\{(\dnun-\Delta\nusn)^{}\}\frac{\partial  \sigma}{\partial b }(X,\asn,\bsn) ,
    \Big]\\
    T_{n,4}(X)&=\frac{\lambdan}{8}
    \{(\dnun-\Delta\nusn)^{}\}^2.
\end{align*}
Therefore, we can view the quantity $[p_{\lambdan,\Gn}(X,Y)-p_{\lambdasn,\Gsn}(X,Y)]/\dtwo)$ as a linear combination of elements of 
the set $\mathcal{L}=\cup_{\tau=1}^4\mathcal{L}_{\tau}$, where
\begin{align*}
    \mathcal{L}_1&=\left\{\frac{\partial  \sigma}{\partial a^{(u)}}(\xabs)\frac{\partial f}{\partial  \sigma}(\ysigma)\Bar{f}(X):u\in[d]\right\}\\
    &\cup\left\{\frac{\partial  \sigma}{\partial b}(\xabs)\frac{\partial f}{\partial  \sigma}(\ysigma)\Bar{f}(X) \right\}\\
    &\cup\left\{\frac{\partial^2  \sigma}{\partial a^{(u)}\partial a^{(v)}}(\xabs)\frac{\partial f}{\partial  \sigma}(\ysigma)\Bar{f}(X):u,v\in[d]\right\},\\
    &\cup\left\{\frac{\partial^2  \sigma}{\partial a^{(u)}\partial b}(\xabs)\frac{\partial f}{\partial  \sigma}(\ysigma)\Bar{f}(X):u \in[d]\right\},\\
    &\cup\left\{\frac{\partial^2  \sigma}{\partial b^2}(\xabs)\frac{\partial f}{\partial  \sigma}(\ysigma)\Bar{f}(X) \right\},\\
    \mathcal{L}_2
    &=\left\{
    \frac{\partial^2f}{\partial  \sigma^2}(\ysigma)\Bar{f}(X) \right\}\\
    &\cup\left\{
    \frac{\partial  \sigma}{\partial a^{(u)} }(\xabs)
    \frac{\partial  \sigma}{\partial a^{(v)} }(\xabs)
    \frac{\partial^2f}{\partial  \sigma^2}(\ysigma)\Bar{f}(X):u,v\in[d]\right\}\\
    &\cup\left\{
    \frac{\partial  \sigma}{\partial a^{(u)}}(\xabs)
    \frac{\partial   \sigma}{\partial  b}(\xabs)
    \frac{\partial^2f}{\partial  \sigma^2}(\ysigma)\Bar{f}(X):u \in[d]\right\}\\
    &\cup\left\{
    (\frac{\partial  \sigma}{  \partial  b}(\xabs))^2
    \frac{\partial^2f}{\partial  \sigma^2}(\ysigma)\Bar{f}(X) \right\},\\
    \mathcal{L}_3&=
    \left\{\frac{\partial  \sigma}{\partial a^{(u)}}(\xabs)
    \frac{\partial^3f}{\partial  \sigma^3}(\ysigma)\Bar{f}(X):u\in[d] \right\}\\
    &\cup\left\{\frac{\partial  \sigma}{\partial b}(\xabs)
    \frac{\partial^3f}{\partial  \sigma^3}(\ysigma)\Bar{f}(X)  \right\},\\
    \mathcal{L}_4&=\left\{
    \frac{\partial^4 f}{\partial  \sigma^4}(\ysigma)\Bar{f}(X) \right\}.
\end{align*}
Assume by contrary that all the coefficients of these elements vanish when $n\to\infty$. 
Looking at the coefficients of 
$\dfrac{\partial^2  \sigma}{\partial a^{(u)}\partial a^{(v)}}(\xabs)\dfrac{\partial f}{\partial  \sigma}(\ysigma)\Bar{f}(X)$ and
$\dfrac{\partial^2  \sigma}{\partial a^{(u)}\partial b}(\xabs)\dfrac{\partial f}{\partial  \sigma}(\ysigma)\Bar{f}(X)$
and
$\dfrac{\partial^2  \sigma}{\partial b^{2}}(\xabs)\dfrac{\partial f}{\partial  \sigma}(\ysigma)\Bar{f}(X)$
, 
we get for all $u,v\in[d]$,
\begin{align}
    \label{eq:nondistinguishable_independent_3}
    [(\lambdasn-\lambdan)(\Delta\asn)^{(u)}(\Delta\asn)^{(v)}+\lambdan(\Delta\an-\Delta\asn)^{(u)}(\Delta\an-\Delta\asn)^{(v)}]/\dtwo&\to 0,
    \nonumber\\
    [(\lambdasn-\lambdan)(\Delta\asn)^{(u)}(\Delta\bsn)^{ }+\lambdan(\Delta\an-\Delta\asn)^{(u)}(\Delta\bn-\Delta\bsn)^{ }]/\dtwo&\to 0,
    \nonumber\\
    [(\lambdasn-\lambdan) (\Delta\bsn)^{ 2}+\lambdan (\Delta\bn-\Delta\bsn)^{ 2}]/\dtwo&\to 0
    .
\end{align}
Let $u=v$ in the first equation in equation \eqref{eq:nondistinguishable_independent_3}, we achieve that for all $u\in[d]$,
\begin{align}   
    \label{eq:nondistinguishable_independent_4}
    [(\lambdasn-\lambdan)\{(\Delta\asn)^{(u)}\}^2+\lambdan\{(\Delta\an-\Delta\asn)^{(u)}\}^2]/\dtwo\to 0,
\end{align}
which implies that
\begin{align}
    \label{eq:nondistinguishable_independent_5}
    [(\lambdasn-\lambdan)\|\Delta\asn\|^2+\lambdan\|\Delta\an-\Delta\asn\|^2]/\dtwo\to 0,
    \nonumber\\
    [(\lambdasn-\lambdan)\|\Delta\bsn\|^2+\lambdan\|\Delta\bn-\Delta\bsn\|^2]/\dtwo\to 0
    .
\end{align}

We also have each term inside equation \eqref{eq:nondistinguishable_independent_5} is non-negative, thus  
\begin{align}
\label{eq:nondistinguishable_independent_11}
    (\lambdasn-\lambdan)\|\Delta\asn\|^2/\dtwo \to 0,~\lambdan\|\Delta\an-\Delta\asn\|^2/\dtwo &\to 0, \nonumber \\ 
    (\lambdasn-\lambdan)\|\Delta\bsn\|^2/\dtwo \to 0,~ \lambdan\|\Delta\bn-\Delta\bsn\|^2/\dtwo &\to 0.
\end{align}

Applying the AM-GM inequality, we have for all $u,v\in[d]$,
\begin{align}
\label{eq:nondistinguishable_independent_12}
    \dfrac{(\lambdasn-\lambdan)(\Delta\asn)^{(u)}(\Delta\asn)^{(v)}}{\dtwo}\to 0,~ \dfrac{\lambdan(\Delta\an-\Delta\asn)^{(u)}(\Delta\an-\Delta\asn)^{(v)}}{\dtwo} &\to 0,
    \\ \label{eq:nondistinguishable_independent_13_hello}   \dfrac{(\lambdasn-\lambdan)(\Delta\asn)^{(u)}(\Delta\bsn)^{ }}{\dtwo}\to 0,~ \dfrac{\lambdan(\Delta\an-\Delta\asn)^{(u)}(\Delta\bn-\Delta\bsn)^{ }}{\dtwo} &\to 0.
\end{align}

Next, by considering the coefficients of $\dfrac{\partial  \sigma}{\partial a^{(u)}}(X,\asn,\bsn)\dfrac{\partial f}{\partial  \sigma}( \ysigma)\Bar{f}(X)$,
$\dfrac{\partial  \sigma}{\partial b}(X,\asn,\bsn)\dfrac{\partial f}{\partial  \sigma}( \ysigma)\Bar{f}(X)$
and 
$\dfrac{\partial^2f}{\partial  \sigma^2}( \ysigma)\Bar{f}(X)$, we have
\begin{align}
    \label{eq:nondistinguishable_independent_1}
    [\lambdan(\Delta\an)^{(u)}-\lambdasn(\Delta\asn)^{(u)}]/\dtwo&\to 0,\quad u\in[d],\\
    \label{eq:nondistinguishable_independent_1.2}
    [\lambdan(\Delta\bn)^{ }-\lambdasn(\Delta\bsn)^{ }]/\dtwo&\to 0,\quad    \\
    \label{eq:nondistinguishable_independent_2}
    [\lambdan(\Delta\nun)^{ }-\lambdasn(\Delta\nusn)^{ }]/\dtwo&\to 0.\quad 
\end{align}
Noting that for $u,v\in[d]$,
\begin{align*}
    \lambdasn(\dasn)^{(u)}(\dan-\dasn)^{(v)} &= (\lambdan(\dan)^{(v)} - \lambdasn(\dasn)^{(v)})(\dasn)^{(u)}+(\lambdasn-\lambdan)(\dan)^{(v)}(\dasn)^{(u)},\\
    \lambdan(\dan)^{(u)}(\dan-\dasn)^{(v)} &= \lambdasn(\dasn)^{(u)}(\dan-\dasn)^{(v)} - (\lambdan(\dan)^{(u)} - \lambdasn(\dasn)^{(u)})(\dan-\dasn)^{(v)}.
\end{align*}  
Thus, from equation \eqref{eq:nondistinguishable_independent_12} and equation \eqref{eq:nondistinguishable_independent_1}, we achieve that for $u,v\in[d]$,
\begin{align*}
    \lambdasn(\dasn)^{(u)}(\dan-\dasn)^{(v)}/\dtwo &\to 0,\\ \lambdan(\dan)^{(u)}(\dan-\dasn)^{(v)}/\dtwo &\to 0. 
\end{align*}

Noting that for all $u\in[d]$,
\begin{align*}
    \lambdasn(\dasn)^{(u)}(\dbn-\dbsn)^{ } &= (\lambdan\dbn - \lambdasn\dbsn)(\dasn)^{(u)}+(\lambdasn-\lambdan)(\dbn)(\dasn)^{(u)},\\
    \lambdan(\dan)^{(u)}(\dbn-\dbsn)^{ } &= \lambdasn(\dasn)^{(u)}(\dbn-\dbsn)^{ } - \left(\lambdan(\dan)^{(u)} - \lambdasn(\dasn)^{(u)}\right)(\dbn-\dbsn)^{ }.
\end{align*}
Thus, from 
equation \eqref{eq:nondistinguishable_independent_1} and equation \eqref{eq:nondistinguishable_independent_1.2}, we have for all $u\in[d]$,
\begin{align*}
    \lambdasn(\dasn)^{(u)}(\dbn-\dbsn) /\dtwo &\to 0\\ \lambdan(\dan)^{(u)}(\dbn-\dbsn) /\dtwo &\to 0. 
\end{align*}

By using the same arguments we will derive 
\begin{align}
    \label{eq:nondistinguishable_independent_11.3}
    \lambdan\|\Delta\an\|.\|\Delta\bn-\Delta\bsn\|/\dtwo\to0,\\
    \label{eq:nondistinguishable_independent_12.3}
    \lambdasn\|\Delta\asn\|.\|\Delta\bn-\Delta\bsn\|/\dtwo\to0,\\
    \label{eq:nondistinguishable_independent_11.4}
    \lambdan\|\Delta\bn\|.\|\Delta\bn-\Delta\bsn\|/\dtwo\to0,\\
    \label{eq:nondistinguishable_independent_12.4}
    \lambdasn\|\Delta\bsn\|.\|\Delta\bn-\Delta\bsn\|/\dtwo\to0.
\end{align}
By using the same arguments to derive equation~\eqref{eq:nondistinguishable_independent_4}, equation~\eqref{eq:nondistinguishable_independent_11} and equation~\eqref{eq:nondistinguishable_independent_12}, we can point out that
\begin{align}
    [(\lambdasn-\lambdan)\|\Delta\nusn\|^2+\lambdan\|\Delta\nun-\Delta\nusn\|^2]/\dtwo&\to 0,\nonumber\\
    \lambdan\|\Delta\nun\|.\|\Delta\nun-\Delta\nusn\|/\dtwo&\to 0,\nonumber\\
    \label{eq:nondistinguishable_independent_13}
    \lambdasn\|\Delta\nusn\|.\|\Delta\nun-\Delta\nusn\|/\dtwo&\to 0,\\
    \lambdan\|\Delta\an\|.\|\Delta\nun-\Delta\nusn\|/\dtwo&\to 0,\nonumber\\
    \lambdasn\|\Delta\asn\|.\|\Delta\nun-\Delta\nusn\|/\dtwo&\to 0\nonumber\\
    \lambdan\|\Delta\bn\|.\|\Delta\nun-\Delta\nusn\|/\dtwo&\to 0,\nonumber\\
    \lambdasn\|\Delta\bsn\|.\|\Delta\nun-\Delta\nusn\|/\dtwo&\to 0\nonumber.
\end{align}
Collecting results in equation~\eqref{eq:nondistinguishable_independent_4}, and equations~\eqref{eq:nondistinguishable_independent_11} to
\eqref{eq:nondistinguishable_independent_13}, we obtain that
\begin{align*}
    1=\dtwo/\dtwo\to 0,
\end{align*}
which is a contradiction. Therefore, not all the coefficients in the representation of 
${\pminus}/\dtwo$
tend to 0 as $n\to\infty$. Let us denote by $m_n$ the maximum of the absolute values of those coefficients. Based on the previous result, $1/m_n\not\to\infty$. Additionally, we define
\begin{align*}
    [(\lambdasn-\lambdan)\{-(\Delta a^*_{  n})^{(u)}\}+\lambdan(\Delta a_{  n}-\Delta a^*_{ n})^{(u)}]/m_n&\to\alpha_{1 u},\\
    [(\lambdasn-\lambdan)\{-(\Delta b^*_{  n})^{ }\}+\lambdan(\Delta b_{  n}-\Delta b^*_{ n})^{ }]/m_n&\to\alpha_{20},\\
    [(\lambdasn-\lambdan)\{-(\Delta \nu^*_{  n})^{ }\}+\lambdan(\Delta \nu_{  n}-\Delta \nu^*_{ n})^{ }]/m_n&\to\alpha_{30},\\
    [(\lambdasn-\lambdan)(\Delta a^*_{  n})^{(u)}(\Delta a^*_{  n})^{(v)}+\lambdan(\Delta a_{  n}-\Delta a^*_{ n})^{(u)}(\Delta a_{  n}-\Delta a^*_{  n})^{(v)}]/m_n&\to\beta_{ 1 uv},\\
    [(\lambdasn-\lambdan)(\Delta a^*_{  n})^{(u)}(\Delta b^*_{  n})^{ }+\lambdan(\Delta a_{  n}-\Delta a^*_{ n})^{(u)}(\Delta b_{  n}-\Delta b^*_{  n})^{ }]/m_n&\to\beta_{ 2 u0},\\
    [(\lambdasn-\lambdan)(\Delta b^*_{  n})^{2} +\lambdan (\Delta b_{  n}-\Delta b^*_{  n})^{ 2}]/m_n&\to\beta_{ 3 00},\\
    [(\lambdasn-\lambdan)(\Delta\asn)^{(u)}(\Delta\nusn)^{ }+\lambdan(\Delta\an-\Delta\asn)^{(u)}(\Delta\nun-\Delta\nusn)^{ }]/m_n&\to\gamma_{1u},\\
    [(\lambdasn-\lambdan)(\Delta\bsn)^{ }(\Delta\nusn)^{ }+\lambdan(\Delta\bn-\Delta\bsn)^{ }(\Delta\nun-\Delta\nusn)^{ }]/m_n&\to\gamma_{20},\\
    [(\lambdasn-\lambdan) (\Delta\nusn)^{ 2}+ \lambdan(\Delta\nun-\Delta\nusn)^{2 }]/m_n&\to\gamma_{30},
\end{align*}
when $n\to\infty$ for all $u,v\in[d]$. Note that at least one among $\alpha_{\tau u},\beta_{\tau uv}$ and $\gamma_{\tau u}$ where $\tau \in \{ 1,2,3\}$ must be different from zero. By applying the Fatou's lemma, we get
\begin{align*}
    0=\lim_{n\to\infty}\frac{1}{m_n}\frac{2V(p_{\lambdan,\Gn},p_{\lambdasn,\Gsn})}{\dtwo}\geq \int\liminf_{n\to\infty}\frac{1}{m_n}\frac{|p_{\lambdan,\Gn}(X,Y)-p_{\lambdasn,\Gsn}(X,Y)|}{\dtwo}d(X,Y).
\end{align*}
On the other hand,
\begin{align*}
    \frac{1}{m_n}\frac{p_{\lambdan,\Gn}(X,Y)-p_{\lambdasn,\Gsn}(X,Y)}{\dtwo}\to\sum_{\tau=1}^4E_{\tau}(X)\frac{\partial^{\tau} f}{\partial  \sigma^{\tau}}(\ysigmazero)\Bar{f}(X),
\end{align*}
where 
\begin{align*}
    E_1(X)&=\sum_{u=1}^{d}\alpha_{1 u}\frac{\partial  \sigma}{\partial a^{(u)}}( \xabzero)
    +\alpha_{20}\frac{\partial  \sigma}{\partial b}( \xabzero)
    +\sum_{1\leq u,v\leq d}\frac{\beta_{1uv}}{1+\mathbf{1}_{\{u=v\}}}\frac{\partial^2  \sigma}{ \partial a^{(u)} \partial a^{(v)}}( \xabzero)\\
    &+\sum_{1\leq u\leq d}{\beta_{2u0}}{ }\frac{\partial^2  \sigma}{ \partial a^{(u)} \partial b}( \xabzero)
    +{\beta_{300}}{ }\frac{\partial^2  \sigma}{\partial b^2}( \xabzero)
    ,\\
    E_2(X)&=
    \frac{1}{2}\alpha_{30}
    +\sum_{1\leq u,v\leq d}\frac{\beta_{1uv}}{1+\mathbf{1}_{\{u=v\}}}
    \frac{\partial   \sigma}{   \partial a^{(v)}}( \xabzero)
    \frac{\partial   \sigma}{ \partial a^{(u)}  }( \xabzero)
    \\
    &~
    +\sum_{1\leq u\leq d} {\beta_{2u0}}{ }
    \frac{\partial   \sigma}{ \partial a^{(u)}  }( \xabzero)
    \frac{\partial   \sigma}{  \partial b}( \xabzero)
    +\frac{1}{2}{\beta_{300}}{ }
    (\frac{\partial \sigma}{ \partial b  }( \xabzero))^2
    \\
    E_3(X)&=\frac{1}{2}\sum_{u=1}^{d}
    \gamma_{1u}\frac{\partial  \sigma}{\partial a^{(u)}}( \xabzero)
    +\frac{1}{2}\gamma_{20}\frac{\partial  \sigma}{\partial b}( \xabzero)
    ,\\
    E_4(X)&=\frac{1}{8}\gamma_{30}
    .
\end{align*}
Then, we have
\begin{align*}
    \sum_{\tau=1}^4E_{\tau}(X)\frac{\partial^{\tau}f}{\partial  \sigma^{\tau}}(\ysigmazero )=0.
\end{align*}
It is worth noting that for almost surely $(X,Y)$, the set
\begin{align*}
    \left\{\frac{\partial^{\tau}f}{\partial  \sigma^{\tau}}(\ysigmazero):\tau\in[4]\right\}
\end{align*}
is linearly independent, which leads to the fact that $E_{\tau}(X)=0$ for almost surely $X$ for any $\tau\in[4]$. 

Since $\sigma(X,a,b)$ and $\nu$ are algebraically independent and $E_3(X)=E_4(X)=0$ for almost surely $X$, 
we get 
$\gamma_{\tau u}=0$ 
for all $u$. 
It follows from the result $E_2(X)=0$ and 
$\gamma_{\tau u}=0$ that 
$\alpha_{30}=\beta_{\tau uv}=0$ for all $u,v$. Next, $E_1(X)=0$ implies that 
$\alpha_{\tau u}=0$ for all $u$. This contradicts the fact that not all $\alpha_{\tau u}$, $\beta_{\tau uv}$ and $\gamma_{uv}$ vanish. Thus, we obtain the conclusion for this case.

\subsubsection*{Case 2:}
Either $A_n\to 0$ or $B_n\to 0$ as $n\to\infty$. WLOG, we assume that $A_n\not\to 0$ and $B_n\to 0$, which means that 
$(\an,\bn,\nun)\not\to(a_0,b_0,\nu_0)$ and $(\asn,\bsn,\nusn)\to(a_0,b_0,\nu_0)$. 
Let us denote
\begin{align*}
    D^{\prime}_{2}((\lambdan,\Gn),(\lambdasn,\Gsn)):=|\lambdasn-\lambdan|B_n+\lambdan A_n+\lambdasn B_n.
\end{align*}
As 
$\dptwo \gtrsim \dtwo$ and $[p_{\lambdan,\Gn}(X,Y)-p_{\lambdasn,\Gsn}(X,Y)]/\dtwo\to 0$, 
we get $[p_{\lambdan,\Gn}(X,Y)-p_{\lambdasn,\Gsn}(X,Y)]/\dptwo\to 0$. By using the Taylor expansion up to the first order, we have
\begin{align*}
    &p_{\lambdan,\Gn}(X,Y)-p_{\lambdasn,\Gsn}(X,Y)
    =(\lambdasn-\lambdan)[f(\ysigmazero)-f(\ysigma)]\Bar{f}(X)\\
    &~+\lambdan[f(\ysigman)-f(\ysigma)]\Bar{f}(X)\\
    &=(\lambdasn-\lambdan)\Big[\sum_{u=1}^{d}(-\Delta\asn)^{(u)}\frac{\partial  \sigma}{\partial a^{(u)}}(\xabs)\frac{\partial f}{\partial  \sigma}(\ysigma)
    \\
    &~
    +  
    (-\Delta\bsn)^{ }\frac{\partial  \sigma}{\partial b}(\xabs)\frac{\partial f}{\partial  \sigma}(\ysigma)
    +\frac{1}{2}
    (-\Delta\nusn)^{ }
    \frac{\partial^2f}{\partial  \sigma^2}(\ysigma)+R_3(X,Y)\Big]\Bar{f}(X)\\
    &~+\lambdan[f(\ysigman)-f(\ysigma)]\Bar{f}(X),
\end{align*}
where $R_3(X,Y)$ is a Taylor remainder that satisfies
\begin{align*}
    (\lambdasn-\lambdan)|R_3(X,Y)|/\dptwo=\bigO(B_n^{\gamma'})\to 0,
\end{align*}
for some $\gamma'>0$. Thus, we can treat $[p_{\lambdan,\Gn}(X,Y)-p_{\lambdasn,\Gsn}(X,Y)]/\dptwo$ as a linear combination of elements of the set
\begin{align}
    \mathcal{K}=\Big\{&f(\ysigman)\Bar{f}(X),f(\ysigma)\Bar{f}(X),\nonumber\\
    &~\frac{\partial  \sigma}{\partial a^{(u)}}(\xabs)\frac{\partial f}{\partial  \sigma}(\ysigma)\Bar{f}(X),\nonumber
    \frac{\partial  \sigma}{\partial b}(\xabs)\frac{\partial f}{\partial  \sigma}(\ysigma)\Bar{f}(X),\nonumber\\
    \label{eq:nondistinguishable_independent_case_2}
    &~
    \frac{\partial^2f}{\partial  \sigma^2}(\ysigma)\Bar{f}(X):u\in[d] \Big\}.
\end{align}
Assume by contrary that all the coefficients of those elements go to 0 as $n\to\infty$, then we get for all $u\in[d]$,
\begin{align*}
    \frac{\lambdan}{\dptwo}\to 0, \quad 
    (\lambdasn-\lambdan)\frac{(-\Delta\asn)^{(u)}}{\dptwo}\to0,\\
    (\lambdasn-\lambdan)\frac{(-\Delta\bsn)^{}}{\dptwo}\to0,\quad
    (\lambdasn-\lambdan)\frac{(-\Delta\nusn)^{ }}{\dptwo}\to 0,
\end{align*}
which leads to 
\begin{align*}
    \frac{(\lambdasn-\lambdan)B_n}{\dptwo}\to 0,\quad \frac{\lambdan A_n}{\dptwo}\to 0, \quad \frac{\lambdasn B_n}{\dptwo}\to 0.
\end{align*}
As a result,
\begin{align*}
    1=\frac{|\lambdasn-\lambdan|B_n+\lambdan A_n+\lambdasn B_n}{\dptwo}\to 0,
\end{align*}
which is a contradiction. Consequently, not all the aforementioned coefficients tend to 0 as $n\to\infty$.

Subsequently, we denote by $m'_n$ the maximum of the absolute values of those coefficients, then $1/m'_n\not\to\infty$. 
Recall that $\dbtwo=\lambdan\|(\Delta\an,\dbn,\Delta\nun)\|^2\to 0$. 
Since $(\an,\bn,\nun)\not\to(a_0,b_0,\nu_0)$, we get $\lambdan\to 0$. Thus, 
\begin{align*}
    \frac{1}{m'_n}\frac{p_{\lambdan,\Gn}(X,Y)-p_{\lambdasn,\Gsn}(X,Y)}{\dptwo}\to\sum_{\ell=1}^2\alpha'_{\ell}(X)\frac{\partial^{\ell}f}{\partial  \sigma^{\ell}}(\ysigmazero)=0.
\end{align*}
Recall that for almost surely $(X,Y)\in\mathcal{X}\times\mathcal{Y}$, the set
\begin{align}
    \label{eq:nondistinguishable_independent_14}
    \left\{\frac{\partial^{\ell} f}{\partial  \sigma^{\ell}}(\ysigmazero):0\leq\ell\leq 2\right\}
\end{align}
is linearly independent. Then, equation~\eqref{eq:nondistinguishable_independent_14} indicates that 
\begin{align*}
    & \alpha'_1(X)=\lim_{n\to\infty}
    \Big[
    \sum_{u=1}^{d}
    \frac{1}{m'_n}\frac{(\lambdasn-\lambdan)(-\Delta\asn)^{(u)}}{\dptwo}\frac{\partial  \sigma}{\partial a^{(u)}}(X, a_0, b_0)
    +\frac{1}{m'_n}\frac{(\lambdasn-\lambdan)(-\Delta\bsn)^{ }}{\dptwo}\frac{\partial  \sigma}{\partial b}(X, a_0, b_0)
    \Big]
    =0,\\
    &\alpha'_2(X)=\lim_{n\to\infty}
     \frac{1}{m'_n}\frac{(\lambdasn-\lambdan)(-\Delta\nusn)^{ }}{\dptwo}
    =0.
\end{align*}
Due to the algebraic independence of $\sigma(X,a,b)$ and $\nu$, we obtain that for all $u\in[d]$,
\begin{align*}
    \lim_{n\to\infty}\frac{1}{m'_n}\frac{(\lambdasn-\lambdan)(-\Delta\asn)^{(u)}}{\dptwo}=0,\\
    \lim_{n\to\infty}\frac{1}{m'_n}\frac{(\lambdasn-\lambdan)(-\Delta\bsn)^{ }}{\dptwo}=0,\\
    \lim_{n\to\infty}\frac{1}{m'_n}\frac{(\lambdasn-\lambdan)(-\Delta\nusn)^{ }}{\dptwo}=0.
\end{align*}
Since $1/m'_n$ is bounded, the above limits imply that the coefficients of the following items
$$\frac{\partial  \sigma}{\partial a^{(u)}}(\xabs)\frac{\partial f}{\partial  \sigma}(\ysigma),
$$ 
$$
\frac{\partial  \sigma}{\partial b}(\xabs)\frac{\partial f}{\partial  \sigma}(\ysigma),
$$ 
$$\frac{\partial^2f}{\partial  \sigma^2}(\ysigma)$$
vanish when $n\to\infty$ for all $u\in[d] $. Moreover, those of $f(\ysigman)$ and $f(\ysigma)$, which are $\lambdan/\dptwo$ and $-\lambdan/\dptwo$, also go to 0. 
This contradicts the assumption that at least one among the coefficients of elements in the set $\mathcal{K}$ not vanishing. Thus, we reach the conclusion for this case.

\subsubsection*{Case 3:}
Both $A_n$ and $B_n$ do not vanish. This means that $( \an, \bn,\nun)\not\to( a_0, b_0,\nu_0)$ and $(\asn,\bsn,\nusn)\not\to( a_0, b_0,\nu_0)$. In this case, we will consider only the most difficult case when $( \an, \bn,\nun)$ and $(\asn,\bsn,\nusn)$ converge to the same point $(a',b',\nu')\neq( a_0, b_0,\nu_0)$.

From the formulation of the metric $\done$ in the proof of Theorem~\ref{theorem:sigma-linear-f0-notGaussian}, it is clear that $\dtwo\leq D_1((\lambdan,\Gn),(\lambdasn,\Gsn))$. 
Therefore, we get $[p_{\lambdan,\Gn}(X,Y)-p_{\lambdasn,\Gsn}(X,Y)]/\done\to 0$ as $n\to\infty$. By following the same steps to derive equation~\eqref{eq:distinguishable_independent_taylor} and equation~\eqref{eq:distinguishable_independent_fatou} and abuse notations defined for those equations, we obtain that
\begin{align*}
     &p_{\lambdan,\Gn}(X,Y)-p_{\lambdasn,\Gsn}(X,Y)
     =(\lambdasn-\lambdan)[f(\ysigmazero)-f(\ysigma)]\Bar{f}(X)\\
     &~+\lambdan\Big[\sum_{u=1}^{d}( \dan -\dasn)^{(u)}\frac{\partial  \sigma}{\partial a^{(u)}}(X,\asn,\bsn)\frac{\partial f}{\partial  \sigma}(\ysigma)\\
     &~+( \dbn -\dbsn)^{ }\frac{\partial  \sigma}{\partial b}(X,\asn,\bsn)\frac{\partial f}{\partial  \sigma}(\ysigma)
     \\
     &~
     +\frac{1}{2} (\dnun-\dnusn)^{ }
     \frac{\partial^2 f}{\partial  \sigma^2}(\ysigma)+R(X,Y)\Big]\Bar{f}(X).
 \end{align*}
 with a note that not all the coefficients of elements of the set $\mathcal{K}$ (defined in equation~\eqref{eq:nondistinguishable_independent_case_2}) go to 0 as $n\to\infty$. Furthermore,
\begin{align*}
    \frac{1}{m_n}\frac{p_{\lambdan,\Gn}(X,Y)-p_{\lambdasn,\Gsn}(X,Y)}{\done}\to&~\eta f(\ysigmazero)\Bar{f}(X)\\
    &~+\sum_{\ell=0}^2\alpha_{\ell}(X)\frac{\partial^\ell f}{\partial  \sigma^{\ell}}(\ysigmap)\Bar{f}(X)=0, 
\end{align*}
where
\begin{align*}
    \eta&=\lim_{n\to\infty}\frac{1}{m_n}\frac{(\lambdasn-\lambdan)}{\done},\quad  \alpha_0(X)=\lim_{n\to\infty}\frac{1}{m_n}\frac{-(\lambdasn-\lambdan)}{\done},\\
    \alpha_1(X)&=\lim_{n\to\infty}
    \Big[
    \sum_{u=1}^{d}
    \frac{1}{m_n}\frac{\lambdan( \dan-\dasn)^{(u)}}{\done}\frac{\partial  \sigma}{\partial a^{(u)}}(X, a',b' )
    +
    \frac{1}{m_n}\frac{\lambdan( \dbn-\dbsn)^{}}{\done}\frac{\partial  \sigma}{\partial b}(X, a',b' )
    \Big]
    ,\\
    \alpha_2(X)&=\lim_{n\to\infty} \frac{1}{2m_n}\frac{\lambdan(\dnun-\dnusn)^{ }}{\done}
    .
\end{align*}
Recall that $\dbtwo=\lambdan\|( \dan, \dbn,\dnun)\|^2\to 0$ and $\dbstwo=\lambdasn\|(\dasn, \dbsn,\dnusn)\|^2\to 0$. 
Since $( \an, \bn,\nun)\not\to( a_0, b_0,\nu_0)$ and $(\asn,\bsn,\nusn)\not\to( a_0, b_0,\nu_0)$, we get that $\lambdan\to 0$ and $\lambdasn\to 0$. As a result, we achieve that $\eta=\alpha_0(X)=0$, which leads to 
\begin{align*}
    \sum_{\ell=0}^2\alpha_{\ell}(X)\frac{\partial^\ell f}{\partial  \sigma^{\ell}}(\ysigmap)=0.
\end{align*}
Additionally, as the set 
\begin{align*}
    \left\{\frac{\partial^{\tau} f}{\partial  \sigma^{\tau}}(\ysigmap):0\leq\tau\leq 2\right\}
\end{align*}
is linearly independent, then $\alpha_1(X)=\alpha_2(X)=0$ for almost surely $X$. Moreover, $\sigma(X,a,b)$ and $\nu$ are algebraically independent and $1/m_n$ is bounded, we get that for all $u\in[d]$,
\begin{align*}
    \lim_{n\to\infty}\frac{\lambdan(\dan-\dasn)^{(u)}}{\done}=0, \\
    \lim_{n\to\infty}\frac{\lambdan(\dbn-\dbsn)^{ }}{\done}=0, \\
    \lim_{n\to\infty}\frac{\lambdan(\dnun-\dnusn)^{ }}{\done}=0.
\end{align*}
This contradicts the fact that not all the coefficients of elements in the set $\mathcal{K}$ vanish as $n\to\infty$. Hence, we reach the conclusion for this case.
\end{proof}

\subsection{Proof of Theorem~\ref{theorem:sigma-linear-f0-Gaussian-varphi-linear}}
\label{appendixproof:sigma-linear-f0-Gaussian-varphi-linear}
In this section, we consider both the experts $\varphi$ and $\sigma$ are of linear forms, and the density $f_0=f$ is a Gaussian density.
Building on the explanation in Section \ref{section:sigma-linear-f0-Gaussian-varphi-linear},  we will focus only on the case when $\varphi$ and $\sigma$ are the same linear functions.
The most challenging part is that there is an interaction between parameters $b$ and $\nu$ via the following heat equation:
\begin{align*}
        \frac{\partial^2 f}{\partial b^2} (Y|a^{\top}X+b ,\nu ) &= 2 \frac{\partial f}{\partial \nu}(Y|a^{\top}X+b ,\nu).
\end{align*}
Recall for the simplicity of the presentation in the paper, we define $\Delta G=(\Delta a,\Delta b,\Delta\nu) = (a - a_0,b - b_0,\nu-\nu_0 )$ and $\Delta \Gs=(\Delta \as,\Delta \bs,\Delta\nus) = (\as - a_0,\bs - b_0,\nus-\nu_0 )$ for any element $(a,b,\nu), (\as,\bs,\nus)\in \Theta$.
For the convenience of readers, we will first restate the result of Theorem~\ref{theorem:sigma-linear-f0-Gaussian-varphi-linear}:
denote for any sequence $(l_n)_{n\geq 1}$,
\begin{align*}
    \Xi_1(l_n):=
\Big\{ 
(\lambda,G)={(\lambda,a,b,\nu)}\in\Xi:
\frac{l_n}{\min_{1\leq i\leq d}\{ 
|(\da)_i|^2,|\db|^4,|\dnu|^2
\}\sqrt{n}}\leq\lambda
\Big\}.
\end{align*}
Then, the followings hold for any sequence $(l_n )_{n\geq 1}$ such that $l_n/\log n\to\infty$ as $n\to\infty$:
\begin{align*}
    \sup_{(\lambdas,\Gs)\in \Xi_1(l_n)  }
    \mathbb{E}_{p_{\lambdas,\Gs}} \Big[ 
    (\Vert\das\Vert^4+|\dbs|^8+|\dnus|^4)
    \times |\widehat{\lambda}_n
    -\lambdas|^2 \Big] 
    &\lesssim \frac{\log n}{n},
    \\
    \sup_{(\lambdas,\Gs)\in \Xi_1(l_n) }
    \mathbb{E}_{p_{\lambdas,\Gs}} 
    \Big[ (\lambdas)^2 
    (\Vert\das\Vert^2+|\dbs|^4+|\dnus|^2)
    \qquad\qquad&
    \\
    \times(\Vert\dhan-\das\Vert^2+|\dhbn-\dbs|^4+|\dhnun-\dnus|^2)
    \Big] 
    &\lesssim \frac{\log n}{n}.
\end{align*}
To prove the above results,  we can establish the following result concerning the lower bound of $V(p_{\lambda, G}, p_{\lambda^*, G_*})$.
\begin{lemma}
    \label{applemma:sigma-linear-f0-Gaussian-varphi-linear} 
Assume that $f_0$ and $f$ are both Gaussians, and the mean experts $\varphi$ in $f_0$ and $\sigma$ in $f$ are both linear functions.
Futhermore, if $f_0(Y|\varphi(a_0^{\top}X+b_0),\nu_0)=f(Y|\sigma(a_0^{\top}X+b_0),\nu_0)$ for some $(a_0,b_0,\nu_0)\in\Theta$.
For any $(\lambda, G),(\lambda^*, G_*)$, let us define
    \begin{align*}
        &D_4((\lambda,\G),(\lambdas, \Gs))
        =~
        \lambda 
        \left(
        \Vert \Delta a \Vert^2
        +|\Delta b |^4+|\Delta \nu |^2 
        \right)
        +\lambdas
        \left(
        \Vert\Delta\as\Vert^2+|\Delta\bs|^4+|\Delta\nus|^2
        \right)\\
        &-\min\{\lambda,\lambdas\}
        \left[        
        \Vert\Delta a\Vert^2+|\Delta b|^4+|\Delta \nu|^2+
        \Vert\Delta\as\Vert^2+|\Delta\bs|^4+|\Delta\nus|^2
        \right]\\
        &+\left[
        \lambda
\left(
        \Vert\Delta a\Vert+|\Delta b|^2+|\Delta \nu|
        \right)
        +\lambdas
        \left(
        \Vert\Delta\as\Vert+|\Delta\bs|^2+|\Delta\nus|
        \right)
        \right]
        \times
        \left(
        \Vert\Delta a-\Delta\as\Vert+
        |\Delta b-\Delta\bs|^2+
        |\Delta \nu-\Delta\nus|
        \right).
    \end{align*}
    Then, there exists a positive constant $C_4$ that depends on $\Theta$, $(a_0,b_0,\nu_0)$ and satisfies the following:
\begin{align*}
    V(p_{\lambda, G},p_{\lambda^*, G_*})\geq C_4D_4((\lambda, G),(\lambda^*, G_*)).
\end{align*}
    for all $(\lambda, G)$ and $(\lambda^*, G_*)$
\end{lemma}

\begin{proof}[Proof of Theorem~\ref{theorem:sigma-linear-f0-Gaussian-varphi-linear}]
Similar to the proof in the previous section, we define:
    \begin{align*}
        &\overline{D_4}((\lambda,G),(\lambdas, \Gs))
        =
        (\lambdas -\lambda )
        \Big(
        \Vert \Delta a\Vert+|\Delta b|^2+|\Delta\nu |
        \Big)
        \Big(
        \Vert\Delta\as \Vert+|\Delta\bs |^2+|\Delta\nus |
        \Big)
        \\
        &+\Big(
        \lambda 
        (
        \Vert \Delta a\Vert+|\Delta b|^2+|\Delta\nu |
        )
        +\lambdas 
        (
        \Vert\Delta\as \Vert+|\Delta\bs |^2+|\Delta\nus |
        )
        \Big)\cdot
        \Big(
        \Vert\Delta a-\Delta\as \Vert+
        |\Delta b-\Delta\bs |^2+
        |\Delta\nu -\Delta\nus |
        \Big).        
    \end{align*}
It will be shown that $D_4((\lambda ,G),(\lambdas , \Gs )) \asymp \overline{D_4}((\lambda ,G),(\lambdas , \Gs ))$. 
Considering the above formulation of $\overline{D_4}((\lambda ,G),(\lambdas , \Gs ))$ and assuming, without loss of generality, that $\lambdas \geq \lambda$, we can infer that
 \begin{align*}
         &\overline{D_4}((\lambda ,G ),(\lambdas , \Gs ))
         \\&\lesssim
         (\lambdas -\lambda )
         \Big(
         \Vert\Delta\as \Vert + \Vert\Delta a-\Delta\as \Vert+
         |\Delta\bs |^2 + |\Delta b-\Delta\bs |^2 +
         |\Delta\nus | + |\Delta\nu -\Delta\nus |
         \Big)
         \cdot
         \Big(
         \Vert\Delta\as \Vert+|\Delta\bs |^2+|\Delta\nus |
         \Big)
         \\
         &+\Big(
         \lambda 
         (
         \Vert \Delta a\Vert+|\Delta b|^2+|\Delta\nu |
         )
         +\lambdas 
         (
         \Vert\Delta\as \Vert+|\Delta\bs |^2+|\Delta\nus |
         )
         \Big)
         \cdot
         \Big(
         \Vert\Delta a-\Delta\as \Vert+
         |\Delta b-\Delta\bs |^2+
        |\Delta\nu -\Delta\nus |
        \Big)       
        \\
        &\lesssim (\lambdas -\lambda )
        \Big(
        \Vert\Delta\as \Vert +
        |\Delta\bs |^2 +
        |\Delta\nus | 
        \Big)\Big(
        \Vert\Delta\as \Vert +
        |\Delta\bs |^2 +
        |\Delta\nus | 
        \Big)\\
        &+\lambdas 
        \Big(
        \Vert\Delta\as \Vert +
        |\Delta\bs |^2 +
        |\Delta\nus | 
        \Big) \Big(
        \Vert\Delta a-\Delta\as \Vert+
        |\Delta b-\Delta\bs |^2+
        |\Delta\nu -\Delta\nus |
         \Big) \\
        &+\Big(
        \lambda 
        (
        \Vert \Delta a\Vert+|\Delta b|^2+|\Delta\nu |
        )
        +\lambdas 
        (
        \Vert\Delta\as \Vert+|\Delta\bs |^2+|\Delta\nus |
        )
        \Big)
        \Big(
        \Vert\Delta a-\Delta\as \Vert+
        |\Delta b-\Delta\bs |^2+
        |\Delta\nu -\Delta\nus |
         \Big)\\
         &\lesssim D_4((\lambda ,G ),(\lambdas , \Gs )).
     \end{align*}
On the other hand, 
\begin{align*}
        &\overline{D_4}((\lambda ,G ),(\lambdas , \Gs ))
        \gtrsim
        (\lambdas -\lambda )
        \Big(
        \Vert \Delta a\Vert + \Vert\Delta a-\Delta\as \Vert+
        |\Delta b|^2 + |\Delta b-\Delta\bs |^2 +
        |\Delta\nu | + |\Delta\nu -\Delta\nus |
        \Big)
        \cdot\\&
        \Big(
        \Vert\Delta\as \Vert+|\Delta\bs |^2+|\Delta\nus |
        \Big)
        +\Big(
        \lambda 
        (
        \Vert \Delta a\Vert+|\Delta b|^2+|\Delta\nu |
        )
        +\lambdas 
        (
        \Vert\Delta\as \Vert+|\Delta\bs |^2+|\Delta\nus |
        )
        \Big)
        \cdot
        \Big(
        \Vert\Delta a-\Delta\as \Vert+\\&
        |\Delta b-\Delta\bs |^2+
        |\Delta\nu -\Delta\nus |
        \Big)       
        \gtrsim
        D_4((\lambda ,G ),(\lambdas , \Gs ))
    \end{align*}
    where the last inequality come from the triangle inequality and Cauchy–Schwarz inequality. Now we could conclude that $D_4((\lambda,G),(\lambdas, \Gs)) \asymp \overline{D_4}((\lambda,G),(\lambdas, \Gs))$, i.e.
\begin{align*}
    D_4((\lambda, G), (\lambda^*, G_*)) &\asymp |\lambda - \lambdas| 
    (\Vert \Delta a\Vert + |\Delta b|^2+ | \Delta \nu |  )(\Vert\das\Vert+ | \dbs| ^2+| \dnus| ) \\
    &+ (\lambdas
    (\Vert\das\Vert+| \dbs|^2+ |\dnus |)+\lambda(\Vert\da\Vert+| \db|^2+ |\dnu |) )
    (\Vert a-\as\Vert+ | b-\bs |^2+ |\nu-\nus |)
\end{align*}
Combining this result with Theorem \ref{theorem:ConvergenceRateofDensityEstimation} and Lemma \ref{applemma:sigma-linear-f0-Gaussian-varphi-linear}, we have 
\begin{align}
    \label{eq:second_appthm:lower-mle-Gaussian-linear}
    \sup_{(\lambdas,\Gs)\in \Xi_1(l_n)  }
    \mathbb{E}_{p_{\lambdas,\Gs}} \Big( 
    (\Vert \Delta \widehat{a}_n\Vert^2 +  |\Delta \widehat{b}_n |^4+  |\Delta \widehat{\nu}_n  |^2 )(\Vert\das\Vert^2+  |\dbs |^4+ |\dnus |^2)
    |\widehat{\lambda}_n
    -\lambdas|^2 \Big) 
    &\lesssim \frac{\log n}{n},\\
    \label{eq:first_appthm:lower-mle-Gaussian-linear}
    \sup_{(\lambdas,\Gs)\in \Xi_1(l_n) }
    \mathbb{E}_{p_{\lambdas,\Gs}} 
    \Big( (\lambdas)^2 
    (\Vert\das\Vert^2+ |\dbs |^4+ |\dnus |^2)
    (\Vert\han-\as\Vert^2+ |\hbn-\bs |^4+ |\hnun-\nus |^2)
    \Big) 
    &\lesssim \frac{\log n}{n}.
\end{align}

We follow the same argument as in the proof of Theorem \ref{theorem:sigma-nonlinear-f0-Gaussian-varphi-nonlinear}. Applying the Cauchy-Schwartz inequality, we can bound  $\mathbb{E}_{p_{\lambdas,G_*}}\left(\Vert \Delta \as\Vert^4 +  | \Delta \bs |^8 +  |\Delta \nus| ^4 |\widehat{\lambda}_n-\lambdas|^2\right)$ by
\begin{align*}
    & \mathbb{E}_{p_{\lambdas,G_*}}\left((\Vert \Delta \as \Vert^2 +  |\Delta \bs |^4 +  | \Delta \nus |^2)(\Vert \dhan \Vert^2 +  |\dhbn  |^4 +  | \dhnun |^2) |\widehat{\lambda}_n-\lambdas|^2\right)\\
    +&\mathbb{E}_{p_{\lambdas,G_*}}\left((\Vert \Delta \as \Vert^2 +  |\Delta \bs |^4 +  | \Delta \nus |^2)(\Vert\dhan-\das\Vert^2+ |\dhbn-\dbs |^4+ |\dhnun-\dnus |^2)|\widehat{\lambda}_n-\lambdas|^2\right)\\
\end{align*}
The first term has been estimated in  equation~\eqref{eq:second_appthm:lower-mle-Gaussian-linear}. For the second term, if $\lambdas \neq 0$ it is in fact an implication of  equation~\eqref{eq:first_appthm:lower-mle-Gaussian-linear}, noting that $\limsup \frac{|\hat{\lambda}_n-\lambdas|}{\lambdas} \leq \frac{2}{\lambdas}$. Otherwise, if $\lambdas = 0$, we have 
\begin{align*}
    D_4((\widehat{\lambda}_n,\widehat{G}_n), (\lambda^*, G_{*})) &= \widehat{\lambda}_n(\Vert \Delta \widehat{a}_n \Vert^2 +  | \Delta \widehat{b}_n  |^4 +  | \Delta \widehat{\nu}_n  |^2) \\
    &+ \widehat{\lambda}_n(\Vert \Delta \widehat{a}_n \Vert +  | \Delta \widehat{b}_n  |^2 +  | \Delta \widehat{\nu}_n  |)(\Vert \Delta \widehat{a}_n - \das\Vert +  |\Delta \widehat{b}_n - \dbs |^2 +  |\Delta \widehat{\nu}_n - \dnus |)\\
    &\geq \widehat{\lambda}_n(\Vert \Delta \widehat{a}_n \Vert +  | \Delta \widehat{b}_n  |^2 +  | \Delta \widehat{\nu}_n  |)(\Vert \Delta \widehat{a}_n - \das\Vert +  |\Delta \widehat{b}_n - \dbs |^2 +  |\Delta \widehat{\nu}_n - \dnus |).
\end{align*}
The conclusion follows from direct application of Lemma
\ref{applemma:sigma-linear-f0-Gaussian-varphi-linear} and Theorem \ref{theorem:ConvergenceRateofDensityEstimation}.
\end{proof}

We will now prove Lemma \ref{applemma:sigma-linear-f0-Gaussian-varphi-linear}.
\begin{proof}[Proof of Lemma \ref{applemma:sigma-linear-f0-Gaussian-varphi-linear}]
Given $\overline{G}=(\Bar{a},\Bar{b},\Bar{\nu})$, where $\Bar{\lambda}\in[0,1]$ and $(\Bar{a},\Bar{b},\Bar{\nu})$ can be identical to $({a_0},{b_0},{\nu_0})$, we will prove that
\begin{itemize}
    \item[(i)] If $({a_0},{b_0},{\nu_0})\neq(\Bar{a},\Bar{b},\Bar{\nu})$ and $\Bar{\lambda}>0$, then
    \begin{align*}
     \lim_{\varepsilon\to 0}\inf_{(\lambda,G),(\lambda^*,G_*)}\left\{\frac{\|p_{\lambda, G} - p_{\lambda^*, G_*}\|_{\infty}}{D_1((\lambda,G),(\lambda^*,G_*))}:D_1((\lambda,G),(\overline{\lambda},\overline{G}))\vee D_1((\lambda^*, G_*),(\overline{\lambda},\overline{G}))\leq\varepsilon\right\}>0.
    \end{align*}
 
    \item[(ii)] If $({a_0},{b_0},{\nu_0})=(\Bar{a},\Bar{b},\Bar{\nu})$ or $({a_0},{b_0},{\nu_0})\neq(\Bar{a},\Bar{b},\Bar{\nu})$ and $\Bar{\lambda}=0$, then
    \begin{align*}
     \lim_{\varepsilon\to 0}\inf_{(\lambda,G),(\lambda^*,G_*)}\left\{\frac{\|p_{\lambda, G} - p_{\lambda^*, G_*}\|_{\infty}}{D_4((\lambda,G),(\lambda^*,G_*))}:D_4((\lambda,G),(\overline{\lambda},\overline{G}))\vee D_4((\lambda^*, G_*),(\overline{\lambda},\overline{G}))\leq\varepsilon\right\}>0.
    \end{align*}
\end{itemize}


Note that we can apply the arguments in the proof of Theorem~\ref{theorem:sigma-linear-f0-notGaussian} for part (i). Therefore, we will provide proof for only part (ii) in this section, particularly the most demanding setting when $({a_0},{b_0},{\nu_0})=(\Bar{a},\Bar{b},\Bar{\nu})$. 


Assume that the conclusion in part (ii) does not hold, we can find two sequences 
$G_n=(a_n,b_{n},\nu_{n})$
and 
$\Gsn=(a_n^*, b^*_{n},\nu^*_{n})$
and two sequences of mixing proportions $\lambda_n$ and $\lambda_n^* \in [0, 1]$
that satisfy the following properties
\begin{align*}
     \begin{cases}
        D_4((\lambda_n,G_n),(\overline{\lambda},\overline{G})) \to 0 ,\\
        D_4((\lambda^*_n,G_{*,n}),(\overline{\lambda},\overline{G})) \to 0 ,\\
        \|p_{\lambda_n,G_n}-p_{\lambda_n^*,G_{*,n}}\|_{\infty}/D_4((\lambda_n,G_n),(\lambda_n^*,G_{*,n}))\to 0.
     \end{cases}
 \end{align*}
when $n\to\infty$. WLOG, we assume that $\lambdasn\geq\lambdan$. Now, we take into account only the case when 
$\|(\Delta a_n,\Delta b_n, \Delta \nu_n ) \|\to0$ 
and 
$\|(\Delta a^*_n,\Delta b^*_n, \Delta \nu^*_n ) \|\to0$, 
while the other two cases can be argued in the same way as in the proof of Theorem~
\ref{theorem:sigma-nonlinear-f0-Gaussian-varphi-nonlinear}
, Case 2 and Case 3. Here $\Delta a_n = a_n - a_0$, $\Delta a_n^* = a_n^* - a_0$.


We consider 
\begin{align*}
    p_{\lambdan, \Gn}(X,Y)-p_{\lambdasn, \Gsn}(X,Y)
    &= 
    (\lambdasn-\lambdan)
    [
    f(Y|
    \varphi (a_0^{\top}X+b_0 ),\nu_0
    )
    -
    f(Y|
    \sigma \Big((a_n^*)^{\top}X+b_n^* \Big),\nu_n^*
    )
    ]\Bar{f}(X)\\
    &~+\lambdan[f(Y|
    \sigma \Big(a_n^{\top}X+b_n \Big),\nu\textbf{}
    )-f(Y|
    \sigma \Big((a_n^*)^{\top}X+b_n^* \Big),\nu_n^*
    )]\Bar{f}(X)\\
    &=S_n+T_n.
\end{align*}

WLOG, since $\varphi(\cdot)$ and $\sigma (\cdot)$ are the same linear function, for simplicity, we could just consider that they are both identity.
{Also for the simplicity of the proof, we will consider the parameter vector $a$ in the mean expert is just a one-dimension vector.}
By means of Taylor expansion up to the fourth order, we get that
\begin{align}
    S_n&=(\lambdasn-\lambdan)
    \Big[
    \sum_{|\alpha|=1}^4\frac{1}{\alpha!}
    (-\Delta a^*_n)^{\alpha_1}
    (-\Delta b^*_n)^{\alpha_2}
    (-\Delta \nu^*_n)^{\alpha_3}\nonumber\\
    &\hspace{2.8cm}~
    \times\frac{\partial^{|\alpha|}f}{\partial a^{\alpha_1}\partial b ^{\alpha_2}\partial\nu^{\alpha_3}}
    (Y|\sigma\left( (a_n^*)^{\top}X+b_n^*\right),\nu_n^*)+R_1(X,Y)
    \Big]\Bar{f}(X)\nonumber\\
    &=(\lambdasn-\lambdan)\Big[\sum_{|\alpha|=1}^4\frac{1}{\alpha!}
    \{-\Delta a^*_n\}^{\alpha_1}\{-\Delta b^*_n\}^{\alpha_2}\{-\Delta\nu^*_n\}^{\alpha_3}\nonumber\\
    &\hspace{2.8cm}~\times\frac{X^{\alpha_1}}{2^{\alpha_3}}\frac{\partial^{\alpha_1+\alpha_2+2\alpha_3}f}{\partial \sigma^{\alpha_1+\alpha_2+2\alpha_3}}
    (Y|\sigma\left( (a_n^*)^{\top}X+b_n^*\right),\nu_n^*)+R_1(X,Y)
    \Big]\Bar{f}(X)\nonumber\\
    &=(\lambdasn-\lambdan)\Big[\sum_{\alpha_1=0}^4\sum_{\ell=0}^{2(4-\alpha_1)}\sum_{\alpha_2,\alpha_3}\frac{1}{2^{\alpha_3}\alpha!}
    \{-\Delta a^*_n\}^{\alpha_1}\{-\Delta b^*_n\}^{\alpha_2}\{-\Delta\nu^*_n\}^{\alpha_3}\nonumber\\
    \label{eq:nondistinguishable_dependent_1}
    &\hspace{2.8cm}\times X^{\alpha_1}\frac{\partial^{\ell+\alpha_1}f}{\partial \sigma^{\ell+\alpha_1}}
    (Y|\sigma\left( (a_n^*)^{\top}X+b_n^*\right),\nu_n^*)+R_1(X,Y)
    \Big]\Bar{f}(X),
\end{align}
where $\alpha_2$ and $\alpha_3$ in the above sum satisfy $\alpha_2+2\alpha_3=\ell$ and 
\textcolor{black}{$0~\vee~1-\alpha_1\leq\alpha_2+\alpha_3\leq 4-\alpha_1$}
. In addition, $R_1(X,Y)$ is the Taylor remainder such that $(\lambdasn-\lambdan)|R_1(X,Y)|/D_4((\lambdan,\Gn),(\lambdasn, \Gsn)) \to 0$ as $n\to\infty$. Analogously, we have
\begin{align}
    T_n&=\lambdan\Big[\sum_{\alpha_1=0}^4\sum_{\ell=0}^{2(4-\alpha_1)}\sum_{\alpha_2,\alpha_3}\frac{1}{2^{\alpha_3}\alpha!}\{\Delta a_n-\Delta a_n^*\}^{\alpha_1}\{\Delta b_n-\Delta b_n^*\}^{\alpha_2}\{\Delta\nu_n-\Delta\nu_n^*\}^{\alpha_3}\nonumber\\
    \label{eq:nondistinguishable_dependent_2}
    &\hspace{2.8cm}\times X^{\alpha_1}\frac{\partial^{\ell+\alpha_1}f}{\partial \sigma^{\ell+\alpha_1}}
    (Y|\sigma\left( (a_n^*)^{\top}X+b_n^*\right),\nu_n^*)
    +R_2(X,Y)\Big]\Bar{f}(X),
\end{align}
where $R_2(X,Y)$ is the Taylor remainder such that $\lambdan|R_2(X,Y)|/D_4((\lambdan,\Gn),(\lambdasn, \Gsn)) \to0$.

Subsequently, we define
\begin{align*}
    \mathcal{F}_n=\Big\{X^{\alpha_1}\frac{\partial^{\ell+\alpha_1}f}{\partial \sigma^{\ell+\alpha_1}}
    (Y|\sigma\left( (a_n^*)^{\top}X+b_n^*\right),\nu_n^*)
    \Bar{f}(X):0\leq\alpha_1\leq 4,0\leq\ell\leq 2(4-\alpha_1)\Big\}.
\end{align*}
From equations (\ref{eq:nondistinguishable_dependent_1}) and (\ref{eq:nondistinguishable_dependent_2}), the quantity 
$p_{\lambdan, \Gn}(X,Y)-p_{\lambdasn, \Gsn}(X,Y)$
can be viewed as a linear combination of elements in the set $\mathcal{F}_n$. 
Let us denote by 
$E_{\alpha_1,\ell}$ 
the coefficients of those elements in the representation of 
$[p_{\lambdan, \Gn}(X,Y)-p_{\lambdasn, \Gsn}(X,Y)]$, i.e.
\begin{align*}
    E_{\alpha_1,\ell}
    =\sum_{\alpha_2+2\alpha_3=\ell}\frac{1}{2^{\alpha_3}\textcolor{black}{\alpha!}}
    \Big[
    &
    (\lambdasn-\lambdan)
    {(-\Delta\asn)}^{\alpha_1}{(-\Delta\bsn)}^{\alpha_2}{(-\Delta\nusn)}^{\alpha_3}
    \\&
    +\lambdan\{\Delta \an -\Delta\asn\}^{\alpha_1}\{\Delta\bn-\Delta\bsn\}^{\alpha_2}\{\Delta\nun-\Delta\nusn\}^{\alpha_3}
    \Big].
\end{align*}
Assume by contrary that all the coefficients of elements in the set $\mathcal{F}_n$ in the representation of 
$[p_{\lambdan, \Gn}(X,Y)-p_{\lambdasn, \Gsn}(X,Y)]/D_4((\lambdan,\Gn),(\lambdasn, \Gsn))$
vanish when $n$ tends to infinity, for all $0\leq\alpha_1\leq 4$ and $0~\vee~1-\alpha_1\leq\ell\leq 2(4-\alpha_1)$, i.e.
\begin{align}
    \label{eq:nondistinguishable_dependent_3}
    E_{\alpha_1,\ell}/D_4((\lambdan,\Gn),(\lambdasn, \Gsn)) \to 0.
\end{align}
Now, we will divide the proof into two main scenarios:

\subsubsection*{Case 1:}
The following ratio does not tend to infinity: 
$$
R_n = \dfrac{\lambdan(|\Delta\an|+|\Delta\bn|^2+|\Delta\nun|)
+\lambdasn(|\Delta\asn|+|\Delta\bsn|^2+|\Delta\nusn|)}{\lambdan(|\Delta\an-\Delta\asn|+|\Delta\bn-\Delta\bsn|^2+|\Delta \nun-\Delta\nusn|)}\not\to\infty.$$ 

It follows that as $n$ is large enough
\begin{align*}
    D_4((\lambdan,\Gn),(\lambdasn,\Gsn))&\lesssim(\lambdasn-\lambdan)
    (|\Delta\asn|^2+|\Delta\bsn|^4+|\Delta\nusn|^2)\\
    &~+\lambdan(
    |\Delta\an-\Delta\asn|^2+
    |\Delta\bn-\Delta\bsn|^4+
    |\Delta\nun-\Delta\nusn|^2
    )\\
    &=:H_n.
\end{align*}

Putting this result together with equation (\ref{eq:nondistinguishable_dependent_3}), we have $E_{\alpha_1,\ell}/H_n\to 0$ as $n\to\infty$ for any feasible $\alpha_1,\ell$. Then, for $\alpha_1=2$ and $\ell=0$, we obtain
\begin{align*}
   \frac{E_{2,0}}{H_n}=\dfrac{(\lambdasn-\lambdan)|\Delta\asn|^2+\lambdan|\Delta\an-\Delta\asn|^2}{H_n}\to0,
\end{align*}
which leads to
\begin{align*}
    \dfrac{(\lambdasn-\lambdan)(|\Delta\bsn|^4+|\Delta\nusn|^2)+\lambdan(|\Delta\bn-\Delta\bsn|^4+|\Delta\nun-\Delta\nusn|^2)}{H_n}\to1.
\end{align*}
As a consequence, we get
\begin{align*}
    T_{0,\ell}:&=
    \dfrac
    {\sum_{\substack{\alpha_2+2\alpha_3=\ell\\ \alpha_2+\alpha_3\leq 4}}\frac{1}{2^{\alpha_3}\alpha_2!\alpha_3!}[(\lambdasn-\lambdan)\{-\Delta\bsn\}^{\alpha_2}\{-\Delta\nusn\}^{\alpha_3}+\lambdan\{\Delta\bn-\Delta\bsn\}^{\alpha_2}\{\Delta\nun-\Delta\nusn\}^{\alpha_3}]}
    {(\lambdasn-\lambdan)(|\Delta\bsn|^4+|\Delta\nusn|^2)+\lambdan(|\Delta\bn-\Delta\bsn|^4+|\Delta\nun-\Delta\nusn|^2)}
    \\
    &=\dfrac{E_{0,\ell}}
    {(\lambdasn-\lambdan)(|\Delta\bsn|^4+|\Delta\nusn|^2)+\lambdan(|\Delta\bn-\Delta\bsn|^4+|\Delta\nun-\Delta\nusn|^2)}
    \\
    &=\dfrac{E_{0,\ell}/H_n}
    {
    \Big(
    {(\lambdasn-\lambdan)(|\Delta\bsn|^4+|\Delta\nusn|^2)+\lambdan(|\Delta\bn-\Delta\bsn|^4+|\Delta\nun-\Delta\nusn|^2)}
    \Big)
    /H_n}
    \\
    &\to \dfrac{0}{1}=0.
\end{align*}
Since $(\lambdasn-\lambdan)$ and $\lambdan$ play the same role, we can assume that $\lambdasn-\lambdan\leq \lambdan$ for all $n\to\infty$. 

\textbf{Case 1.1}: $(\lambdasn-\lambdan)/\lambdan\not\to0$ as $n\to\infty$.

Under this setting, we define $p_n=\max\{\lambdasn-\lambdan,\lambdan\}$ and
\begin{align*}
    M_n=\max\{|\Delta\bsn|,|\Delta\bn-\Delta\bsn|,|\Delta\nusn|^{1/2},|\Delta\nun-\Delta\nusn|^{1/2}\}.
\end{align*}
In addition, we denote
\begin{align*}
(\lambdasn-\lambdan)/p_n\to c_1^2 &,\quad \lambdan/p_n\to c_2^2, \\
    \Delta\bsn/M_n\to -a_1 &,\quad (\Delta\bn-\Delta\bsn)/M_n\to a_2,\\
    \frac{1}{2}\Delta\nusn/M_n^2\to -b_1 &,\quad \frac{1}{2}(\Delta\nun-\Delta\nusn)/M_n^2\to b_2.
\end{align*}
Note that at least one among $a_1,b_1,a_2,b_2$ and both $c_1,c_2$ are different from zero. Next, for each $\ell\in[4]$, 
we divide both the numerators and the denominators of $T_{0,\ell}$ by $p_nM_n^{\ell}$, then we achieve the following system of polynomial equations
\begin{align*}
    \sum_{\alpha_2+2\alpha_3=\ell}\frac{1}{\alpha_2!\alpha_3!}(c_1^2a_1^{\alpha_2}b_1^{\alpha_3}+c_2^2a_2^{\alpha_2}b_2^{\alpha_3}),\quad \forall \ell\in[4].
\end{align*}
According to Propsition 2.1 in \citet{Ho-Nguyen-Ann-16}, this system admits only the trivial solution $a_1=b_1=a_2=b_2=0$, which is a contradiction. Thus, the setting that $(\lambdasn-\lambdan)/\lambdan\not\to0$ as $n\to\infty$ cannot occur.

\textbf{Case 1.2}: $(\lambdasn-\lambdan)/\lambdan\to0$, or equivalently $\lambdasn/\lambdan\to 1$ as $n\to\infty$.

Now, we consider the following quantity:
\begin{align*}
    T_{0,1}=\dfrac{(\lambdasn-\lambdan)(-\Delta\bsn)+\lambdan(\Delta\bn-\Delta\bsn)}{(\lambdasn-\lambdan)(|\Delta\bsn|^4+|\Delta\nusn|^2)+\lambdan(|\Delta\bn-\Delta\bsn|^4+|\Delta\nun-\Delta\nusn|^2)}.
\end{align*}
By dividing both the numerator and the denominator of $T_{0,1}$ by $\lambdan M_n$ (with a convention that if the new denominator goes to 0, then the new numerator also goes to 0), we get
\begin{align*}
    \dfrac{\lambdasn-\lambdan}{\lambdan}
    \cdot
    \dfrac{(-\Delta\bsn)}{M_n}+\dfrac{\Delta\bn-\Delta\bsn}{M_n}\to 0.
\end{align*}
As $(\lambdasn-\lambdan)/\lambdan\to 0$ and $|\Delta\bsn|\leq M_n$, then the first term in the above sum tends to 0. This implies that
\begin{align}
    \label{eq:nondistinguishable_dependent_4}
    (\Delta\bn-\Delta\bsn)/M_n\to0.
\end{align}
Similarly, we continue to divide both the numerator and the denominator of $T_{0,2}$ by $\lambdan M_n^2$, and obtain that
\begin{align*}
    \dfrac{(\lambdasn-\lambdan)(-\Delta\bsn)^2+\lambdan(\Delta\bn-\Delta\bsn)^2}{\lambdan M_n^2}+\dfrac{(\lambdasn-\lambdan)(-\Delta\nusn)+\lambdan(\Delta\nun-\Delta\nusn)}{\lambdan M_n^2}\to0.
\end{align*}
As demonstrated above, the first term in the above sum goes to 0 when $n\to\infty$, therefore, the second term also tends to 0, i.e.
\begin{align*}
    \dfrac{\lambdasn-\lambdan}{\lambdan}
    \cdot
    \dfrac{\Delta\nusn}{M_n^2}+\dfrac{\Delta\nun-\Delta\nusn}{M_n^2}\to 0.
\end{align*}
Recall that $(\lambdasn-\lambdan)/\lambdan\to 0$ and $|\Delta\nusn|\leq M_n^2$, then $\dfrac{\lambdasn-\lambdan}{\lambdan}\cdot\dfrac{\Delta\nusn}{M_n^2}\to 0$, leading to 
\begin{align}
    \label{eq:nondistinguishable_dependent_5}
    \dfrac{\Delta\nun-\Delta\nusn}{M_n^2}\to 0
\end{align}
Assume by contrary that $M_n\in\{|\Delta\bn-\Delta\bsn|,|\Delta\nun-\Delta\nusn|^{1/2}\}$, then it follows from equations (\ref{eq:nondistinguishable_dependent_4}) and (\ref{eq:nondistinguishable_dependent_5}) that
\begin{align*}
    1=
    \dfrac{
    \max\{
    |\Delta\bn-\Delta\bsn|^2,
    |\Delta\nun-\Delta\nusn|
    \}
    }{M_n^2}
    \to 0,
\end{align*}
which is a contradiction. As a result, $M_n\in\{|\Delta\bsn|,|\Delta\nusn|^{1/2}\}$. 

For the sake of simplicity, we will consider only the setting when $M_n=|\Delta\bsn|$, whereas the other case can be argued similarly. Subsequently, we continue to consider two small cases of this setting.

\textbf{Case 1.2.1}: $(\lambdasn-\lambdan)|\Delta\bsn|^4\leq\lambdan|\Delta\bn-\Delta\bsn|^4$ for all $n$.

Assume by contrary that 
$|\Delta\bn-\Delta\bsn|/|\Delta\nun-\Delta\nusn|^{1/2}\not\to0$. By dividing both the numerator and the denominator of $T_{0,1}$ by 
$\lambdan|\Delta\bn-\Delta\bsn|$, we get
\begin{align*}
    \dfrac{\lambdasn-\lambdan}{\lambdan}\cdot
    \dfrac{
    \Delta\bsn
    }{
    \Delta\bn-\Delta\bsn
    }\to 1.
\end{align*}

From equation
(\ref{eq:nondistinguishable_dependent_4})
, we have 
$|\Delta\bn-\Delta\bsn|/|\Delta\bsn|\to0$
, which together with the above equation deduces that
\begin{align*}
    \dfrac{
    \lambdasn-\lambdan
    }{
    \lambdan
    }\cdot
    \dfrac{
    |\Delta\bsn|^4
    }{
    |\Delta\bn-\Delta\bsn|^4
    }\to\infty.
\end{align*}
This contradicts the assumption of Case 1.2.1. Consequently, we must have that 
\begin{align*}
    \frac{|\Delta\bn-\Delta\bsn|~~}{~|\Delta\nun-\Delta\nusn|^{1/2}}\to0.
\end{align*}
WLOG, assume that $\Delta\nun-\Delta\nusn>0$ for all $n$. Additionally, let us denote
\begin{align*}
    -\Delta\bsn &=r_1^n(\Delta\nun-\Delta\nusn)^{1/2},\\
    \Delta\nusn&=r_2^n(\Delta\nun-\Delta\nusn).
\end{align*}
By dividing the numerators and the denominators of $T_{0,\ell}$ by $\lambdan(\Delta\nun-\Delta\nusn)^{\ell/2}$ for any $\ell\in[3]$, then the new denominators will go to zero. Thus, all the new numerators also tend to 0, i.e.
\begin{align*}
    \begin{cases}
        \vspace{0.2cm}\dfrac{\lambdasn-\lambdan}{\lambdan}\cdot r_1^n\to 0,\\
        \vspace{0.2cm}\dfrac{\lambdasn-\lambdan}{\lambdan}\cdot[(r_1^n)^2+r_2^n]+1\to 0,\\
        \dfrac{\lambdasn-\lambdan}{\lambdan}\cdot\Big[\dfrac{(r_1^n)^3}{6}+\dfrac{r_1^nr_2^n}{2}\Big]\to 0.
    \end{cases}
\end{align*}
Note that equation~(\ref{eq:nondistinguishable_dependent_5}) indicates that $1/|r_1^n|\to0$ as $n\to\infty$. Therefore, the third limit in the above system is reduced to
\begin{align*}
   \dfrac{\lambdasn-\lambdan}{\lambdan}\cdot\Big[\dfrac{(r_1^n)^2}{3}+r_2^n\Big]\to0.
\end{align*}
Putting this result together with the second limit, we obtain that
\begin{align*}
     \dfrac{\lambdasn-\lambdan}{\lambdan}\cdot\dfrac{2}{3}(r_1^n)^2+1\to 0.
\end{align*}
However, the left hand side of the above equation cannot converge to 0 as it is not less than one. Thus, case 1.2.1 cannot occur.

\textbf{Case 1.2.2}: $(\lambdasn-\lambdan)|\Delta\bsn|^4>\lambdan|\Delta\bn-\Delta\bsn|^4$ for all $n$.

It is worth noting that we can achieve that $(\lambdasn-\lambdan)|\Delta\bsn|^4>\lambdan|\Delta\nun-\Delta\nusn|^2$ for all $n$ by arguing similarly to case 1.2.1. 
Next, let us denote
\begin{align*}
    \Delta\bn-\Delta\bsn &=s_1^n(-\Delta\bsn),\\
    -\Delta\nusn &=s_2^n(\Delta\bsn)^2,\\
    \Delta\nun-\Delta\nusn &=s_3^n(\Delta\bsn)^2.
\end{align*}
For $\ell\in[4]$, we divide both the numerators and the denominators of $T_{0,\ell}$ by 
$(\lambdasn-\lambdan)(-\Delta\bsn)^{\ell}$. Then, we see that the new numerators goes to 0, i.e.
\begin{align*}
    1+\frac{\lambdan}{\lambdasn-\lambdan}\cdot s_1^n&\to0,\\
    \Big[1+\frac{\lambdan}{\lambdasn-\lambdan}\cdot (s_1^n)^2\Big]+s_2^n+\frac{\lambdan}{\lambdasn-\lambdan}\cdot s_3^n&\to 0,\\
    \frac{1}{6}\Big[1+\frac{\lambdan}{\lambdasn-\lambdan}\cdot(s_1^n)^3\Big]+\frac{1}{2}\Big[s_2^n+\frac{\lambdan}{\lambdasn-\lambdan}\cdot s_1^ns_3^n\Big]&\to 0,\\
    \frac{1}{24}\Big[1+\frac{\lambdan}{\lambdasn-\lambdan}\cdot(s_1^n)^4\Big]+\frac{1}{4}\Big[s_2^n+\frac{\lambdan}{\lambdasn-\lambdan}\cdot(s_3^n)^2\Big]+\frac{1}{8}\Big[(s_2^n)^2+\frac{\lambdan}{\lambdasn-\lambdan}\cdot(s_3^n)^2\Big]&\to 0.
\end{align*}
Recall that $M_n=|\Delta\bsn|$, then $|s_i^n|\leq 1$ for all $i\in[3]$. Suppose that $s_i^n\to s_i$ for any $i\in[3]$. From equations~(\ref{eq:nondistinguishable_dependent_4}) and (\ref{eq:nondistinguishable_dependent_5}), we know that $s_1=s_3=0$. Looking at the above system, the first and third limits yield that $s_2=-\frac{1}{3}$, while the second and the fourth indicate that
\begin{align*}
    \frac{1}{6}+s_2+\frac{1}{2}s_2^2=0,
\end{align*}
which contradicts the fact that $s_2=-\frac{1}{3}$. Thus, case 1.2.2 does not hold.
\subsubsection*{Case 2:}
The following ratio tends to infinity:
\begin{equation}
\label{eq:lqd_rn_infty}
    R_n = \dfrac{
\lambdan(
|\Delta\an|+|\Delta\bn|^2+|\Delta\nun|
)
+\lambdasn(
|\Delta\asn|+|\Delta\bsn|^2+|\Delta\nusn|
)
}{
\lambdan(
|\Delta\an-\Delta\asn|+
|\Delta\bn-\Delta\bsn|^2+
|\Delta\nun-\Delta\nusn|
)
}\to\infty.
\end{equation}

Recall we have defined
    \begin{align*}
        &\overline{D_4}((\lambdan,\Gn),(\lambdasn, \Gsn))
        =
        (\lambdasn-\lambdan)
        \Big(
        |\Delta\an|+|\Delta\bn|^2+|\Delta\nun|
        \Big)
        \Big(
        |\Delta\asn|+|\Delta\bsn|^2+|\Delta\nusn|
        \Big)
        \\
        &+\Big(
        \lambdan
        (
        |\Delta\an|+|\Delta\bn|^2+|\Delta\nun|
        )
        +\lambdasn
        (
        |\Delta\asn|+|\Delta\bsn|^2+|\Delta\nusn|
        )
        \Big)\cdot
        \Big(
        |\Delta\an-\Delta\asn|+
        |\Delta\bn-\Delta\bsn|^2+
        |\Delta\nun-\Delta\nusn|
        \Big).        
    \end{align*}

In the proof of Theorem \ref{theorem:sigma-linear-f0-Gaussian-varphi-linear}, we already have $D_4((\lambdan,\Gn),(\lambdasn, \Gsn)) \asymp \overline{D_4}((\lambdan,\Gn),(\lambdasn, \Gsn))$.
Recall equation (\ref{eq:nondistinguishable_dependent_3}), we have $E_{\alpha_1,\ell}/{D_4((\lambdan,\Gn),(\lambdasn, \Gsn))}\to 0$ as $n\to\infty$ for any feasible $\alpha_1,\ell$. Then, for $\alpha_1=[4]$ and $0\leq\ell\leq 2(4-\alpha_1)$, we would have that
    \begin{align*}
        F_{\alpha_1,\ell}={E_{\alpha_1,\ell}}/{\overline{D_4}((\lambdan,\Gn),(\lambdasn, \Gsn))} \to 0 .
    \end{align*}
    
For convenience, we denote and recall the following terms that are frequently employed in our proof

\begin{align}
    \label{eq:lqd_an}
    A_n &:= -F_{1,0} = \dfrac{(\lambdasn - \lambdan)\dasn + \lambdan(\dasn-\dan)}{\dfourbar} =  \dfrac{\lambdasn\dasn - \lambdan\dan}{\dfourbar}\to 0\\
    \label{eq:lqd_bn}
    B_n &:= -F_{0,1} = \dfrac{(\lambdasn - \lambdan)\dbsn + \lambdan(\dbsn-\dbn)}{\dfourbar} =  \dfrac{\lambdasn\dbsn - \lambdan\dbn}{\dfourbar}\to 0\\
    \label{eq:lqd_cn}
    C_n &:= F_{0,2} = \dfrac{(\lambdasn - \lambdan)(\dbsn)^2 + \lambdan(\dbsn-\dbn)^2 + (\lambdasn - \lambdan)(-\dnusn) + \lambdan(\dnun-\dnusn)}{\dfourbar} \to 0. \\
     \label{eq:lqd_f03}
    F_{0,3} &:= \dfrac{1}{6}\dfrac{(\lambdasn - \lambdan)(-\dbsn)^3
    +\lambdan(\dbn-\dbsn)^3}{\dfourbar}\\
    &\notag+ \dfrac{1}{2}\dfrac{(\lambdasn - \lambdan)(-\dnusn)(-\dbsn) + \lambdan(\dbn-\dbsn)(\dnun-\dnusn)}{\dfourbar} \to 0\\
    \label{eq:lqd_f04}
    F_{0,4} &:= \frac{1}{24}\dfrac{(\lambdasn - \lambdan)(-\dbsn)^4+\lambdan(\dbn-\dbsn)^4}{\dfourbar}
    + \frac{1}{8}\dfrac{(\lambdasn - \lambdan)(-\dnusn)^2+ \lambdan(\dnun-\dnusn)^2}{\dfourbar}\\
    &\notag +  \frac{1}{4}\dfrac{(\lambdasn - \lambdan)(-\dnusn)(-\dbsn)^2 + \lambdan(\dbn-\dbsn)(\dnun-\dnusn)}{\dfourbar} \to 0.
\end{align}
    Also, let $M'_n = \max\{|\dan|,|\dbn|^2, |\dnun|,|\dasn|,|\dbsn|^2, |\dnusn|\}$. We consider sub-cases based on the asymptotic behavior of $\lambdasn/\lambdan$. 

\textbf{Case 2.1}: $\lambdasn/\lambdan\not\to 1$ and $\lambdasn/\lambdan\not\to\infty$. 

From the assumption about $R_n$ at the beginning of Case 2, we would have 
\begin{align}
\label{eq:lqd_nondistinguishable_dependent_6}
    \frac{|\Delta\an-\Delta\asn|}{\mnp}\to 0~,~
    \frac{|\Delta\bn-\Delta\bsn|^2}{\mnp}\to 0,~
    \frac{|\Delta\nun-\Delta\nusn|}{\mnp}\to 0 .
\end{align}
Given the symmetry between $(|\Delta\an|, |\Delta\bn|^2, |\Delta\nun|)$ and $(|\Delta\asn|,|\Delta\bsn|^2, |\Delta\nusn|)$, we can assume, without loss of generality, that $\mnp \in \max\{|\Delta\an|,|\Delta\bn|^2,|\Delta\nun|\}$. Under this assumption, we encounter three distinct cases.

\textbf{Case 2.1.1: } $M'_n = |\dbn|^2$. In this case, given that $\frac{|\Delta\bn-\Delta\bsn|^2}{\mnp}\to 0$, which is equivalent to $\frac{|\Delta\bn-\Delta\bsn|^2}{|\dbn|^2}\to 0$, we achieve that $\frac{\Delta\bn}{\Delta\bsn} \to 1$. 
Noting that  $\lambdasn/\lambdan\not\to 1$, we would have
\begin{equation*}
    (\lambdasn-\lambdan)
        \Big(
        |\Delta\an|+|\Delta\bn|^2+|\Delta\nun|
        \Big)
        \Big(
        |\Delta\asn|+|\Delta\bsn|^2+|\Delta\nusn|
        \Big) \asymp \lambdasn |\Delta\bn|^4\\
\end{equation*}
and
\begin{align*}
    &\Big(
        \lambdan
        (
        |\Delta\an|+|\Delta\bn|^2+|\Delta\nun|
        )
        +\lambdasn
        (
        |\Delta\asn|+|\Delta\bsn|^2+|\Delta\nusn|
        )
        \Big)
        \Big(
        |\Delta\an-\Delta\asn|+
        |\Delta\bn-\Delta\bsn|^2+
        |\Delta\nun-\Delta\nusn|
        \Big)\\
        &\lesssim \lambdasn|\Delta\bn|^2 \Big(
        |\Delta\an-\Delta\asn|+
        |\Delta\bn-\Delta\bsn|^2+
        |\Delta\nun-\Delta\nusn|
        \Big) \lesssim  \lambdasn|\Delta\bn|^4. 
\end{align*}
Thus, we have an asymptotic estimation $\dfourbar\asymp \lambdasn|\Delta\bn|^4$. By substituting this estimation to the equation \eqref{eq:lqd_bn}, we would achieve
\begin{equation*}
    (\dbsn)^3 B_n \asymp  \dfrac{(\lambdasn - \lambdan)(\dbsn)^4 + \lambdan(\dbsn-\dbn)(\dbsn)^3}{\lambdasn|\Delta\bn|^4} \to 0
\end{equation*}
Meanwhile, following the discussion at the beginning of Case 2.1.1, we have the estimation for the terms in previous equation
\begin{equation*}
    \dfrac{(\lambdasn - \lambdan)|\dbsn|^4}{\lambdasn|\Delta\bn|^4} \asymp 1, \quad \dfrac{ \lambdan(|\dbsn-\dbn||\dbsn|^3}{\lambdasn|\Delta\bn|^4} \lesssim \dfrac{\lambdan(\dbsn-\dbn)}{\lambdan|\Delta\bn|} \to 0
\end{equation*}
This leads to a contradiction, so this case is invalid.

\textbf{Case 2.1.2: } $M'_n = |\dnun|$. In this case, given that $\frac{|\Delta\nun-\Delta\nusn|}{\mnp}\to 0$, which is equivalent to $\frac{|\Delta\nun-\Delta\nusn|}{|\dnun|}\to 0$, we achieve that $\frac{\Delta\nun}{\Delta\nusn} \to 1$. Using the same argument as in the Case 2.1.1, we would achieve the asymptotic behavior $\dfourbar \asymp \lambdasn|\dnun|^2$. By substituting this estimation to the equation \eqref{eq:lqd_cn}, we would have
\begin{equation*}
    |\dnun| C_n \asymp  \dfrac{(\lambdasn - \lambdan)|\dbsn|^2|\dnun| + \lambdan(\dbsn-\dbn)^2|\dnun| + (\lambdasn - \lambdan)(-\dnusn)|\dnun| + \lambdan(\dnun-\dnusn)|\dnun|}{\lambdasn|\Delta\nun|^2} \to 0.
\end{equation*}
Following the discussion at the beginning of Case 2.1 and Case 2.1.2, noting that $\lambdasn/\lambdan \not\to \infty$, we would have the following estimation for the terms in previous equation
\begin{equation*}
    \dfrac{\lambdan(\dbsn-\dbn)^2|\dnun|}{\lambdasn|\Delta\nun|^2} \asymp \dfrac{(\dbsn-\dbn)^2}{|\Delta\nun|} \to 0, \quad  \dfrac{\lambdan(\dnun-\dnusn)|\dnun|}{\lambdasn|\Delta\nun|^2} \asymp \dfrac{\dnun-\dnusn}{|\Delta\nun|} \to 0.
\end{equation*}
This would lead to
\begin{equation*}
    \dfrac{(\dbsn)^2 - \dnusn}{\nusn} \asymp \dfrac{(\lambdasn-\lambdan)(\dbsn)^2|\dnun| + (\lambdasn-\lambdan)(-\nusn)|\dnun|}{\lambdasn|\Delta\nun|^2} \to 0.
\end{equation*}
In other words, $(\dbsn)^2/\dnusn \to 1$. By substituting this estimation to $B_n$ as in Case 2.1.1, we receive a contradiction. 

\textbf{Case 2.1.3: } $M_n' = |\dan|$. Using the same argument as in Case 2.1.1 and 2.1.2, we have $\dfourbar \asymp \lambdasn|\dasn|^2$. By substituting this estimation to equation~\eqref{eq:lqd_an}, we would have 
\begin{equation*}
    \dasn A_n \asymp \dfrac{(\lambdasn-\lambdan)(\dasn)^2+\lambdasn(\dasn-\dan)\dasn}{\lambdan|\dasn|^2} \to 1. 
\end{equation*}
Meanwhile, following the estimation in equation~\eqref{eq:lqd_nondistinguishable_dependent_6}, noting that $\lambdasn/\lambdan \not\to 1$, we have 
\begin{equation*}
    \dfrac{(\lambdasn-\lambdan)(\dasn)^2}{\lambdasn|\dasn|^2} \asymp 1, \quad \dfrac{\lambdan(\dasn-\dan)\dasn}{\lambdasn|\dasn|^2} \asymp \dfrac{\dasn-\dan}{\dasn} \to 0. 
\end{equation*}
This leads to a contradiction, so this case is also invalid. 

\textbf{Case 2.2: } $\lambdasn/\lambdan \to 1$. Then the equation \eqref{eq:lqd_nondistinguishable_dependent_6} still holds. Given the symmetry between $(|\Delta\an|, |\Delta\bn|^2, |\Delta\nun|)$ and $(|\Delta\asn|,|\Delta\bsn|^2, |\Delta\nusn|)$, we can assume, without loss of generality, that $\mnp \in \max\{|\Delta\an|,|\Delta\bn|^2,|\Delta\nun|\}$. Under this assumption, we consider three distinct cases.

\textbf{Case 2.2.1: } $\mnp = |\dbn|^2$. Using the same argument as at the beginiing of Case 2.1.1, noting that the equation \eqref{eq:lqd_nondistinguishable_dependent_6} holds, we achieve that $\dbsn/\dbn \to 1$. Assume that 
\begin{align*}
    \max{
    \Big\{
    \frac{\lambdan|\Delta\an-\Delta\asn|}{(\lambdasn-\lambdan)|\Delta\bn|^2},~
    \frac{\lambdan|\Delta\bn-\Delta\bsn|^2}{(\lambdasn-\lambdan)|\Delta\bn|^2},~
    \frac{\lambdan|\Delta\nun-\Delta\nusn|}{(\lambdasn-\lambdan)|\Delta\bn|^2}
    \Big\}
    }\to\infty,
\end{align*}
then we could conclude that
\begin{align*}
    B_{2,n}:=\frac{(\lambdasn-\lambdan)(\Delta\bn\Delta\bsn)^2}{\overline{D_4}((\lambdan,\Gn),(\lambdasn, \Gsn))}
    \leq
    \min
    \Big\{
    \frac{(\lambdasn-\lambdan)|\Delta\bn|^2}{\lambdan|\Delta\an-\Delta\asn|},~
    \frac{(\lambdasn-\lambdan)|\Delta\bn|^2}{\lambdan|\Delta\bn-\Delta\bsn|^2},~
    \frac{(\lambdasn-\lambdan)|\Delta\bn|^2}{\lambdan|\Delta\nun-\Delta\nusn|}
    \Big\}
    \to 0.
\end{align*}
From this, noting that $\dbn \asymp \dbsn$, we could have
\begin{equation}
    \label{eq:lqd:case2.2.1_eq1}
    \dfrac{(\lambdasn-\lambdan)(|\dasn|+|\dbsn|^2+|\dnusn|)(|\dan|+|\dbn|^2+|\dnun|)}{\dfourbar} \asymp \dfrac{(\lambdasn-\lambdan)(\Delta\bn\Delta\bsn)^2}{\dfourbar} \to 0. 
\end{equation}
Meanwhile, we also could have following estimation. Firstly,
\begin{align*}
    \frac{\lambdan(\Delta\bn)^2(\Delta\an-\Delta\asn)}{\overline{D_4}((\lambdan,\Gn),(\lambdasn, \Gsn))} &= \frac{(\lambdan\Delta\an-\lambdasn\Delta\asn)(\Delta\bn)^2}{\overline{D_4}((\lambdan,\Gn),(\lambdasn, \Gsn))}+ \frac{(\lambdasn-\lambdan)(\Delta\bn)^2\Delta\asn}{\overline{D_4}((\lambdan,\Gn),(\lambdasn, \Gsn))}\\
    &= -(\Delta\bn)^2A_n+\dfrac{\dasn}{(\dbsn)^2}B_{2,n}.
\end{align*}
Noting that $A_n,B_{2,n} \to 0$, and $|\dasn| \lesssim |\dbsn|^2$, we have 
\begin{equation*}
    \frac{\lambdan(\Delta\bn)^2(\Delta\an-\Delta\asn)}{\overline{D_4}((\lambdan,\Gn),(\lambdasn, \Gsn))} \to 0.
\end{equation*}
Secondly, 
\begin{align*}
    \frac{\lambdan(\Delta\bn)^2|\Delta\bn-\Delta\bsn|^2}{\overline{D_4}((\lambdan,\Gn),(\lambdasn, \Gsn))}&=\frac{(\lambdan\Delta\bn-\lambdasn\Delta\bsn)(\Delta\bn)^2(\Delta\bn-\Delta\bsn)}{\overline{D_4}((\lambdan,\Gn),(\lambdasn, \Gsn))}+ \frac{(\lambdasn-\lambdan)(\Delta\bn)^2(\Delta\bn-\Delta\bsn)\Delta\bsn}{\overline{D_4}((\lambdan,\Gn),(\lambdasn, \Gsn))}\\
    &= (\Delta\bn)^2(\Delta\bn-\Delta\bsn)B_n + \dfrac{\Delta\bn-\Delta\bsn}{\dbsn}B_{2,n}
\end{align*}
Given that $B_n,B_{2,n} \to 0$ and $\frac{\Delta\bn-\Delta\bsn}{\dbsn} \to 0$ (an implication from equation~\eqref{eq:lqd_nondistinguishable_dependent_6}), we achieve 
\begin{equation*}
    \frac{\lambdan(\Delta\bn)^2|\Delta\bn-\Delta\bsn|^2}{\overline{D_4}((\lambdan,\Gn),(\lambdasn, \Gsn))} \to 0.
\end{equation*}
Thirdly, by multiplying $C_n$ with $(\dbn)^2$, we achieve 
\begin{align*}
   (\dbn)^2C_n &= \dfrac{(\lambdasn - \lambdan)(\dbsn)^4 + \lambdan(\dbsn-\dbn)^2(\dbn)^2 + (\lambdan\dnun - \lambdasn\dnusn)(\dbn)^2}{\dfourbar}\\
   &=\dfrac{(\lambdasn - \lambdan)(\dbsn)(\dbn)^3 + (\lambdasn\dbsn - \lambdan\dbn)(\dbsn-\dbn)(\dbn)^2+ (\lambdan\dnun - \lambdasn\dnusn)(\dbn)^2}{\dfourbar}\\
   &= \frac{\dbn}{\dbsn}B_{2,n} + (\dbsn-\dbn)(\dbn)^2 B_n + \dfrac{(\lambdan\dnun - \lambdasn\dnusn)(\dbn)^2}{\dfourbar}.
\end{align*}
As $C_n,B_{2,n},B_n \to 0$, and $\dbsn/\dbn \to 1$ as discussed at the beginning of Case 2.2.1, we obtain
\begin{equation*}
    NB_n:= \dfrac{(\lambdan\dnun - \lambdasn\dnusn)(\dbn)^2}{\dfourbar} \to 0. 
\end{equation*}
Thus, we have the following estimation:
\begin{align*}
    \frac{\lambdan(\Delta\bn)^2|\Delta\nusn-\Delta\nun|^2}{\overline{D_4}((\lambdan,\Gn),(\lambdasn, \Gsn))} &= \dfrac{(\lambdan\dnun - \lambdasn\dnusn)(\dbn)^2}{\dfourbar}  + \dfrac{(\lambdasn-\lambdan)(\dbn)^2\dnusn}{\dfourbar}\\
    &= NB_n + \dfrac{\dnusn}{(\dbsn)^2}B_{2,n} \to 0
\end{align*}
thanks to $NB_n, B_{2,n} \to 0$, and $\dfrac{|\dnusn|}{(\dbsn)^2} \lesssim \dfrac{|\dnusn|}{(\dbn)^2} \lesssim \dfrac{|\dnusn|}{(\dbn)^2} = \dfrac{|\dnusn|}{M'_n} \lesssim 1$. 

From these estimation, we achieve 
\begin{align}
\label{eq:lqd:case2.2.1_eq2}
\begin{split}
    &\dfrac{(\lambdan(|\dan|+|\dbn|^2+|\dnun|) + \lambdasn(|\dasn|+|\dbsn|^2+|\dnusn|))(|\dan-\dasn|+|\dbn-\dbsn|^2+|\dnun-\dnusn|)}{\dfourbar} \\
    &\lesssim \dfrac{\lambdan|\dbn|^2(|\dan-\dasn|+|\dbn-\dbsn|^2+|\dnun-\dnusn|)}{\dfourbar} \to 0. 
\end{split}
\end{align}
By adding equation~\eqref{eq:lqd:case2.2.1_eq1} and equation~\eqref{eq:lqd:case2.2.1_eq2}, we would receive that 
\begin{equation*}
    \dfrac{\dfourbar}{\dfourbar} \to 0,
\end{equation*}
which is a contradiction. Thus, we conclude that 
\begin{align}
\label{eq:lqd:case2.2.1_not_to_infty}
    \max{
    \Big\{
    \frac{\lambdan|\Delta\an-\Delta\asn|}{(\lambdasn-\lambdan)|\Delta\bn|^2},~
    \frac{\lambdan|\Delta\bn-\Delta\bsn|^2}{(\lambdasn-\lambdan)|\Delta\bn|^2},~
    \frac{\lambdan|\Delta\nun-\Delta\nusn|}{(\lambdasn-\lambdan)|\Delta\bn|^2}
    \Big\}
    }\not\to\infty. 
\end{align}
Thus, from the formulation of $\dfourbar$, noting that $\dbsn/\dbn \to 1$, we could conclude that 
\begin{align*}
    \dfourbar &\asymp (\lambdasn-\lambdan)(\dbsn)^2(\dbn)^2 + \lambdasn|\dbn|^2(|\dan-\dasn|+|\dbn-\dbsn|^2+|\dnun-\dnusn|)\\
    &\asymp (\lambdasn-\lambdan)(\dbsn)^4 + (\lambdasn-\lambdan)(\dbsn)^4 \quad \text{(by equation~\eqref{eq:lqd:case2.2.1_not_to_infty})}\\
    &\asymp (\lambdasn-\lambdan)(\dbsn)^4.
\end{align*}

Additionally, let us denote 
\begin{equation*}
    \dbn - \dbsn = r_1^n \dbsn, \quad \dnun-\dnusn = r_2^n(\dbsn)^2, \quad -\dnusn = r_3^n(\dbsn)^2. 
\end{equation*}
Then, based on the discussion at the beginning of Case 2.2 that the equation \eqref{eq:lqd_nondistinguishable_dependent_6} holds, we have $r_1^n, r_2^n\to 0$, and $|r_3^n| \not\to \infty$. Moreover, the equation 
\eqref{eq:lqd:case2.2.1_not_to_infty} indicates that 
\begin{equation*}
    \dfrac{\lambdan|r_1^n|}{\lambdasn-\lambdan} \not\to \infty, \dfrac{\lambdan|r_2^n|}{\lambdasn-\lambdan} \not\to \infty.
\end{equation*}

We divide the numerators and the denominators of $F_{0,3}$ in equation~\eqref{eq:lqd_f03} by $(\lambdasn-\lambdan)(\dbsn)^3$ and $F_{0,4}$ in equation~\eqref{eq:lqd_f04} by $(\lambdasn-\lambdan)(\dbsn)^4$, respectively. Then, we see that the new numerators go to 0 as the new denominators do not go to infinity, i.e. 

\begin{align*}
    \frac{1}{6}\Big[-1+\frac{\lambdan}{\lambdasn-\lambdan}\cdot(r_1^n)^3\Big]+\frac{1}{2}\Big[r_3^n+\frac{\lambdan}{\lambdasn-\lambdan}\cdot r_1^nr_2^n\Big]&\to 0,\\
    \frac{1}{24}\Big[1+\frac{\lambdan}{\lambdasn-\lambdan}\cdot(r_2^n)^4\Big]+\frac{1}{4}\Big[r_3^n+\frac{\lambdan}{\lambdasn-\lambdan}\cdot r^n_1r_2^n\Big]+\frac{1}{8}\Big[(r_3^n)^2+\frac{\lambdan}{\lambdasn-\lambdan}\cdot(r_2^n)^2\Big]&\to 0.
\end{align*}
From the discussion above, we would have
\begin{align*}
    \frac{\lambdan}{\lambdasn-\lambdan}\cdot (r_1^n)^3 \to 0,
    \frac{\lambdan}{\lambdasn-\lambdan}\cdot (r_2^n)^4 \to 0,
    \frac{\lambdan}{\lambdasn-\lambdan}\cdot r^n_1 r_2^n \to 0,
    \frac{\lambdan}{\lambdasn-\lambdan}\cdot (r_2^n)^2\to 0.
\end{align*}
Suppose that $r_3^n\to r_3$, the above system will become
\begin{align*}
\begin{cases}
        -\frac{1}{6}+\frac{1}{2}r_3=0,\\
        \frac{1}{24}-\frac{1}{4}r_3+\frac{1}{8}(r_3)^2=0,
\end{cases}
\end{align*}
without a solution, which is a contradiction. Thus, case 2.2.1 does not hold.

\textbf{Case 2.2.2: } $M'_n = |\dnun|$. Using the same argument as in Case 2.2.1, we would have that 
\begin{align}
\label{eq:lqd::case2.2.2_not_to_infty}
    \max{
    \Big\{
    \frac{\lambdan|\Delta\an-\Delta\asn|}{(\lambdasn-\lambdan)|\Delta\nun|},~
    \frac{\lambdan|\Delta\bn-\Delta\bsn|^2}{(\lambdasn-\lambdan)|\Delta\nun|},~
    \frac{\lambdan|\Delta\nun-\Delta\nusn|}{(\lambdasn-\lambdan)|\Delta\nun|}
    \Big\}
    }\not\to\infty,
\end{align}
which is an important ingredient for the estimation that
\begin{equation*}
    \dfourbar \asymp (\lambdasn-\lambdan)(\dnusn)^2.
\end{equation*}
Next, let us denote
\begin{align*}
    \Delta\nun-\Delta\nusn =s_1^n|\Delta\nusn|,\quad \Delta\bn-\Delta\bsn =s_2^n|\Delta\nusn|^{1/2},\quad
    \Delta\bsn =s_3^n|\Delta\nusn|^{1/2}.
\end{align*}
From equation \eqref{eq:lqd_nondistinguishable_dependent_6}, we could conclude that 
$s_1^n\to 0, s_2^n\to 0$, and $|s_3^n|\not\to\infty$.
From the equation \eqref{eq:lqd::case2.2.2_not_to_infty}, we would have
\begin{align*}
    \frac{\lambdan|s_1^n|}{\lambdasn-\lambdan}\not\to\infty, 
    \frac{\lambdan|s_2^n|^2}{\lambdasn-\lambdan}\not\to\infty,
\end{align*}

Now we divide both the numerators and the denominators of $F_{0,3}$ and $F_{0,4}$ by 
$(\lambdasn-\lambdan)|\Delta\nusn|^{3/2}$ and 
$(\lambdasn-\lambdan)|\Delta\nusn|^{2}$, respectively. 
Then, we see that the new numerators go to 0 while the new denominators do not go to infinity, i.e.

We divide the numerators and the denominators of $F_{0,3}$ in equation~\eqref{eq:lqd_f03} by $(\lambdasn-\lambdan)|\Delta\nusn|^{3/2}$ and in equation~\eqref{eq:lqd_f04} by $(\lambdasn-\lambdan)|\Delta\nusn|^{2}$, respectively. Then, we see that the new numerators go to 0 as the new denominators do not go to infinity, i.e. 

\begin{align*}
    \frac{1}{6}\Big[-(s^n_3)^3+\frac{\lambdan}{\lambdasn-\lambdan}\cdot(s_2^n)^3\Big]+\frac{1}{2}\Big[s_3^n+\frac{\lambdan}{\lambdasn-\lambdan}\cdot s_1^ns_2^n\Big]&\to 0,\\
    \frac{1}{24}\Big[(s^n_3)^4+\frac{\lambdan}{\lambdasn-\lambdan}\cdot(s_2^n)^4\Big]+\frac{1}{4}\Big[-(s_3^n)^2+\frac{\lambdan}{\lambdasn-\lambdan}\cdot s^n_1(s_2^n)^2\Big]+\frac{1}{8}\Big[1+\frac{\lambdan}{\lambdasn-\lambdan}\cdot(s_1^n)^2\Big]&\to 0.
\end{align*}
Based on the assumptions about $s^n_i, i\in[3]$, we could get that
\begin{align*}
    \frac{\lambdan}{\lambdasn-\lambdan}\cdot s_1^ns_2^n \to 0,
    \frac{\lambdan}{\lambdasn-\lambdan}\cdot (s_2^n)^4\to 0,
    \frac{\lambdan}{\lambdasn-\lambdan}\cdot s_1^n(s_2^n)^2 \to 0, \frac{\lambdan}{\lambdasn-\lambdan}\cdot (s_1^n)^2 \to 0.
\end{align*}
Suppose that $s_3^n\to s_3$, the above system will become
\begin{align*}
\begin{cases}
        -\frac{1}{6}s_3^3+\frac{1}{2}s_3=0,\\
        \frac{1}{24}s_3^4-\frac{1}{4}s_3^2+\frac{1}{8}=0,
\end{cases}
\end{align*}
without a solution, which is a contradiction. Thus, case 2.2.2 does not hold.

\textbf{Case 2.2.3: } $M'_n = |\dan|$. Using the same argument as in Case 2.2.1, we would have that 
\begin{align}
\label{eq:lqd::case2.2.2_not_to_infty_hello}
    \max{
    \Big\{
    \frac{\lambdan|\Delta\an-\Delta\asn|}{(\lambdasn-\lambdan)|\Delta\an|},~
    \frac{\lambdan|\Delta\bn-\Delta\bsn|^2}{(\lambdasn-\lambdan)|\Delta\an|},~
    \frac{\lambdan|\Delta\nun-\Delta\nusn|}{(\lambdasn-\lambdan)|\Delta\an|}
    \Big\}
    }\not\to\infty,
\end{align}
which is an important ingredient for the estimation that
\begin{equation*}
    \dfourbar \asymp (\lambdasn-\lambdan)(\dasn)^2.
\end{equation*} 
Meanwhile, the term $F_{2,0}$ gives us that 
\begin{align*}
    F_{2,0} &= \dfrac{(\lambdasn-\lambdan)(\dasn)^2+\lambdan(\dasn-\dan)^2}{\dfourbar} \asymp \dfrac{(\lambdasn-\lambdan)(\dasn)^2+\lambdan(\dasn-\dan)^2}{(\lambdasn-\lambdan)(\dasn)^2}\\
    &= 1 + \dfrac{\lambdan(\dasn-\dan)^2}{(\lambdasn-\lambdan)(\dasn)^2} \to 0
\end{align*}
which is a contradiction, as $0 \lesssim \dfrac{\lambdan(\dasn-\dan)^2}{(\lambdasn-\lambdan)(\dasn)^2}$. In overall, the case 2.2 cannot hold. 

\textbf{Case 2.3: } $\lambdasn/\lambdan \to \infty$. Recall at the beginning of case 2 that $$\mnp=\max\{|\Delta\an|,|\Delta\bn|^2,|\Delta\nun|,|\Delta\asn|,|\Delta\bsn|^2,|\Delta\nusn|\}.$$ 

Obviously we have $\dfourbar\lesssim\lambdasn(\mnp)^2$. Dividing both the numerator and the denominator of $A_n$ (defined in equation~\eqref{eq:lqd_an}) and $B_n$ (defined in equation~\eqref{eq:lqd_bn}) by $\lambdasn(\mnp)^{1/2}$ and $\lambdasn\mnp$, respectively, and considering that the new denominators tend to 0, we can conclude that the new numerators also tend to 0. In other words, we obtain the following results:
\begin{align*}
\frac{
\lambdan\Delta\bn-\lambdasn\Delta\bsn
}{
\lambdasn(\mnp)^{1/2}
}
=
\frac{
\lambdan\Delta\bn
}{
\lambdasn(\mnp)^{1/2}
}
-\frac{
\Delta\bsn
}{
(\mnp)^{1/2}
}
\to 0 
\end{align*}
\begin{align*}
\frac{
\lambdan\Delta\an-\lambdasn\Delta\asn
}{\lambdasn\mnp}
=
\frac{
\lambdan\Delta\an
}{\lambdasn\mnp}
-
\frac{
\Delta\asn
}{\mnp}
\to 0 
\end{align*}
Now we have $\lambdan/\lambdasn\to 0$, so we can conclude that $\dbsn/(\mnp)^{1/2}\to 0$ and $\dasn/\mnp\to 0$ .

Dividing both the numerator and the denominator of $C_n$ by $\lambdasn\mnp$, similarly we obtain the following results:
\begin{align*}
\frac{(\lambdasn-\lambdan){\Delta\bn\Delta\bsn}+\lambdan\Delta\nun-\lambdasn{\Delta\nusn}}{\lambdasn\mnp} = \dfrac{(\lambdasn-\lambdan)}{\lambdasn}\dfrac{\dbn}{(M'_n)^{1/2}}\dfrac{\dbsn}{(M'_n)^{1/2}} + \dfrac{\lambdan}{\lambdasn}\dfrac{\dnun}{M'_n} - \dfrac{\dnusn}{M'_n}\to 0.
\end{align*}
Reminding that $\frac{\lambdasn}{\lambdan} \to \infty$ as a hypothesis of this case, $|\dbn| \leq (M'_n)^{1/2}$, $|\dnun| \leq M'_n$, and more importantly, $|\dbsn|/(M'_n)^{1/2} \to 0$, we achieve $\dnusn/\mnp \to 0$. To sum up the estimation above: 
\begin{equation}
\label{eq:lqd_case2.3_star_not_be_maximal}
    \dasn/\mnp \to 0, \dbsn/(\mnp)^{1/2} \to 0, \dnusn/\mnp \to 0.
\end{equation}
Thus, WLOG, we can suppose that $\mnp \in \{|\dan|,|\dbn|^2,|\dnun|\}$. Moreover, by considering three situations $\mnp = |\dan|$, $\mnp = |\dbn|^2$, $\mnp = |\dnun|$ separately, combining with the estimation in equation~\eqref{eq:lqd_case2.3_star_not_be_maximal}, we can easily achieve that 
\begin{equation}
    \label{eq:lqd_case2.3_differential_based_on_mnp}
    |\Delta\an-\Delta\asn|+|\Delta\bn-\Delta\bsn|^2+|\Delta\nun-\Delta\nusn| \asymp M_n'
\end{equation}
Combining this estimation, the obvious fact that $\lambdan(|\Delta\an| + |\Delta\bn|^2 + |\Delta\nun|) \asymp \lambdan M_n'$, and the assumption about $R_n$ in equation \eqref{eq:lqd_rn_infty}, we have
\begin{equation}
    \label{eq:lqd:case2.3_star_term_is_still_strong}
    \lambdasn\max\{|\dasn|,|\dbsn|^2,|\dnusn|\}/(\lambdan\mnp)\to \infty.
\end{equation}
Let $N_n' = \max\{ |\dasn|, |\dbsn|^2, |\dnusn|\}$. It is evident that $M_n'\to 0, N_n'\to 0$. Based on this assumption, by estimating each term in $\dfourbar$ separately, combining with equation~\eqref{eq:lqd_case2.3_differential_based_on_mnp} and equation~\eqref{eq:lqd:case2.3_star_term_is_still_strong}, we have 
\begin{align*}
    \dfourbar &\asymp (\lambdasn-\lambdan)M_n'N_n' + \lambdasn M_n'N_n'\\
   &\asymp \lambdasn M_n'N_n'. 
\end{align*}
Let discuss three distinct cases based on $N_n'$. 

\textbf{Case 2.3.1: } $N_n'= |\dbsn|^2$ for all $n$, then $\dfourbar \asymp \lambdasn M_n'|\dbsn|^2$. By dividing both the numerator and denominator of $B_n$ (in equation~\eqref{eq:lqd_bn}) by $\lambdasn\dbsn$ , since the new denominator of $B_n$ goes to 0, its new numerator must go to 0, i.e.,
\begin{equation*}
    \dfrac{\lambdasn\dbsn - \lambdan\dbn}{\lambdasn\dbsn} \to 0 \Leftrightarrow (\lambdasn\dbsn)/(\lambdan\dbn) \to 1 \Leftrightarrow (\lambdasn|\dbsn|^2)/(\lambdan\dbn\dbsn) \to 1. 
\end{equation*}
Moreover, $\lambdan|\dbn\dbsn|/\lambdan M'_n \lesssim 1$, thus 
\begin{equation*}
    \lambdasn|\dbsn|^2/\lambdan M'_n \lesssim 1,
\end{equation*}
which is a contradiction to equation~\eqref{eq:lqd:case2.3_star_term_is_still_strong}. 
Noting that from equation~\eqref{eq:lqd_case2.3_star_not_be_maximal}, $(\lambdasn|\dbsn|^2)/(\lambdan|\dbn|^2)\to 0$. This is a contradiction with \eqref{eq:lqd:case2.3_star_term_is_still_strong}.

\textbf{Case 2.3.2: } $N_n'= |\dasn|$ for all $n$, by dividing both the numerator and denominator of $A_n$ by $\lambdasn\dasn$ , since the new denominator of $A_1$ goes to 0, its new numerator must go to 0, i.e.,
\begin{equation*}
    \dfrac{\lambdasn\dasn - \lambdan\dan}{\lambdasn\dasn} \to 0 \Leftrightarrow (\lambdasn\dasn)/(\lambdan\dan) \to 1. 
\end{equation*}
Thus, $(\lambdasn\dasn)/(\lambdan M'_n) \lesssim 1$, which is also a contradiction with equation~\eqref{eq:lqd:case2.3_star_term_is_still_strong}.

\textbf{Case 2.3.3: } $N_n' = |\dnusn|$. Firstly, we mention a small remark which would be helpful during our proof: 

\textbf{Claim:} Given three sequences of real numbers $(\mathfrak{a}_n)_n, (\mathfrak{b}_n)_n, (\mathfrak{c}_n)_n$ satisfying 
$(\mathfrak{a}_n - \mathfrak{b}_n)/\mathfrak{c}_n \to 0$. If $\mathfrak{a}_n/\mathfrak{b}_n \not \to 1$ (i.e. there does not exist a subsequence $\mathfrak{a}_{n_i},\mathfrak{b}_{n_i}$ such that $\mathfrak{a}_{n_i}/\mathfrak{b}_{n_i} \to 1$), then 
\begin{equation*}
    \dfrac{\mathfrak{a}_n}{\mathfrak{c}_n}\to 0, \dfrac{\mathfrak{b}_n}{\mathfrak{c}_n}\to 0.
\end{equation*}
In fact, by writing $(\mathfrak{a}_n - \mathfrak{b}_n)/\mathfrak{c}_n$ as 
\begin{equation*}
    (\mathfrak{a}_n - \mathfrak{b}_n)/\mathfrak{c}_n = \dfrac{\mathfrak{b}_n}{\mathfrak{c}_n}\left(\dfrac{\mathfrak{a}_n}{\mathfrak{b}_n}-1\right),
\end{equation*}
then $\mathfrak{a}_n/\mathfrak{b}_n-1\not\to 0$, thus $\mathfrak{b}_n/\mathfrak{c}_n \to 0$, thus $\mathfrak{a}_n/\mathfrak{c}_n = \mathfrak{b}_n/\mathfrak{c}_n + (\mathfrak{a}_n - \mathfrak{b}_n)/\mathfrak{c}_n \to 0$

Return to our problem. Denote the term that are useful in the next step of our proof: 
\begin{align*}
    TB_n &= \dfrac{(\lambdasn-\lambdan)\dbsn}{\lambdan(\dbn-\dbsn)},\\
    TBN_n &= \dfrac{(\lambdasn-\lambdan)(\dbsn)^2+\lambdan(\dbsn-\dbn)^2}{(\lambdasn-\lambdan)(\dnusn)+\lambdan(\dnusn-\dnun)},\\
    G_n &= \dfrac{\lambdasn\dnusn-\lambdan\dnun}{\dfourbar}.
\end{align*}
We have 
\begin{align}
\label{eq:lqd:case2_3_3_equivalent_of_TBn}
    \dfrac{(\lambdasn-\lambdan)\dbsn}{\lambdan(\dbsn-\dbn)} \to 1 &\Leftrightarrow \dfrac{\lambdasn\dbsn}{\lambdan(\dbn-\dbsn)} \to 1 \Leftrightarrow \dfrac{\lambdan(\dbn-\dbsn)}{\lambdasn\dbsn} \to 1\\
    \notag
    &\Leftrightarrow \dfrac{\lambdan\dbn}{\lambdasn\dbsn} - \dfrac{\lambdan}{\lambdasn} \to 1 \Leftrightarrow \dfrac{\lambdan\dbn}{\lambdasn\dbsn} \to 1. 
\end{align}
\textbf{Case 2.3.3.1: } $TB_n \not\to 1$. Applying the claim to $\mathfrak{a}_n = (\lambdasn-\lambdan)\dbsn$, $\mathfrak{b}_n = \lambdan(\dbn-\dbsn)$, $\mathfrak{c}_n=\dfourbar$ in the formulae of $B_n$ (in equation~\eqref{eq:lqd_bn}), then we achieve that 
\begin{equation*}
    \dfrac{(\lambdasn-\lambdan)\dbsn}{\dfourbar} \to 0, \dfrac{\lambdan(\dbn-\dbsn)}{\dfourbar} \to 0.
\end{equation*}
Thus, from the hypothesis that $\dbsn\to0$, $\dbn \to 0$
\begin{equation*}
    \dfrac{(\lambdasn-\lambdan)(\dbsn)^2}{\dfourbar} \to 0, \dfrac{\lambdan(\dbn-\dbsn)^2}{\dfourbar} \to 0.
\end{equation*}
By substituting these estimation to the formulae of $C_n$ (in equation~\eqref{eq:lqd_cn}), we achieve that
\begin{equation*}
     G_n = \dfrac{\lambdasn\dnusn-\lambdan\dnun}{\dfourbar} \to 0. 
\end{equation*}
Using the same argument as in Case 2.3.2, we get the contradiction. 

\textbf{Case 2.3.3.2: } $TBN_n \not\to 1$. Then, applying the claim to $\mathfrak{a}_n = (\lambdasn-\lambdan)(\dbsn)^2+\lambdan(\dbsn-\dbn)^2$, $\mathfrak{b}_n = (\lambdasn-\lambdan)(\dnusn)+\lambdan(\dnusn-\dnun)$, $\mathfrak{c}_n=\dfourbar$ in the formulae of $C_n$ (in equation~\eqref{eq:lqd_cn}), then we achieve that 
\begin{equation*}
    G_n = \dfrac{\lambdasn\dnusn-\lambdan\dnun}{\dfourbar} \to 0. 
\end{equation*}
Using the same argument as in Case 2.3.2, we get the contradiction.

\textbf{Case 2.3.3.3: } Both $TB_n,TBN_n \to 1$. Then, as discussed in equation~\eqref{eq:lqd:case2_3_3_equivalent_of_TBn}, $TB_n \to 1$ implies that $$(\lambdasn\dbsn)/(\lambdan\dbn)\to 1.$$ 
As $\lambdasn/\lambdan \to \infty$, the previous equation implies that 
\begin{equation}
    \label{eq:lqd:case2.3.3.3:incomparability_dbn_dbsn}
    \dbn/\dbsn \to \infty \text{ or } \dbsn/\dbn \to 0.
\end{equation}

Next, we would like to find the equivalent form of the nominator and denominator of $TBN_n$. For the nominator, 
we write 
\begin{align*}
    \dfrac{(\lambdasn-\lambdan)(\dbsn)^2+\lambdan(\dbsn-\dbn)^2}{\lambdan(\dbn)^2} = \dfrac{\lambdasn-\lambdan}{\lambdasn}\left(\dfrac{\dbsn}{\dbn}\right)^2 + \left(1-\dfrac{\dbsn}{\dbn}\right)^2.
\end{align*}
From the hypothesis that $\lambdasn/\lambdan \to \infty$ as well as equation~\eqref{eq:lqd:case2.3.3.3:incomparability_dbn_dbsn}, we have 
\begin{equation}
    \label{eq:lqd:case2.3.3.3:equivalent_form_nominator}
    \dfrac{(\lambdasn-\lambdan)(\dbsn)^2+\lambdan(\dbsn-\dbn)^2}{\lambdan(\dbn)^2} \to 1.
\end{equation}
For the denominator of $TBN_n$, according to equation \eqref{eq:lqd:case2.3_star_term_is_still_strong}, noting that $|\dnusn| = \max\{|\dasn|,|\dbsn|^2,|\dnusn|\}$ (hypothesis of Case 2.3.3), $|\dnun| \leq M_n'$, then $(\lambdan\dnun)/(\lambdasn\dnusn) \to 0$. Thus, 
\begin{equation}
    \label{eq:lqd:case2.3.3.3:equivalent_form_denominator}
    \dfrac{\lambdasn\dnusn-\lambdan\dnun}{\lambdasn\dnusn} = 1 - \dfrac{\lambdan\dnun}{\lambdasn\dnusn} \to 1.
\end{equation}
By combining equation~\eqref{eq:lqd:case2.3.3.3:equivalent_form_nominator} and equation~\eqref{eq:lqd:case2.3.3.3:equivalent_form_denominator}, and the formulae of $TBN_n$, we achieve that
\begin{equation*}
    \dfrac{\lambdasn\dnusn}{\lambdan(\dbn)^2} \to 1 \text{ or } \dfrac{\lambdasn|\dnusn|}{\lambdan|\dbn|^2} \to 1. 
\end{equation*}
As $|\dnusn| = \max\{|\dasn|,|\dbsn|^2,|\dnusn|$, $|\dbn|^2 \leq M_n'$, this equation implies that $\lambdasn\max\{|\dasn|,|\dbsn|^2,|\dnusn|\}/(\lambdan M_n') \lesssim 1$, which is a contradiction to equation~\eqref{eq:lqd:case2.3_star_term_is_still_strong}. 

We conclude the proof of our theorem. 
\end{proof}

\section{PROOFS FOR THE MINIMAX LOWER BOUNDS}
\label{appendix:lower-bounds-proof}

In this appendix, we present the proofs for 
the minimax lower bounds introduced in Section \ref{section:linear-sigma} and Section \ref{section:non-linear-sigma}. 

\subsection{Proof of Theorem \ref{thm:lower-distinguish}}
\label{proof:lower-distinguish}

In order to prove the result, we will at first define two distances:
\begin{align*}
    &d_1((\lambdaone,G_1),(\lambdatwo,G_2))
    =\lambda_1\Vert (a_1,b_1,\nu_1)-(a_2,b_2,\nu_2) \Vert,
    \\&
    d_2((\lambdaone,G_1),(\lambdatwo,G_2))
    = | \lambdaone -\lambdatwo |^2,
\end{align*}
for any $(\lambda_1,G_1)=(\lambda_1,a_1,b_1,\nu_1)\in\Xi$ and $(\lambda_2,G_2)=(\lambda_2,a_2,b_2,\nu_2)\in\Xi$.
Obviously  $d_2((\lambda_1, G_1), (\lambda_2, G_2))$ is a proper distance. 
The structure for $d_1((\lambda_1, G_1), (\lambda_2, G_2))$ tells us that it is not symmetric.
Only when $\lambdaone=\lambdatwo=\lambda$, $d_1((\lambda, G_1), (\lambda, G_2))$ is symmetric.
Also $d_1((\lambda_1, G_1), (\lambda_2, G_2))$ still satisfies a weak triangle inequality:
\begin{align*}
    d_1((\lambda_1, G_1), (\lambda_2, G_2))+
    d_1((\lambda_2, G_2), (\lambda_3, G_3))
    \geq
    \min
    \{
    d_1((\lambda_1, G_1), (\lambda_2, G_2)),
    d_1((\lambda_2, G_2), (\lambda_3, G_3))
    \}.
\end{align*}
Therefore, we will apply the modified Le Cam method for nonsymmetric loss, as outlined in Lemma C.1 of \citet{gadat2020parameter}, to handle this distance. 
For $f$ satisfies all assumptions in Theorem \ref{thm:lower-distinguish}, based on the Taylor expansion, we have the following results:
\begin{lemma}
\label{prop:lower-distinguish}
    Given $f$ in Theorem \ref{thm:lower-distinguish}, we achieve for any $r < 1$ that
\begin{align*}
        &\text{(i)}~~
            \lim_{\epsilon \rightarrow 0} 
        \inf_{
        (\lambda,G_1)=\left(\lambda, a_1,b_1,\nu_1\right), 
        (\lambda,G_2)=\left(\lambda, a_2,b_2,\nu_2\right)
        }
        \left\{\displaystyle
        \frac{h\left(p_{\lambda,G_1}, p_{\lambda,G_2}\right) }{d_1^r\left((\lambda, G_1), (\lambda, G_2)\right)}
        : d_1\left((\lambda, G_1), (\lambda, G_2)\right) \leq \epsilon\right
        \}=0,\\
        &\text{(ii)}~~
            \lim_{\epsilon \rightarrow 0} 
        \inf_{
        (\lambdaone,G)=\left(\lambdaone, a,b ,\nu \right), 
        (\lambdatwo,G)=\left(\lambdatwo, a ,b ,\nu \right)
        }
        \left\{\displaystyle
        \frac{h\left(p_{\lambdaone,G}, p_{\lambdatwo,G}\right) }{d_2^r\left((\lambdaone, G), (\lambdatwo, G)\right)}
        : d_2\left((\lambdaone, G), (\lambdatwo, G)\right) \leq \epsilon\right
        \}=0.
        \end{align*}
\end{lemma}
We will prove this lemma later.

\begin{proof}[Proof of Theorem \ref{thm:lower-distinguish}]
Denote $(\lambdas,\Gs)=(\lambdas,\as,\bs,\nus)$ and assume $r<1$. 
Given Lemma \ref{prop:lower-distinguish} part (i) , for any sufficiently small $\epsilon > 0$, there exists 
$ (\lambdas,G'_{*}) = (\lambda^{*}, a^{*}_1, \bs_1,\nu^{*}_1) $ 
such that 
$d_1((\lambdas,G_{*}), (\lambdas,G'_{*})) = d_1((\lambdas,G'_{*}), (\lambdas,G_{*})) = \epsilon $
,
there exists a constant $C_0$, s.t.
\begin{align}
\label{apppf:lower1-hellinger-distance}
    h(p_{\lambdas,G_{*}}, p_{\lambdas,G'_{*}}) \leq C_0 \epsilon^r.
\end{align} 

Now we will denote $p^n_{\lambdas,G_{*}}$ as the density of the $n$-i.i.d. sample $(X_1,Y_1),\cdots,(X_n,Y_n)$.
Lemma C.1 in \citet{gadat2020parameter} tells us that
\begin{align*}
\label{proof:lower-distinguish-eq1}
    \inf_{(\llbgn) \in \Xi }\sup_{(\lbg)\in \Xi }\mathbb{E}_{\plbg} \Big( \lambda^2 \Vert (\overline{a}_n, \overline{b}_n, \overline{\nu}_n)-(a,b,\nu) \Vert^2 \Big) 
    \geq \frac{\epsilon^2}{2}\Big(1-V(p^n_{\lambdas,G_{*}}, p^n_{\lambdas,G'_{*}}) \Big)
    \geq \frac{\epsilon^2}{2}
    \sqrt{1-\left( 1-C_0^2\epsilon^{2r} \right)^n}.
\end{align*}
Last inequality is from the definition of the Total Variation distance and Hellinger distance and equation \eqref{apppf:lower1-hellinger-distance}. 
Let $\epsilon^{2r}=\displaystyle\frac{1}{C_0^2n}$, then for any $r<1$ we have
\begin{align*}
    \inf_{(\llbgn)\in \Xi }\sup_{(\lbg)\in \Xi }\mathbb{E}_{\plbg} \Big( \lambda^2 \Vert (\overline{a}_n, \overline{b}_n, \overline{\nu}_n)-(a,b,\nu) \Vert^2 \Big) 
    \geq c_1 n^{-1/r},
\end{align*}
where $c_1$ is some positive constant.
Following a similar reasoning and using Lemma \ref{prop:lower-distinguish} part (ii) , we will obtain
\begin{align*}
    \inf_{(\llbgn)\in \Xi }\sup_{(\lbg)\in \Xi }\mathbb{E}_{\plbg} \Big( |\overline{\lambda}_n
    -\lambda|^2 \Big) \geq c_2 n^{-1/r},
\end{align*}
for some positive constant $c_2$. 
Consequently, we establish all of the results for Theorem \ref{thm:lower-distinguish}.
\end{proof}

\begin{proof}[Proof of Lemma \ref{prop:lower-distinguish}]
(i) 
Now consider two sequences 
\begin{align*}
    \begin{cases}
        (\lambdan,G_{1,n})=(\lambdan,a_{1,n},b_{1,n},\nu_{1,n})\\
        (\lambdan,G_{2,n})=(\lambdan,a_{2,n},b_{2,n},\nu_{2,n})
    \end{cases}
\end{align*}
with the same proportion $\lambdan$. Just refer to the contaminated MoE model definition \eqref{eq:deviated-model} we will have
\begin{align*}
h^2\left(p_{\lambdan,G_{1, n}}, p_{\lambdan,G_{2, n}}\right) 
 \leq \lambda_n 
\int \frac{\left[f
\left(\ysigmaonen\right)
-f\left(\ysigmatwon\right)\right]^2}{f\left(\ysigmatwon\right)} d(X,Y)
\end{align*}
due to 
$\sqrt{p_{\lambdan,G_{1, n}}(Y|X)}+\sqrt{p_{\lambdan,G_{2, n}}(Y|X)}>\sqrt{\lambdan f\left(\ysigmatwon\right)}$
and
$p_{\lambdan,G_{1, n}}(Y|X)-p_{\lambdan,G_{2, n}}(Y|X)=\lambdan\left(f
\left(\ysigmaonen\right)
-f\left(\ysigmatwon\right)\right)$
. 
Consider the Taylor expansion up to the first order with remainder in integral form, we have
\begin{align*}
    &f\left(\ysigmaonen\right)-f\left(\ysigmatwon\right)
    \\
    &=\sum_{|\alpha|=1}\frac{(a_{1,n}-a_{2,n})^{\alpha_1}(b_{1,n}-b_{2,n})^{\alpha_2}(\nu_{1,n}-\nu_{2,n})^{\alpha_3}}{\alpha_1!\alpha_2!\alpha_3!}\frac{\partial f^{|\alpha|}}{\partial a^{\alpha_1}\partial b^{\alpha_2} \partial \nu^{\alpha_3}}(\ysigmatwon)
    \\&+\sum_{|\alpha|=1}\frac{(a_{1,n}-a_{2,n})^{\alpha_1}(b_{1,n}-b_{2,n})^{\alpha_2}(\nu_{1,n}-\nu_{2,n})^{\alpha_3}}{\alpha_1!\alpha_2!\alpha_3!}
    \times
    \\&~
    \int_0^1
    \frac{\partial f^{|\alpha|}}{\partial a^{\alpha_1}\partial b^{\alpha_2} \partial \nu^{\alpha_3}}\left(Y|\sigma\left(\left(a_{2,n}+t(a_{1,n}-a_{2,n})\right)^\top X+b_{2,n}+t(b_{1,n}-b_{2,n})\right),\nu_{2,n}+t(\nu_{1,n}-\nu_{2,n})\right)dt.
\end{align*}
So it follows that 
\begin{align*}
    \frac{h^2(p_{\lambdan,G_{1,n}}, p_{\lambdan,G_{2,n}})}{d_{1}^{2r}((\lambdan,G_{1,n}), (\lambdan,G_{2,n}))} \to 0.
\end{align*}
when we have
    $d_1^{2-2r}((\lambdan,G_{1,n}),(\lambdan,G_{2,n}))/\lambdan
    =\lambdan^{1-2r}\Vert (a_{1,n},b_{1,n},\nu_{1,n})-(a_{2,n},b_{2,n},\nu_{2,n}) \Vert^{2-2r}\to 0$
and
    $d_1((\lambdan,G_{1,n}),(\lambdan,G_{2,n}))/\lambdan
    =\Vert (a_{1,n},b_{1,n},\nu_{1,n})-(a_{2,n},b_{2,n},\nu_{2,n}) \Vert \to 0$.
Naturally we got part (i).

(ii) The proof for part (ii) is basically the same as in part (i). 
Similarly, consider two sequences 
\begin{align*}
    \begin{cases}
        (\lambdaonen,G_{n}) = (\lambda_{1,n}, a_n, b_n, \nu_n),\\
        (\lambdatwon,G_{n}) = (\lambda_{2,n}, a_n, b_n, \nu_n),
    \end{cases}
\end{align*}
with different proportion and the same $G_n=(\an,\bn\nun)\in\Theta$.
we will have
\begin{align*}
    \frac{h^2(p_{\lambdaonen,\Gn},p_{\lambdatwon,\Gn})}{d^{2r}_2((\lambdaonen,\Gn),(\lambdatwon,\Gn))}
    &\leq \frac{(\lambdaonen-\lambdatwon)^{2-2r}}{(1-\lambdaonen)\wedge \lambdaonen}\int \frac{(f_0(\ysigmazero)-f(\ysigman))^2}{f_0(\ysigmazero)+f(\ysigman)}d(X,Y)
    \\&
    \leq \frac{2(\lambdaonen-\lambdatwon)^{2-2r}}{(1-\lambdaonen)\wedge\lambdaonen}
\end{align*}
Now consider when $ \displaystyle\frac{(\lambdaonen-\lambdatwon)^{2-2r}}{ (1-\lambdaonen)\wedge\lambdaonen }\to 0$, we will reach the conclusion of part (ii).
\end{proof}

\subsection{Proof of Theorem \ref{appthm:lower-mle-Gaussian-linear}}
\label{appendix:appthm:lower-mle-Gaussian-linear}
\begin{proof}[Proof of Theorem \ref{appthm:lower-mle-Gaussian-linear}]
The proof follows similar steps to the argument in the last section.
To be precise, we define 
\begin{align*}
\begin{cases}
    d_3((\lambda_1,G_1),(\lambda_2,G_2)) &= \lambda_1\Vert((a_1,b_1,\nu_1) - (a_2,b_2,\nu_2))\Vert^4,\\
    d_4((\lambda_1,G_1),(\lambda_2,G_2)) &= |\lambda_1 - \lambda_2 |\Vert((\Delta a_1,\Delta b_1,\Delta \nu_1) - (\Delta a_2,\Delta b_2,\Delta \nu_2))\Vert^4.
\end{cases}
\end{align*}
It is still straighforward that $d_4$ still satisfies the weak triangle inequality, meanwhile $d_3$ no longer satisfies this inequality. However, a close look at $d_3$ shows that 
\begin{equation*}
    d_3(\lambda_1,G_1),(\lambda_3,G_3)) + d_3(\lambda_2,G_2),(\lambda_3,G_3)) \geq \dfrac{\min\{d_3((\lambda_1,G_1),(\lambda_2,G_2)), d_3((\lambda_2,G_2),(\lambda_1,G_1))\}}{8}.
\end{equation*}
Under this remark, we can modify the Le Cam method in Lemma C.1, \citep{gadat2020parameter} to achieve the following estimation
\begin{equation*}
    \inf_{(\overline{\lambda}_n,\overline{G}_n) \in \Xi} \sup_{({\lambda},{G}) \in \Xi_1(l_n)}\mathbb{E}_{p_{\lambda,G}}(d_3^2((\lambda,G),(\llambdan,\overline{G}_n))) \geq \dfrac{\epsilon^2}{128} \left\{1-V(p_{\lambdaone,G_1}^n,p_{\lambdatwo,G_2}^n)\right\} 
\end{equation*}
for $(\lambdaone,G_1),(\lambdatwo,G_2) \in \Xi_1(l_n)$ satisfying $d_3((\lambda_1,G_1),(\lambda_2,G_2))\land d_3((\lambda_2,G_2),(\lambda_1,G_1)) \geq \epsilon/4$. Following the same schema as in Lemma \ref{prop:lower-distinguish}, we can demonstrate two subsequent results for any $r > 1$: 
\begin{itemize}
    \item [(i)] Two sequences can be found  
    \begin{align*}
        \begin{cases}
            (\lambdan,G_{1,n})=(\lambdan,a_{1,n},b_{1,n},\nu_{1,n})\in \Xi_1(l_n),\\
            (\lambdan,G_{2,n})=(\lambdan,a_{2,n},b_{2,n},\nu_{2,n})\in \Xi_1(l_n),
        \end{cases}
    \end{align*}
    such that $d_3((\lambdan,G_{1,n}),(\lambdan,G_{2,n})) \to 0$ and $h(p_{\lambdan,G_{1,n}}, p_{\lambdan,G_{2,n}})/d_3^r((\lambdan,G_{1,n}),(\lambdan,G_{2,n}))\to 0$ as $n \to \infty$. 

    \item [(ii)] Two sequences can also be found  
    \begin{align*}
        \begin{cases}
            (\lambda_{1,n},G_{n})=(\lambda_{1,n},\an,\bn,\nun)\in \Xi_1(l_n),\\
            (\lambda_{2,n},G_{n})=(\lambda_{2,n},a_{n},b_{n},\nu_{n})\in \Xi_1(l_n),
        \end{cases}
    \end{align*}
    such that $d_4(p_{\lambda_{1,n},G_{n}},p_{\lambda_{2,n},G_{n}}) \to 0$ and $h(p_{\lambda_{1,n},G_{ n}}, p_{\lambda_{2,n},G_{ n}})/d_4^r((\lambda_{1,n},G_{n}),(\lambda_{2,n},G_{n}))\to 0$ as $n \to \infty$.
\end{itemize}
The justification for the above results can be omitted, as it follows a similar approach to that of Lemma \ref{prop:lower-distinguish}. This brings us to the conclusion of the theorem. 
\end{proof}

\subsection{Proof of Theorem \ref{appthm:lower-mle-Gaussian-nonlinear}}
\label{appendix:appthm:lower-mle-Gaussian-nonlinear}
\begin{proof}[Proof of Theorem \ref{appthm:lower-mle-Gaussian-nonlinear}]
The proof follows similar steps to the arguments in the previous two sections.
Concretely, define for 
$(\lambda_1,G_1) = (\lambda_1, a_1,b_1,\nu_1)$, $(\lambda_2,G_2) = (\lambda_2,a_2,b_2,\nu_2)$
: 
\begin{align*}
\begin{cases}
    d_5((\lambdaone,G_1),(\lambdatwo,G_2))
    =\lambda_1\Vert \Delta a_1, \Delta b_1, \Delta\nu_1 \Vert \Vert (a_1,b_1,\nu_1)-(a_2,b_2,\nu_2) \Vert,
    \\
    d_6((\lambdaone,G_1),(\lambdatwo,G_2))
    = | \lambdaone -\lambdatwo | \Vert \Delta a_1, \Delta b_1, \Delta\nu_1 \Vert ^2.
\end{cases}
\end{align*}
It is straightforward that both $d_5$ and $d_6$ satisfy the weak triangle inequality. Following the same schema as in Lemma \ref{prop:lower-distinguish}, we can demonstrate two subsequent results for any $r > 1$: 
\begin{itemize}
    \item [(i)] Two sequences can be found  
    \begin{align*}
    \begin{cases}
        (\lambdan,G_{1,n})=(\lambdan,a_{1,n},b_{1,n},\nu_{1,n})\in \Xi_{2}(l_n),\\
        (\lambdan,G_{2,n})=(\lambdan,a_{2,n},b_{2,n},\nu_{2,n})\in \Xi_{2}(l_n),
    \end{cases}
    \end{align*}
    such that $d_5((\lambdan,G_{1,n}),(\lambdan,G_{2,n})) \to 0$ and $h(p_{\lambdan,G_{1,n}}, p_{\lambdan,G_{2,n}})/d_5^r((\lambdan,G_{1,n}),(\lambdan,G_{2,n})\to 0$ as $n \to \infty$. 

    \item [(ii)] Two sequences can also be found  
    \begin{align*}
        \begin{cases}
            (\lambda_{1,n},G_{n})=(\lambda_{1,n},\an,\bn,\nun)\in \Xi_{2}(l_n),\\
            (\lambda_{2,n},G_{n})=(\lambda_{2,n},a_{n},b_{n},\nu_{n})\in \Xi_{2}(l_n),
        \end{cases}
    \end{align*}
    such that $d_6(p_{\lambda_{1,n},G_{n}},p_{\lambda_{2,n},G_{2,n}}) \to 0$ and $h(p_{\lambda_{1,n},G_{n}}, p_{\lambda_{2,n},G_{2,n}})/d_6^r((\lambda_{1,n},G_{n}),(\lambda_{2,n},G_{2,n}))\to 0$ as $n \to \infty$.
\end{itemize}
We can omit the justification for the above results as it can follow a similar approach as in Lemma \ref{prop:lower-distinguish}. This leads to the conclusion of the theorem.   
\end{proof}

\section{PROOF FOR DENSITY ESTIMATION RATE}
\label{appendix:ConvergenceRateofDensityEstimation}

\subsection{General theory for the Proof of Theorem \ref{theorem:ConvergenceRateofDensityEstimation}}
At first we restate Theorem \ref{theorem:ConvergenceRateofDensityEstimation}:
\begin{theorem}
\label{appendixtheorem:ConvergenceRateofDensityEstimation}
Assume that the function $f_0$ is bounded with tail 
$\mathbb{E}_X
\left(
-\log f_0(Y|\varphi (a_0^{\top}X+b_0 ),\nu_0)
\right)
\gtrsim
Y^q
$
for almost surely $Y\in\mathcal{Y}$
for some $q>0$.
and $f$ is the density function of an univariate Gaussian distribution.
Then, there exists a constant $C > 0$ depending only on $\Xi$ such that for all $n \geq 1$,
\begin{align*}
    \sup_{(\lambdas,\Gs)\in\Xi}
    \mathbb{E}_{p_{\lambdas,\Gs}}
    h(p_{\widehat{\lambda}_n,\widehat{G}_n},p_{\lambdas,\Gs})
    \leq
    C\sqrt{\log n/n}.
\end{align*}
\end{theorem}
First we will introduce some necessary notations used throughout this appendix.
$(\mathcal{P},d)$ is a metric space, where $d$ is the metric in $\mathcal{P}$.
An $\epsilon$-net for $(\mathcal{P},d)$ is a collection of balls with radius $\epsilon$, and the union of the balls contains $\mathcal{P}$.
$N(\epsilon,\mathcal{P},d)$ is the covering number for the $\epsilon$-net (minimal cardinality).
$H(\epsilon,\mathcal{P},d)=\log N(\epsilon,\mathcal{P},d)$ is the entropy number.
$N_B(\epsilon,\mathcal{P},d)$ is the bracketing number which means the minimal
number $n$ such that there exists $n$ pairs $({\underline{f_i},\overline{f_i}})_{i=1}^n$ such that ${\underline{f_i}<\overline{f_i}}, d({\underline{f_i},\overline{f_i}})<\epsilon$, and $\mathcal{P}$ was covered by their unions.
$H_B(\epsilon,\mathcal{P},d)=\log N_B(\epsilon,\mathcal{P},d)$ is called the bracketing entropy number.
If $\mathcal{P}$ is a density family, then $d$ is the distance associated with $L^2(m)$, where $m$ m is the Lebesgue measure.

In particular, we denote $\mathcal{P}(\Xi)=\{ p_{\lbg}:(\lbg)\in\Xi \}$,
and  $\barplbg=\displaystyle\frac{\plbgs+\plbg}{2}$
then we define
\begin{align*}
\overline{\mathcal{P}}(\Xi)&=\{ \barplbg
: (\lbg),\in\Xi\},  \\
\overline{\mathcal{P}}^{1/2}(\Xi)&=\{\barplbg^{1/2}: \barplbg \in \linepxi \}. 
\end{align*}
The convergence rate can be inferred from the complexity of the set:
\begin{align*}
\overline{\mathcal{P}}^{1/2}(\Xi,\epsilon):=\{\barplbg^{1/2}\in\linephalfxi: h(\barplbg,\plbgs) \leq \epsilon \} .
\end{align*}
We assess the complexity of this class using the bracketing entropy integral in \cite{Vandegeer-2000} :
\begin{align*}
    \mathcal{J}_B(\epsilon,\linephalf(\Xi,\epsilon),m)=
    \int^{\epsilon}_{\epsilon^2/2^{13}}
    {H_B^{1/2}(u,\linephalf(\Xi,\epsilon),m)}
    du\vee\epsilon
\end{align*}
where $u\vee\epsilon=\max\{u,\epsilon \}$
and
$H_B(\epsilon,\mathcal{P},m)$ represents the $\epsilon$-bracketing entropy of a set $\mathcal{P}$ under the Lebesgue measure $m$, we may omit $m$ for the simplicity.

\begin{lemma}
\label{applemma:convergence-rate-1}
    Assume the following assumption hold:
    Given a universal constant $J > 0$, there exists $N > 0$, possibly depending on  $\Xi$, such that for all $n \geq N$ and all $\epsilon > (\log(n)/n)^{1/2}$, we have
\begin{align}
\label{assumption:A2}
   \mathcal{J}_B(\epsilon, \overline{P}^{1/2}(\Xi, \epsilon)) \leq J \sqrt{n} \epsilon^2. 
\end{align}    
    Then, there exists a constant $C > 0$ depending only on $\Xi$ 
    such that for all $n\geq1$,
\begin{align*}
    \sup_{(\lambdas,\Gs)\in\Xi}
    \mathbb{E}_{p_{\lambdas,\Gs}}
    h(p_{\widehat{\lambda}_n,\widehat{G}_n},p_{\lambdas,\Gs})
    \leq
    C\sqrt{\log n/n}.
\end{align*}
\end{lemma}
This lemma told us we just need to prove the model distribution satisfy the above assumption in equation (\ref{assumption:A2}) to get the conclusion we want. 
But usually the assumption in  equation (\ref{assumption:A2}) is too complex to prove, instead we could prove the following lemma: 
\begin{lemma}
\label{applemma:convergence-rate-add-1}
    If the distribution satisfies
    \begin{align}
    \label{assumption:A3}
        H_B(\epsilon,\mathcal{P}(\Xi),h)\lesssim\log (1/\epsilon),
    \end{align}
    it will meet the assumption in  equation  (\ref{assumption:A2}).
\end{lemma}
Though we have simplified the assumption in  equation  (\ref{assumption:A2}) to equation (\ref{assumption:A3}). It is still hard to prove. But recall the contaminated model
\begin{align*}
    p_{\lambda, G}(Y|X)
    =(1-\lambda)f_0(Y|\varphi (a_0^{\top}X+b_0 ),\nu_0)
    +\lambda f(Y|\sigma \Big((a)^{\top}X+b \Big),\nu),
\end{align*}
and we have the assumption that $f_0$ is bounded and light tail, and $f$ is univariate Gaussian density. We have the following lemma that could help us get the conclusion:
\begin{lemma}
\label{applemma:convergence-rate-2}
    Let $\Theta$ be a bounded subsets of $\mathbb{R}^d\times\mathbb{R}\times\mathbb{R}^{+}$, $f$ is a univariate Gaussian density and $f_0$ is bounded with tail 
$\mathbb{E}_X\left(-\log f_0(Y|\varphi (a_0^{\top}X+b_0 ),\nu_0)\right)\gtrsim Y^q$
for almost surely $Y\in\mathcal{Y}$
for some $q>0$.
    Then, for any $ 0 < \varepsilon < \frac{1}{2} $, the following results hold:
\begin{enumerate}[(i)]
    \item $\log N(\epsilon,\mathcal{P}(\Xi),\Vert\cdot\Vert_\infty)\lesssim\log (1/\epsilon)$,
    \item $H_B(\epsilon,\mathcal{P}(\Xi),h)\lesssim\log (1/\epsilon)$.
\end{enumerate}
\end{lemma}
Now we could get the conclusion for Theorem \ref{theorem:ConvergenceRateofDensityEstimation}.

\subsection{Proof of Lemmas for Theorem \ref{theorem:ConvergenceRateofDensityEstimation}}
Now we will prove Lemma \ref{applemma:convergence-rate-1}, Lemma \ref{applemma:convergence-rate-add-1} and Lemma \ref{applemma:convergence-rate-2} in order. 
At first we need to introduced another Lemma \ref{applemma:convergence-rate-3} before we prove Lemma \ref{applemma:convergence-rate-1}. 
Lemma \ref{applemma:convergence-rate-3} is Theorem 5.11 in \cite{Vandegeer-2000} and its proof can also be found in \cite{Vandegeer-2000}.
\begin{lemma}
\label{applemma:convergence-rate-3}
    Let $R > 0$, $k \geq 1$ and  
 $\mathcal{G}$ is a subset in $\Xi$ where $\Gs\in\mathcal{G}\subset\Xi$ .
 Given $C_1<\infty$, for all $C$ sufficiently large, and for $n\in\mathbb{N}$ and $t>0$ is in the following range
 \begin{align}
     t\leq(8\sqrt{n}R)\wedge(C_1\sqrt{n}R^2/K),
 \end{align}
\begin{align}
     t\geq C^2(C_1+1)\Bigg( R\vee\int^{R}_{t/(2^6\sqrt{n})}H_B^{1/2}\big(\frac{u}{\sqrt{2}},\linephalf(\Xi,R),\mu \big) du\Bigg),
\end{align}
then we will have
\begin{align}
    \mathbb{P}_{\lbgs}
    \Big(
    \sup_{G\in\mathcal{G},h(\barplbg,\plbgs)\leq R}
    |\mu_n(G)|\geq t
    \Big)
    \leq
    C\exp
    \left(
    -\frac{t^2}{C^2(C_1+1)R^2}
    \right).
\end{align}
\end{lemma}

\begin{proof}[Proof of Lemma \ref{applemma:convergence-rate-1}]
Firstly, by Lemma 4.1 and 4.2 in \cite{Vandegeer-2000}, we have
\begin{align*}
    \frac{1}{16}h^2(\phlbgn,\plbgs)\leq h^2(\barphlbgn,\plbgs) \leq \frac{1}{\sqrt{n}}\mu_n(\hlbgn),
\end{align*}
here $\mu_n(\hlbgn)$ is an empirical process defined as
\begin{align*}
    \mu_n(\hlbgn):=\sqrt{n}\int_{\plbgs>0}
    \frac{1}{2}\log(\frac{\barphlbgn}{\plbgs})(\barphlbgn-\plbgs)d(X,Y).
\end{align*}
Thus, for any $\delta>\delta_n:=\sqrt{\log n/n}$, we have
\begin{align*}
    &\mathbb{P}_{\lbgs}(h(\phlbgn,\plbgs)\geq\delta)
    \\&\leq\mathbb{P}_{\lbgs}
    \left(
    \mu_n(\hlbgn)-\sqrt{n}h^2(\phlbgn,\plbgs)\geq0,
    h(\phlbgn,\plbgs)\geq\frac{\delta}{4}
    \right)
    \\&\leq\mathbb{P}_{\lbgs}
    \left(
    \sup_{\lbg:h(\bar{p}_{\lbg},\plbgs)\geq\delta/4}
    \left[
    \mu_n(\lbg)-\sqrt{n}h^2(\barplbg,\plbgs)
    \right]\geq0
    \right)
    \\&\leq\sum_{s=0}^S\mathbb{P}_{\lbgs}
    \left(
    \sup_{\lbg:2^s\delta/4\leq h(\bar{p}_{\lbg},\plbgs)\leq 2^{s+1}\delta/4}
    \left|
    \mu_n(\lbg)
    \right|
    \geq\sqrt{n}2^{2s}(\frac{\delta}{4})^2
    \right)
    \\&\leq\sum_{s=0}^S\mathbb{P}_{\lbgs}
    \left(
    \sup_{\lbg: h(\bar{p}_{\lbg},\plbgs)\leq 2^{s+1}\delta/4}
    \left|
    \mu_n(\lbg)
    \right|
    \geq\sqrt{n}2^{2s}(\frac{\delta}{4})^2
    \right)
\end{align*}
where $S$ is a smallest number such that $2^S\delta/4 > 1$
. 

Now we will use Lemma \ref{applemma:convergence-rate-3}: choose $R=2^{s+1}\delta, C_1=15$ and $t=\sqrt{n}2^{2s}(\delta/4)^2$.
We can confirm that condition (i) in Lemma 3 is met since $2^{s-1} \delta / 4 \leq 1$ for all $s \leq S$.
For the condition (ii), it is still satisfied since
\begin{align*}
    &\int^R_{t/2^6\sqrt{n}}
    H_B^{1/2}\left(\frac{u}{\sqrt{2}},\mathcal{P}^{1/2}(\Xi,R),\mu  \right)
    du\vee2^{s+1}\delta
    \\&=\sqrt{2}\int^{R/\sqrt{2}}_{R^2/2^{13}}
    H_B^{1/2}\left({u},\mathcal{P}^{1/2}(\Xi,R),\mu  \right)
     du\vee2^{s+1}\delta
     \\&\leq2\mathcal{J}_B\left(R,\mathcal{P}^{1/2}(\Xi,R),\mu  \right)
     \\&\leq2J\sqrt{n}2^{2s+1}\delta^2
     \\&=2^6Jt.
\end{align*}
Now since the two conditions in Lemma \ref{applemma:convergence-rate-3} are all satisfied, we could conclude that
\begin{align}
    \mathbb{P}_{\lbgs}
    \left( h(\phlbgn,\plbgs)>\delta \right)
    \leq C\sum_{s=0}^{\infty}
    \exp\left( \frac{2^{2s}n\delta^2}{J^22^{14}} \right)
    \leq 
    c\exp\left( -\frac{n\delta^2}{c} \right),
\end{align}
here constant $c$ is a large constant that does not depend on $\Gs$.
Now we could derive the bound on supremum of expectation:
\begin{align*}
    \mathbb{E}_{\plbgs}h({\phlbgn,\plbgs})
    &=\int^{\infty}_{0}\mathbb{P}
    \left( h({\phlbgn,\plbgs})>\delta \right)d\delta
    \\&\leq\delta_n+c\int^{\infty}_{\delta_n}\exp \left(-\frac{n\delta^2}{c^2} \right)d\delta
    \\& \leq \Tilde{c}\delta_n,
\end{align*}
here $\Tilde{c}$ is independent from $\lbgs$ and $\delta_n:=\sqrt{\log n/n}$.
So we can conclude that
\begin{align*}
    \sup_{(\lambdas,\Gs)\in\Xi}
    \mathbb{E}_{p_{\lambdas,\Gs}}
    h(p_{\widehat{\lambda}_n,\widehat{G}_n},p_{\lambdas,\Gs})
    \leq
    C\sqrt{\log n/n}.
\end{align*}
\end{proof}

\begin{proof}[Proof of Lemma \ref{applemma:convergence-rate-add-1}]
Because $\overline{\mathcal{P}}^{1/2}(\Xi,\delta)\subset\overline{\mathcal{P}}^{1/2}(\Xi)$ and from the definition of Hellinger distance, we have
\begin{align*}
    H_B(\delta,\overline{\mathcal{P}}^{1/2}(\Xi,\delta),\mu)
    \leq
    H_B(\delta,\overline{\mathcal{P}}^{1/2}(\Xi),\mu)
    =
    H_B\left(\frac{\delta}{\sqrt{2}},\overline{\mathcal{P}}(\Xi),h\right).
\end{align*}
Now, using the fact that for densities $ f^*, f_1, f_2 $, we have $ h^2 \left( \frac{f_1 + f^*}{2}, \frac{f_2 + f^*}{2} \right) \leq \frac{h^2(f_1, f_2)}{2} $, it is easy to verify that $ H_B(\delta/\sqrt{2}, \overline{\mathcal{P}}(\Xi), h) \leq H_B(\delta, \mathcal{P}(\Xi), h) $.
Hence, if equation (\ref{assumption:A3}) holds true, then
\begin{align*}
   H_B(\delta,\overline{\mathcal{P}}^{1/2}(\Xi,\delta),\mu)
    \leq
   H_B(\delta,{\mathcal{P}}(\Xi),h)
    \lesssim\log\left(\frac{1}{\delta} \right).
\end{align*}
This implies that
\begin{align*}
    \mathcal{J}_B
    \left(
    \epsilon,
    \overline{\mathcal{P}}^{1/2}(\Xi,\delta),\mu
    \right)
    \lesssim
    \epsilon
    \left(
    \log(\frac{2^{13}}{\epsilon^2})
    \right)^{\frac{1}{2}}
    <n\epsilon^2,\quad
\text{for all }
d\epsilon>\sqrt{\displaystyle\frac{\log n}{n}}.
\end{align*}
\end{proof}

\begin{proof}[Proof of Lemma \ref{applemma:convergence-rate-2}]

\textbf{Proof for (i):}
For any set $S$, we denote $\mathcal{E}_\epsilon(S)$ an $\epsilon-$net of $S$ if each element of S is within $\epsilon$ distance from some elements of $\mathcal{E}_\epsilon(S)$. 
So $|\mathcal{E}_\epsilon(S)|=N(\epsilon,S,\Vert\cdot\Vert_\infty)$
by the definition of the covering number.
Denote $\mathcal{P}(\Theta)=\{ p_{G}:G\in\Theta \}$
and
$p_G(X,Y):=f(\ysigmaa)\bar{f}(X)$.
From Lemma 6 in \cite{ho2022gaussian}, we have
\begin{align*}
    \log
    \left|
\mathcal{E}_\epsilon(\mathcal{P}(\Theta))
    \right|
    =N(\epsilon,\mathcal{P}(\Theta),\Vert\cdot\Vert_\infty)
    \lesssim\log(1/\epsilon).
\end{align*}
Indeed, for any $\lambda\in[0,1], G\in\Theta$, there exists $\tlb\in\mathcal{E}_\epsilon([0,1]),\tg\in\mathcal{E}_\epsilon(\mathcal{P}(\Theta))$ s.t. $|\lambda-\tlb|\leq\epsilon$ and $\Vert p_G-p_{\tg} \Vert_\infty\leq\epsilon$.
By triangle inequalities we will have
\begin{align*}
    \Vert \plbg-p_{\tlb,\tg} \Vert_\infty
    \leq
    |\lambda-\tlb|(\Vert f_0 \Vert_\infty + \Vert f \Vert_\infty)
    +\tlb\Vert p_G-p_{\tg} \Vert_\infty
    \lesssim \epsilon.
\end{align*}
Hence we have the conclusion that
\begin{align*}
    \log N(\epsilon,\mathcal{P}(\Xi),\Vert\cdot\Vert_\infty)
    \lesssim
    \log(|\mathcal{E}_\epsilon([0,1])|)
    +
    \log(|\mathcal{E}_\epsilon(\mathcal{P}(\Theta))|)
    \lesssim\log (1/\epsilon).
\end{align*}

\textbf{Proof for (ii):}
First, let $ \eta \leq \varepsilon $ be a positive number, which will be chosen later. We consider $f$ is the density function of an univariate Gaussian distribution, so $f$ is light tail: 
for any $|Y|\geq2a$ and $X\in\mathcal{X}$, 
\begin{align*}
    f(\ysigmaa)\leq\frac{1}{\sqrt{2\pi}\ell}
    \exp\left(-\frac{Y^2}{8u^2} \right).
\end{align*}
Also $f_0$ is bounded with tail 
$
\log f_0(Y|\varphi (a_0^{\top}X+b_0 ),\nu_0)
\lesssim
-Y^q
$
and $f_0(Y|\varphi (a_0^{\top}X+b_0 ),\nu_0)\leq M$
for almost surely $Y\in\mathcal{Y}$
for some $M,q>0$. 
Now let $q=\min\{p,2 \}$ and $C_2=\max\left\{ M,1/{\sqrt{2\pi}\ell} \right\}$, we will have 
\begin{align}
    H(X,Y)=
    \begin{cases}
        C_1\exp(-Y^q)\bar{f}(X),& |Y|\geq2a \\
        C_2\bar{f}(X),&|Y|<2a
    \end{cases}
\end{align}
here $C_1$ is a positive constant depending on $\ell$ and $f_0$. Moreover $H(X,Y)$ is an envelope of $\mathcal{P}(\Xi)$.
Next, let $ g_1, \dots, g_N $ represent an $ \eta $-net over $ \mathcal{P}_k(\Xi) $. Then, we construct the brackets $[p^L_i(X, Y), p^U_i(X, Y)]$ as follows:
\begin{align*}
    \begin{cases}
        p^L_i(X, Y):=\max\{g_i(X,Y)-\eta,0 \}\\
        p^U_i(X, Y):=\min\{g_i(X,Y)+\eta, H(X,Y) \}
    \end{cases}
\end{align*}
for $i=1,\cdots,N$. 
As a result, $ \mathcal{P}_k(\Xi) \subset \bigcup_{i=1}^N [p^L_i(X, Y), p^U_i(X, Y)] $ and $ p^U_i(X, Y) - p^L_i(X, Y) \leq \min\{2\eta, H(X, Y)\} $. Consequently,
\begin{align*}
    &\int\left( p^U_i(X, Y)-p^L_i(X, Y)\right)d(X,Y)\\
    &\leq
    \int_{|Y|<2a}\left( p^U_i(X, Y)-p^L_i(X, Y)\right)d(X,Y)
    +
    \int_{|Y|\geq2a}\left( p^U_i(X, Y)-p^L_i(X, Y)\right)d(X,Y)
    \\&\leq
    \int_{|Y|<2a}2\eta d(X,Y)
    +
    \int_{|Y|\geq2a}H(X,Y)d(X,Y)
    \lesssim\eta.
\end{align*}
This shows that $$ H_B(c\eta, \mathcal{P}(\Xi), \|\cdot\|_1) \leq N \lesssim \log(1/\eta). $$
Setting $ \eta = \epsilon/c $, 
we find $$
H_B(\epsilon, \mathcal{P}(\Xi), \|\cdot\|_1) \lesssim \log(1/\epsilon). $$
Since $ h^2 \leq \|\cdot\|_1 $ holds between the Hellinger distance and the total variation distance, we conclude the bracketing entropy bound.
\end{proof}

\section{PROOFS FOR AUXILIARY RESULTS}
\label{appendix:ProofsforAuxiliaryResults}

\subsection{The derivation of the heat equation}

If $f$ represents a family of univariate Gaussian distributions:
\[
f(Y | \sigma(a^{\top}X + b), \nu) = \frac{1}{\sqrt{2\pi \nu}} \exp\left(-\frac{(Y - \sigma(a^{\top}X + b))^2}{2\nu}\right).
\]
We will have
\begin{align*}
    \frac{\partial^2 f}{\partial \sigma^2} = f(Y | \sigma(a^{\top}X + b), \nu) \cdot \left[ \frac{(Y - \sigma(a^{\top}X + b))^2}{\nu^2} - \frac{1}{\nu} \right],\\
    \frac{\partial f}{\partial \nu} = f(\ysigmaa)\cdot \left[  \frac{(Y - \sigma(a^{\top}X + b))^2}{2\nu^2}-\frac{1}{2\nu} \right].
\end{align*}
So $f(\ysigmaa)$ satisfies the heat partial differential equation (PDE) with respect to the location and covariance parameters:
 \begin{align}
 \label{eq:heat_equation}
     \frac{\partial^2f}{\partial \sigma^2}
     (Y|\sigma(a^{\top}X+b) ,\nu )
     =
     2\frac{\partial f}{\partial \nu}
     (Y|\sigma(a^{\top}X+b) ,\nu 
     ),
 \end{align}
for almost surly $(X,Y)\in\mathcal{X}\times\mathcal{Y}$
and $(a,b,\nu)\in \Theta$.  
In Section \ref{section:sigma-linear-f0-Gaussian-varphi-linear}, we consider $\sigma$ as an identity function, i.e. $\sigma = a^{\top}X + b$.
Thus, we obtain:
   \[
   \frac{\partial \sigma}{\partial b} = 1.
   \]
Then, the second derivative for $f$ respect to $b$ is:
\begin{align*}
   \frac{\partial^2 f}{\partial b^2} 
   = \frac{\partial}{\partial b} \left( \frac{\partial f}{\partial b} \right)
   = \frac{\partial}{\partial b} \left( \frac{\partial f}{\partial \sigma} \cdot \frac{\partial\sigma}{\partial b}\right)
   = \frac{\partial}{\partial b} \left( \frac{\partial f}{\partial \sigma} \cdot 1 \right)
   = \frac{\partial^2 f}{\partial \sigma^2} .
\end{align*}



So we will have
\begin{align*}
\dfrac{\partial^2 f}{\partial b^2} = 2 \dfrac{\partial f}{\partial \nu},
\end{align*}
where the  equation is what we stated in equation 
\eqref{eq:heat_equation_1}.

\subsection{Proof of Proposition~\ref{prop:identifiability}}
\label{appendix:identifiability}

\begin{proof}
Suppose that  $p_{\lambda, G}(Y|X) =p_{\lambda^\prime , G^\prime}(Y|X)$ holds for almost sure every $(X,Y) \in \mathcal{X} \times \mathcal{Y}$:
\begin{equation*}
    \lambda f_0(Y|\varphi(a_0^\top X+b_0),\nu_0) + (1-\lambda) f(Y|\sigma(a^\top X+b),\nu) = \lambda^\prime f_0(Y|\varphi(a_0^\top X+b_0),\nu_0) + (1-\lambda^\prime) f(Y|\sigma((a^\prime)^\top X+b^\prime),\nu^\prime).
\end{equation*}
This means that 
\begin{equation}
\label{eq:prop1_simplified}
    (\lambda^\prime-\lambda)f_0(Y|\varphi(a_0^\top X+b_0),\nu_0) + (1-\lambda^\prime) f(Y|\sigma((a^\prime)^\top X+b^\prime),\nu^\prime) -  (1-\lambda) f(Y|\sigma(a^\top X+b),\nu) = 0.
\end{equation}
If $G=G'$, let $\tilde{G}$ be other parameter in the parameter set. Then, equation \eqref{eq:prop1_simplified} becomes 
\begin{equation*}
    (\lambda^\prime-\lambda)f_0(Y|\varphi(a_0^\top X+b_0),\nu_0) + (\lambda-\lambda^\prime) f(Y|\sigma(a^\top X+b),\nu) + 0 f(Y|\sigma(\tilde{a}^\top X+\tilde{b}),\tilde{\nu}) = 0. 
\end{equation*}
Applying the equation in Definition \ref{def:distinguish} with $\eta_0 = \lambda^\prime-\lambda$, $\eta_1 = \lambda-\lambda^\prime$, and $\eta_2 = 0$, we get that $\lambda = \lambda^\prime$, which means that $(\lambda,G) = (\lambda^\prime,G^\prime)$. 

Otherwise, if $G\neq G'$, then by pplying the equation in Definition \ref{def:distinguish} with $\eta_0 = \lambda^\prime-\lambda$, $\eta_1 = 1-\lambda$, and $\eta_2 = \lambda^\prime-1$ to equation \eqref{eq:prop1_simplified}, we get $\lambda = \lambda^\prime$, and $1-\lambda = 0$. This mean that $\lambda = \lambda^\prime = 1$, which means that the parts $G$ and $G'$ have no contribution. Thus, it is natural to suppose that $(\lambda,G) = (\lambda^\prime,G^\prime)$. 
\end{proof}

\subsection{Proof of Proposition~\ref{lemma:distinguish-linear sigma not Gaussian}}
\label{appendix:lemma:distinguish-linear sigma not Gaussian}
\begin{proof}
Recall the Definition \ref{def:distinguish} for the distinguishibility .
At first if $\eta_0\neq0$, then     
\begin{align*}
    &f_0(Y|\varphi (a_0^{\top}X+b_0 ),\nu_0)
    =-
    \sum_{i=1}^{2}\frac{\eta_i}{\eta_0}f(\ylinearsigmai)
    \end{align*}
    contradicts that $f_0\neq f$, so $\eta_0=0$.
Then since $(a_i,b_i,\nu_i), i=1,2$ are two distinct components, means that the mean and variance are different, then from the property of Gaussian distribution, we will have $\eta_1=\eta_2=0$. So that now $f$ is distinguishable from $f_0$.
\end{proof}

\subsection{Proof of Proposition~\ref{prop:sigma-nonlinear-f0-Gaussian-varphi-nonlinear}}
\label{appendix:prop:sigma-nonlinear-f0-Gaussian-varphi-nonlinear}
    
\begin{proof}
In this proof, we assume technically that the prior density distribution $\bar{f}(X)$ is non-vanishing almost everywhere. This assumption is actually realistic, because adding a quite small independent Gaussian noise can make the density function to be on-vanishing almost everywhere. This assumption is crucial for guarantee that $X$ can attain any value with a sufficient probability. Moreover, we also assume that the prompt function $\sigma$ should not have the form $\pm\log\left(x+c_1+\sqrt{(x+c_1)^2+c_2}\right)+C$. To our knowledge, this expert function has not appeared in common activation function in Deep Learning, which makes our assumption is reasonable. 

Suppose that there exists $\alpha_{\eta}$ such that
\begin{align}
\label{eq:dependence_non_distinguisability}
         \sum_{\ell=0}^2\sum_{|\eta|=\ell}
         \alpha_{\eta}\cdot
         \frac{\partial^{|\eta|}f}{\partial a^{\eta_1}\partial b^{\eta_2}\partial \nu^{\eta_3}}
         ( Y|\sigma(a^{\top}X+b),\nu)=0.
\end{align}

In order to do this exercise, it is necessary to calculate the derivative of order at most $
$ of $f$. The first-order derivative are: 
\begin{equation*}
\label{eq:dependence_non_distinguisability_second_order}
    \dfrac{\partial f}{\partial a_i} = \dfrac{\partial f}{\partial \sigma} \cdot \dfrac{\partial \sigma}{\partial a_i}, \quad \dfrac{\partial f}{\partial b} = \dfrac{\partial f}{\partial \sigma} \cdot \dfrac{\partial \sigma}{\partial b},
\end{equation*}
The second order derivatives are: 
\begin{align*}
    \dfrac{\partial^2 f}{\partial a_i\partial a_j} = \dfrac{\partial^2 f}{\partial \sigma^2} \cdot \dfrac{\partial \sigma}{\partial a_i}\dfrac{\partial \sigma}{\partial a_j} + \dfrac{\partial f}{\partial\sigma} \cdot \dfrac{\partial^2 \sigma}{\partial a_i\partial a_j}, &\quad  \dfrac{\partial^2 f}{\partial b^2} = \dfrac{\partial^2 f}{\partial \sigma^2} \cdot \left(\dfrac{\partial \sigma}{\partial b}\right)^2 +\dfrac{\partial f}{\partial\sigma} \cdot \dfrac{\partial^2 \sigma}{\partial b^2},\\
    \dfrac{\partial^2 f}{\partial a_i\partial b} = \dfrac{\partial^2 f}{\partial \sigma^2} \cdot \dfrac{\partial \sigma}{\partial a_i}\dfrac{\partial \sigma}{\partial b} + \dfrac{\partial f}{\partial\sigma} \cdot \dfrac{\partial^2 \sigma}{\partial a_i\partial b}, &\quad \dfrac{\partial f}{\partial \nu} = \frac{1}{2}\dfrac{\partial^2 f}{\partial \sigma^2}.
\end{align*}
The third order derivatives are: 
\begin{equation*}
    \dfrac{\partial^2 f}{\partial a_i\partial \nu} = 
    \frac{1}{2}\dfrac{\partial^3 f}{\partial \sigma^3} \cdot \dfrac{\partial \sigma}{\partial a_i}, \quad  
    \dfrac{\partial^2 f}{\partial b\partial \nu} = \frac{1}{2}\dfrac{\partial^3 f}{\partial \sigma^3} \cdot \dfrac{\partial \sigma}{\partial b}.
\end{equation*}

The forth order derivatives are: 
\begin{equation*}
   \dfrac{\partial^2 f}{\partial \nu^2} = \frac{1}{4}\dfrac{\partial^4 f}{\partial \sigma^4}.
\end{equation*}
Then, the equation \eqref{eq:dependence_non_distinguisability} can be considered as a linear combination of $\partial^kf/\partial\sigma^k$. According to Lemma \ref{lemma:independent_gaussian}, these coefficients in the linear combination are 0's. We consider these coefficients in decreasing order of the derivative.

\subsubsection*{Order 4:}
There is only one term that contribute to the $\partial^4f/\partial\sigma^4$, that is $\partial^2f/\partial\nu^2$. As the coefficient of $\partial^4f/\partial\sigma^4$ is 0, the coefficient of $\partial^2f/\partial\nu^2$ is also 0. 

\subsubsection*{Order 3:}
The terms contributing to the $\partial^3f/\partial\sigma^3$ are $\dfrac{\partial^2 f}{\partial a_i\partial \nu}$ and $\dfrac{\partial^2 f}{\partial b\partial \nu}$. Let $\rho_i$ be the coefficient of $\dfrac{\partial^2 f}{\partial a_i\partial \nu}$ and $\omega$ be the coefficient of $\dfrac{\partial^2 f}{\partial b\partial \nu}$. Then, by summing up all the coefficient with respect to $\partial^3f/\partial\sigma^3$, we achieve that 
\begin{equation}
    \label{eq:prop3_order3}
    \omega\dfrac{\partial\sigma}{\partial b} + \sum_{i=1}^n \rho_i\dfrac{\partial\sigma}{\partial a_i} = 0 \Leftrightarrow \left(\omega+\rho^{\top}X\right)\sigma'(a^\top X + b) = 0,
\end{equation}
where $\rho = (\rho_1\ldots\rho_n)^\top$ for brevity. As our random vector $X$ has density function that does not vanish almost everywhere in $\mathbb{R}^n$, equation \eqref{eq:prop3_order3} is equivalent to 
\begin{equation*}
\left(\omega+\rho^{\top}x\right)\sigma'(a^\top x + b) = 0
\end{equation*}
almost everywhere in $\mathbb{R}^n$. As $\sigma$ is not linear, its derivative does not equal to 0 in a set $B \subset \mathbb{R}$ of positive measure. As $a \neq 0$, we can find subset $\tilde{B} \subset \mathbb{R}^n$ of positive measure such that $a^\top x + b \in B$ for all $x \in \tilde{B}$, in other words, $\sigma'(a^\top x +b) \neq 0$. Thus, $\omega + \rho^\top x = 0$ for $x \in \tilde{B}$. This equality occurs if and only if $\omega = \rho = 0$. Thus, all the coefficients of $\dfrac{\partial^2 f}{\partial a_i\partial \nu}$ and $\dfrac{\partial^2 f}{\partial b\partial \nu}$ are 0. 

\subsubsection*{Order 2:}
The terms contributing to the $\partial^2f/\partial\sigma^2$ are $\dfrac{\partial^2 f}{\partial a_i\partial a_j}, \dfrac{\partial^2 f}{\partial a_i\partial b}, \dfrac{\partial^2 f}{\partial b^2}, \dfrac{\partial f}{\partial\nu}$. 
Let $\alpha_{i,j}$ be the coefficients of 
$\dfrac{\partial^2 f}{\partial a_i\partial a_j}$, 
$\beta_i$ be the coefficients of 
$\dfrac{\partial^2 f}{\partial a_i\partial b}$, 
$\gamma$ be the coefficient of 
$\dfrac{\partial^2 f}{\partial b^2}$, 
and $\xi$ be the coefficient of $\dfrac{\partial f}{\partial\nu}$. Then, by summing up all the coefficient with respect to $\partial^2f/\partial\sigma^2$, we achieve that
\begin{equation}
\label{eq:prop3_order2}
    \sum_{i,j=1}^n \alpha_{i,j} \dfrac{\partial\sigma}{\partial a_i}\dfrac{\partial\sigma}{\partial a_j} 
    + \sum_{i=1}\beta_i\dfrac{\partial\sigma}{\partial a_i} \dfrac{\partial\sigma}{\partial b} 
    +\gamma
    \left(\dfrac{\partial\sigma}{\partial b}\right)^2
    + \frac{1}{2}\xi = 0 \Leftrightarrow \left(\sum_{i,j=1}^n \alpha_{i,j}X_iX_j + \sum_{i=1}^n\beta_i X_i+\gamma\right)(\sigma^{\prime}(a^\top X + b))^2 + \frac{1}{2}\xi = 0.
\end{equation}

With the hypothesis that $a\neq 0$, by a change of base, we can suppose that $a=(1,0,\ldots,0)^\top$. Then, the equation above can be considred as
\begin{equation*}
    \left(\sum_{i,j=1}^n \alpha_{i,j}X_iX_j + \sum_{i=1}^n\beta_i X_i+\gamma\right)(\sigma^{\prime}(X_1+b))^2 + \frac{1}{2}\xi = 0.
\end{equation*}
As our random vector $X$ has density function that does not vanish almost everywhere in $\mathbb{R}^n$, equation \eqref{eq:prop3_order2} is equivalent to 
\begin{equation}
\label{eq:quadratic}
    \left(\sum_{i,j=1}^n \alpha_{i,j}x_ix_j + \sum_{i=1}^n\beta_i x_i+\gamma\right)(\sigma^{\prime}(x_1+b))^2 + \frac{1}{2}\xi = 0.
\end{equation}
WLOG, we can suppose that $\alpha_{i,j} = \alpha_{j,i}$. By letting $x_2, \ldots, x_n = 0$, we achieve that 
\begin{equation}
\label{eq:quadratic_reduced}
    (\alpha_{1,1}x_1^2 + \beta_1 x_1+\gamma)(\sigma^{\prime}(x_1+b))^2 + \frac{1}{2}\xi = 0.
\end{equation}
By substituting equation \eqref{eq:quadratic_reduced} in equation \eqref{eq:quadratic}, we receive that
\begin{equation*}
    \left(\sum_{i,j\neq 1}^n\alpha_{i,j}x_ix_j + 2\sum_{i=1,j\neq 1}^n\alpha_{i,j}x_ix_j + \sum_{i=2}^n\beta_i\right)(\sigma^{\prime}(x_1+b))^2 = 0.
\end{equation*}
As $\sigma$ is not a linear function, there exists a set $B\subset \mathbb{R}$ of positive measure such that $\sigma^\prime \neq 0$. Thus, it is possible to find a set $B^\prime \subset \mathbb{R}$ such that $\sigma^\prime(x_1+b)\neq 0$.  In this set,
\begin{equation*}
    \sum_{i,j\neq 1}^n\alpha_{i,j}x_ix_j + 2\sum_{i=1,j\neq 1}^n\alpha_{i,j}x_ix_j + \sum_{i=2}^n\beta_i = 0
\end{equation*}
This equation implies that $\alpha_{i,j} = 0$ for $(i,j)\neq (1,1)$, and $\beta_i = 0$ for $i={2,\cdots,n}$. 

Next, we consider the equation \eqref{eq:quadratic_reduced}. Suppose that not all the coefficient $\alpha_{1,1}$, $\beta_1$, and $\gamma$ is equal to 0. We examine it in three different setting: 

\texttt{Case 1: } If $\alpha_{1,1} = 0$. Then, if $\beta_1 \neq 0$, then $(\beta_1x_1+\gamma)(\sigma'(x_1+b))^2+\frac{1}{2}\xi = 0$. Thus, $(\sigma^\prime(x_1+b))^2 = -\xi/2(\beta_1x_1+\gamma) \geq 0$ almost sure, this occurs if and only if $\xi = 0$. Then, $\sigma'(x_1+b) = 0$, which implies that $\sigma$ be a constant, a contradiction to the hypothesis that $\sigma$ is not linear. 

Otherwise, if $\beta_1 = 0$, then $\gamma(\sigma'(x_1+b))^2+\xi/2 = 0$. If $\gamma \neq 0$, then this equation occurs if and only if $\sigma'(x_1+b) = 0$ for all $x_1 \in B'$, which is a contradiction to our hypothesis that $\sigma^\prime(x_1+b)\neq 0$ in $B'$.  If $\gamma=0$, then $\xi = 0$, which is also a contradiction to the hypothesis that not all the coefficient $\alpha_{1,1}$, $\beta_1$, and $\gamma$ is equal to 0.

\texttt{Case 2: } If $\alpha_{1,1} \neq 0$. Then, by letting $y = x_1+b$, $c_1 = \beta_1/(2\alpha_{{1,1}}) - b$, $c_2 = \gamma/\alpha_{1,1} - (\beta_1/(2\alpha))^2$, $c = -\xi/(2\alpha_{1,1})$, equation \eqref{eq:quadratic_reduced} becomes
\begin{equation*}
    \sigma'(y)^2 = \dfrac{c}{(y+c_1)^2 +c_2}. 
\end{equation*}
If $c_2 < 0$, then for $y \in (-c_1-\sqrt{c_2}, -c_1 + \sqrt{c_2})$ we have $\sigma^\prime(y)^2 < 0$, which is not possible. 

If $c_2 = 0$, then $|\sigma'(y)| \to \infty$ as $y \to -c_1$, which is not possible as $\sigma'$ be bounded in a compact set due to its continuous property.

If $c_2 > 0$, then we have the famous formulae $\sigma(y) = \pm\log\left(y+c_1+\sqrt{(y+c_1)^2+c_2}\right)+C$, which is also a contradiction with the hypothesis of $\sigma$

\subsubsection*{Order 1:}
The terms contributing to the $\partial f/\partial\sigma$ are $\dfrac{\partial f}{\partial a_i}, \dfrac{\partial f}{\partial b}, \dfrac{\partial^2 f}{\partial a_i\partial a_j},\dfrac{\partial^2 f}{\partial a_i\partial b},\dfrac{\partial^2 f}{\partial b^2}$. 
Recall we have $\alpha_{i,j}$ as the coefficients of 
$\dfrac{\partial^2 f}{\partial a_i\partial a_j}$, 
$\beta_i$ as the coefficients of 
$\dfrac{\partial^2 f}{\partial a_i\partial b}$, 
$\gamma$ as the coefficient of 
$\dfrac{\partial^2 f}{\partial b^2}$, 
let $\delta_i$ be the coefficient of $\dfrac{\partial f}{\partial a_i}$,
$\eta$ be the coefficient of $\dfrac{\partial f}{\partial b}$
. Then, by summing up all the coefficient with respect to $\partial f/\partial\sigma$, we achieve that
\begin{align*}
    \sum_{i,j=1}^n \alpha_{i,j} \dfrac{\partial^2\sigma}{\partial a_i\partial a_j} 
    + \sum_{i=1}\beta_i\dfrac{\partial^2\sigma}{\partial a_i\partial b}  
    +\gamma
    \dfrac{\partial^2\sigma}{\partial b^2}
    + \sum_{i=1}\delta_i\dfrac{\partial\sigma}{\partial a_i} 
    +\eta\dfrac{\partial\sigma}{\partial b}
    = 0
    \\
    \Leftrightarrow \left(\sum_{i,j=1}^n \alpha_{i,j}X_iX_j + \sum_{i=1}^n\beta_i X_i+\gamma\right)\sigma^{\prime\prime}(a^\top X + b) + 
    \left(\delta^{\top}X+\eta\right)\sigma'(a^\top X + b)  = 0.
\end{align*}
By the proof of above for order 2, all the coefficient of $\dfrac{\partial^2 f}{\partial a_i\partial a_j}, \dfrac{\partial^2 f}{\partial a_i\partial b}, \dfrac{\partial^2 f}{\partial b^2}$
are equal to 0. Thus, we only take into account the terms  $\dfrac{\partial f}{\partial a_i}, \dfrac{\partial f}{\partial b}$. 
Denote $\delta_i$ be the coefficients of $\dfrac{\partial f}{\partial a_i}$, and $\eta$ be the coefficient of $\dfrac{\partial f}{\partial b}$. Then, \begin{equation*}
    \eta\dfrac{\partial\sigma}{\partial b} + \sum_{i=1}^n \delta_i\dfrac{\partial\sigma}{\partial a_i} = 0 \Leftrightarrow \left(\eta+\delta^{\top}X\right)\sigma'(a^\top X + b) = 0,
\end{equation*}
where $\delta = (\delta_1\ldots\delta_n)^\top$ for brevity. Using the same argument as for the case of order 3, we receive that all the coefficient of $\dfrac{\partial f}{\partial a_i}, \dfrac{\partial f}{\partial b}$ are equal to 0. 

\subsubsection*{Order 0:}
The rest term that was not considered is $f$ without derivative. It is obvious that the coefficient of $f$ is equal to 0. 
\end{proof}

\subsection{Linear Independence of Gaussian Derivatives}
\label{appendixlemmaproof:independent_gaussian}
\begin{lemma}
\label{lemma:independent_gaussian}
Let $f(x,\sigma) = \frac{1}{\sqrt{2\pi\nu}}\exp\left(-\frac{(x-\sigma)^2}{2\nu}\right)$ be the Gaussian density/heat kernel. Consider the scalar (with respect to $x$) $a_0,\ldots,a_n$ such that 
\begin{equation}
\label{eq:depedence_gaussian_before_fourier}
    \sum_{i=0}^n a_i\dfrac{\partial^n f}{\partial \sigma^n}(x,\sigma) = 0,
\end{equation}
for almost every $x,\sigma$. Then, $a_i=0$ for $i = {0,\ldots,n}$.
\end{lemma}
  
\begin{proof}
   By taking the Fourier transformation with respect to $\sigma$ in both side of equation \eqref{eq:depedence_gaussian_before_fourier}, we have
    \begin{equation*}
        \left(\sum_{k=0}^n a_k(i\hat{\sigma})^k\right)\hat{f}(x,\hat{\sigma}) = 0
    \end{equation*}
    for almost every $x,\hat{\sigma}$, where $\hat{f}$ be the Fourier transformation of $f$, which is also a Gaussian density/heat kernel. Thus, we have 
\begin{equation*}
    \sum_{k=0}^n a_k(i\hat{\sigma})^k = 0
\end{equation*}
for almost every $\hat{\sigma}$, which means that $a_k = 0$, $0\leq k\leq n$. 
\end{proof}

\section{ADDITIONAL EXPERIMENTS}
\label{appendix:DiscussionandAdditionalExperiments}



\textbf{Synthetic data generation.} 
In this section, we perform numerical experiments to validate the theoretical results of Theorem~\ref{theorem:sigma-linear-f0-Gaussian-varphi-nonlinear},
Theorem~\ref{theorem:sigma-nonlinear-f0-notGaussian},
and Theorem~\ref{theorem:sigma-nonlinear-f0-Gaussian-varphi-linear} in Appendix~\ref{appendix:678} as they share the same convergence behavior of parameter estimation. Then, we continue to empirically justify the results of Theorem~\ref{theorem:sigma-linear-f0-Gaussian-varphi-linear} and Theorem~\ref{theorem:sigma-nonlinear-f0-Gaussian-varphi-nonlinear} in Appendix~\ref{appendix:4} and Appendix~\ref{appendix:9}, respectively.


\subsection{Theorems \ref{theorem:sigma-linear-f0-Gaussian-varphi-nonlinear}, \ref{theorem:sigma-nonlinear-f0-notGaussian} and \ref{theorem:sigma-nonlinear-f0-Gaussian-varphi-linear}}
\label{appendix:678}
Let us begin with the experimental details for Theorems~\ref{theorem:sigma-linear-f0-Gaussian-varphi-nonlinear}, \ref{theorem:sigma-nonlinear-f0-notGaussian} and \ref{theorem:sigma-nonlinear-f0-Gaussian-varphi-linear}.

\textbf{Experimental setup.}
\begin{itemize}
    \item In Theorem~\ref{theorem:sigma-linear-f0-Gaussian-varphi-nonlinear}, $f_0$ belongs to the Gaussian density family, $\varphi$ is a non-linear function, we let $\varphi$ be the sigmoid function, characterized by $\varphi(a_0^{\top} X + b_0)=1/(1+\exp\{-(a_0^{\top} X + b_0)\})$.
Meanwhile, the prompt $f$ is formulated as a Gaussian density with mean $\sigma((\as)^\top X + \bs) = (\as)^\top X + \bs$ and variance $\nu^*$. 
    \item In Theorem~\ref{theorem:sigma-nonlinear-f0-notGaussian}, since the
pre-trained model $f_0$ does not belong to the Gaussian density family, we let $f_0$ be the density of a Student's t-distribution, characterized by mean $\varphi(a_0^{\top} X + b_0)=a_0^{\top} X + b_0$ and degrees of freedom $\nu_0 = 4$.
Meanwhile, the prompt $f$ is modeled as a Gaussian distribution with a non-linear sigmoid mean, given by $\sigma((\as)^{\top} X + \bs) = {1}/(1 + \exp\{-((\as)^{\top} X + \bs)\})$, and a variance of $\nu^*$.
    \item In Theorem~\ref{theorem:sigma-nonlinear-f0-Gaussian-varphi-linear}, $f_0$ belongs to the Gaussian density family, $\varphi$ is a linear function
but $\sigma$ is a non-linear function.
We let $\varphi$ be the identity function and  $\sigma$ be the sigmoid function.
Specifically, the prompt $f$ is formulated as a Gaussian density with mean 
$\sigma((\as)^{\top} X + \bs)=1/(1+\exp\{-((\as)^{\top} X + \bs)\})$
and variance $\nu^*$. 
\end{itemize}



\textbf{True parameters.} We consider two cases for the true parameters $(\lambdas, \Gs) = (\lambdas, \as, \bs, \nus)$ to examine the difference in the MLE convergence behavior when $\lambdas$ is fixed versus when it varies with $n$:

\begin{enumerate}[(i)]
    \item $\lambdas=0.5,\as=\mathbbm{1}_d,\bs=1,\nus=0.01$;
    \item $\lambdas=0.5~n^{-1/4},\as=\mathbbm{1}_d,\bs=1,\nus=0.01$.
\end{enumerate}



\textbf{Results.}
We display the empirical illustration for Theorems~\ref{theorem:sigma-linear-f0-Gaussian-varphi-nonlinear}, \ref{theorem:sigma-nonlinear-f0-notGaussian} and \ref{theorem:sigma-nonlinear-f0-Gaussian-varphi-linear} in Figures~\ref{fig:thm6_experiments}, \ref{fig:thm7_experiments} and \ref{fig:thm8_experiments}, respectively.
Although these three scenarios differ significantly, they share the same property that the pre-trained model $f$ is distinguishable from the prompt $f_0$. As a result, the convergence behavior of parameter estimation under the settings of those theorems are quite similar.
First, we observe that the convergence rates of $\hlambdan$ in both cases across all three scenarios are approximately $\mathcal{O}(n^{-1/2})$.
However, there are mismatches in the convergence rates of $(\han, \hbn, \hnun)$ between the two cases. In case (i), where $\lambdas$ is fixed, $(\han, \hbn, \hnun)$ exhibit similar rates, approximately $\mathcal{O}(n^{-1/2})$, as seen in Figures~\ref{fig:thm6-fixed},~\ref{fig:thm7-fixed} and~\ref{fig:thm8-fixed}.
In case (ii), where $\lambdas$ vanishes at the rate of $\mathcal{O}(n^{-1/4})$, their rates slow down significantly to around $\mathcal{O}(n^{-1/4})$, as shown in Figures~\ref{fig:thm6-var},~\ref{fig:thm7-var} and \ref{fig:thm8-var}. 
These empirical results are fully consistent with the theoretical findings presented in the three theorems, which say that
\begin{align*}
    \sup_{(\lambdas,G_*)\in\Xi}\mathbb{E}_{p_{\lambdas,\Gs}} 
    \Big[
    |\widehat{\lambda}_n
    -\lambdas|^2 
    \Big] 
    =\widetilde{\mathcal{O}}(n^{-1}),\\
    \sup_{(\lambdas,G_*)\in\Xi}\mathbb{E}_{p_{\lambdas,\Gs}} 
    \Big[
    (\lambdas)^2 
    \Vert 
    (\widehat{a}_n, \widehat{b}_n, \widehat{\nu}_n)-(\as,\bs,\nus) 
    \Vert^2 
    \Big] 
    =\widetilde{\mathcal{O}}(n^{-1}).
\end{align*}

\subsection{Theorem~\ref{theorem:sigma-linear-f0-Gaussian-varphi-linear}}
\label{appendix:4}
Subsequently, we perform numerical experiments to validate theoretical results presented in Theorem~\ref{theorem:sigma-linear-f0-Gaussian-varphi-linear} where both the pre-trained model $f$ and the prompt $f_0$ belong to the family of Gaussian densities, while the experts $\varphi$ and $\sigma$ are the same linear functions. Here, we let $\varphi$ and $\sigma$ be identity functions.

\textbf{True parameters.} We consider two following cases of the true parameters $(\lambdas, \Gs) = (\lambdas, \as, \bs, \nus)$ to capture the difference in the MLE convergence behavior when $\as\to a_0$ and $\nus\to\nu_0$ versus when $\bs\to b_0$ as $n\to\infty$:

\begin{enumerate}[(i)]
    \item $\lambdas=0.5,\as=(1 + n^{-1/8})\cdot e_1,\bs=0,\nus=0.01 + n^{-1/8}$;
    \item $\lambdas=0.5,\as = e_1,\bs=n^{-1/8},\nus=0.01$.
\end{enumerate}



\textbf{Results.} We exhibit the empirical results of Theorem~\ref{theorem:sigma-linear-f0-Gaussian-varphi-linear} under the above two cases in
Figure~\ref{fig:thm4_experiments}. Prior to analyzing them, let us recall the theoretical rates in Theorem~\ref{theorem:sigma-linear-f0-Gaussian-varphi-linear}:
\begin{align*}
    \sup_{(\lambdas,\Gs)\in \Xi_1(l_n)  }
    \mathbb{E}_{p_{\lambdas,\Gs}} \Big[ 
    (\Vert\das\Vert^4+|\dbs|^8+|\dnus|^4) |\widehat{\lambda}_n
    -\lambdas|^2 \Big] 
    &=\widetilde{\mathcal{O}}(n^{-1}),
    \\
    \sup_{(\lambdas,\Gs)\in \Xi_1(l_n) }
    \mathbb{E}_{p_{\lambdas,\Gs}} 
    \Big[ (\lambdas)^2 
    (\Vert\das\Vert^2+|\dbs|^4+|\dnus|^2)
    (\Vert\han-\as\Vert^2+|\hbn-\bs|^4+|\hnun-\nus|^2)
    \Big] 
    &=\widetilde{\mathcal{O}}(n^{-1}).
\end{align*}
(i) In the first case when $\lambdas,\bs$ are fixed while $\as\to a_0$ and $\nus\to\nu_0$ at the rate of $\mathcal{O}(n^{-1/8})$, the first bound indicates that $\hlambdan$ converges to $\lambdas$ at the theoretical rate of $\widetilde{\mathcal{O}}(n^{-1/4})$. Meanwhile, the convergence rates of $\han,\hbn$ and $\hnun$ should be $\widetilde{\mathcal{O}}(n^{-3/8})$, $\widetilde{\mathcal{O}}(n^{-3/16})$ and $\widetilde{\mathcal{O}}(n^{-3/8})$, respectively. Now, looking at Figure~\ref{fig:thm4-case1}, we observe that the empirical convergence rates $\mathcal{O}(n^{-0.22})$, $\mathcal{O}(n^{-0.35})$, $\mathcal{O}(n^{-0.15})$ and $\mathcal{O}(n^{-0.38})$ of $\hlambdan,\han,\hbn$ and $\hnun$ totally align with the theory.

(ii) In the second case when $\lambdas,\as,\nus$ are fixed while $\bs\to b_0$ at the rate of $\mathcal{O}(n^{-1/8})$, the first bound implies that $\hlambdan$ might not converge to $\lambdas$. However, in case the convergence occurs, the rate would be significantly slow. This property is empirically illustrated in Figure~\ref{fig:thm4-case2}. Additionally, we also see from that figure that the empirical convergence rates of $\han,\hbn,\hnun$ totally match their theoretical counterparts, specifically $\mathcal{O}(n^{-0.24})$, $\mathcal{O}(n^{-0.12})$, $\mathcal{O}(n^{-0.25})$ compared to $\widetilde{\mathcal{O}}(n^{-1/4})$, $\widetilde{\mathcal{O}}(n^{-1/8})$, $\widetilde{\mathcal{O}}(n^{-1/4})$.

(iii) It can be seen from the above observations that the convergence of $\bs$ to $b_0$ causes more severe effects on the convergence rates of parameter estimation than those induced by the convergences of $\as,\nus$ to $a_0,\nu_0$. In particular, the former makes the mixing proportion $\hlambdan$ possibly not converge to the true value, while the convergence rates of $\han,\hbn,\hnun$ are slower than those generated from the latter.

\subsection{Theorem~\ref{theorem:sigma-nonlinear-f0-Gaussian-varphi-nonlinear}}
\label{appendix:9}
Lastly, we present the empirical results for Theorem~\ref{theorem:sigma-nonlinear-f0-Gaussian-varphi-nonlinear} where both $f$ and $f_0$ belong to the Gaussian density family, $\varphi$ and $\sigma$ are the same nonlinear expert functions. Here, we let $\varphi$ and $\sigma$ be the sigmoid function, that is, 
\begin{align*}
    \varphi(z)=\sigma(z)=\frac{1}{1+\exp(-z)},
\end{align*}
for any $z\in\mathbb{R}$.


\textbf{True parameters.} We take two following cases of the true parameters $(\lambdas, \Gs) = (\lambdas, \as, \bs, \nus)$ into account to understand the effects of the convergence of $\as,\bs,\nus$ to $a_0,b_0,\nu_0$ on the parameter estimation rates:

(i) $\lambdas=0.5\cdot n^{-1/4},\as=e_1,\bs=0,\nus=0.01$;

(ii) $\lambdas=0.5,\as=(1 + n^{-1/4}),\bs=n^{-1/4},\nus=0.01 + n^{-1/4}$.

\textbf{Results.}
Figure~\ref{fig:thm9_experiments} empirically illustrates the parameter estimation rates presented in  Theorem~\ref{theorem:sigma-nonlinear-f0-Gaussian-varphi-nonlinear}, which are given by
\begin{align*}
    \sup_{(\lambdas,G_*)\in\Xi_{2}(l_n)}\mathbb{E}_{p_{\lambdas,\Gs}} 
    \Big[
    \Vert (\das,\dbs,\dnus) \Vert^4
    |\widehat{\lambda}_n
    -\lambdas|^2 
    \Big] &=\widetilde{\mathcal{O}}(n^{-1}),\\
    \sup_{(\lambdas,G_*)\in\Xi_{2}(l_n)}\mathbb{E}_{p_{\lambdas,\Gs}} 
    \Big[
    (\lambdas)^2 
    \Vert (\das,\dbs,\dnus) \Vert^2
    \Vert (\widehat{a}_n, \widehat{b}_n, \widehat{\nu}_n)-(\as,\bs,\nus) \Vert^2 
    \Big] 
    &=\widetilde{\mathcal{O}}(n^{-1}).
\end{align*}
In case (i) when $\as,\bs,\nus$ are fixed,
Figure~\ref{fig:thm9-case1} shows that the convergence rate of $\hlambdan$ to $\lambdas$ is of order $\mathcal{O}(n^{-0.52})$, while those for $\han,\hbn,\hnun$ are slower, standing at around $\mathcal{O}(n^{-0.25})$, as they hinges upon the vanishing rate $\mathcal{O}(n^{-1/4})$ of $\lambdas$. These empirical rates are in line with the above theoretical rates.

In case (ii) when $\as,\bs,\nus$ converge to $a_0,b_0,\nu_0$ at the rate of $\mathcal{O}(n^{-1/4})$, respectively, Figure~\ref{fig:thm9-case2} reveals that $\hlambdan$ might not converge to $\lambdas$ as the displayed rate is significantly slow of order $\mathcal{O}(n^{-0.05})$. This observation totally aligns with the first bound. On the other hand, the MLE $\han,\hbn,\hnun$ still empirically converges to $\as,\bs,\nus$ at the rates of $\mathcal{O}(n^{-0.22})$, $\mathcal{O}(n^{-0.26})$, $\mathcal{O}(n^{-0.22})$, which approximates the theoretical rates $\widetilde{\mathcal{O}}(n^{-1/4})$ very well.

    \begin{figure*}[h]
        \centering
        \subfloat[\textbf{Case (i):} $\lambda^* = 0.5$.\label{fig:thm6-fixed}]{
            \includegraphics[width=0.9\textwidth]{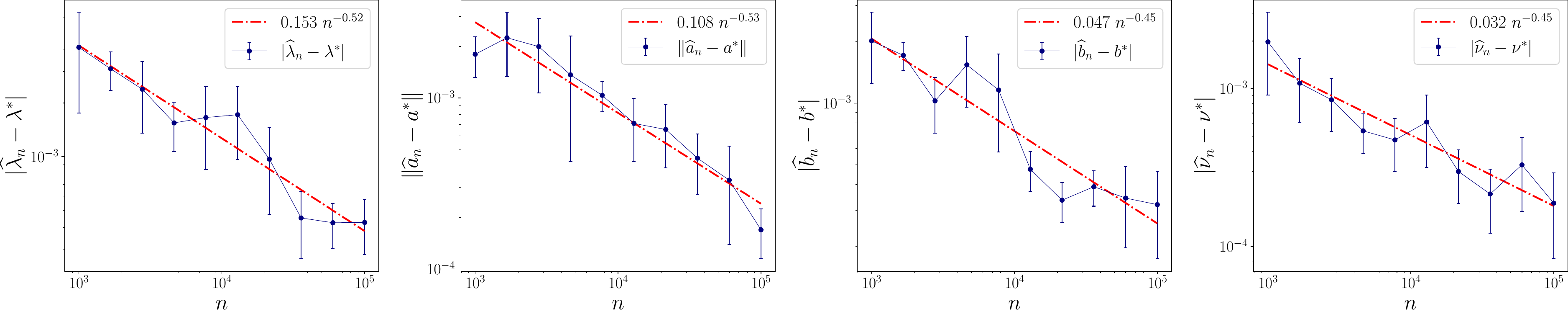}
        }
        
        \subfloat[\textbf{Case (ii):} $\lambda^* = 0.5~n^{-1/4}$.\label{fig:thm6-var}]{
            \includegraphics[width=0.9\textwidth]{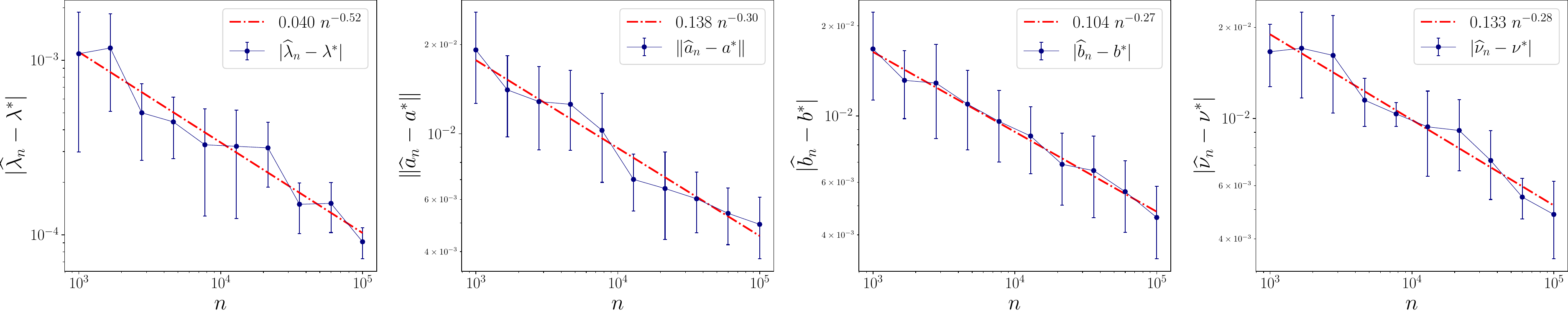}
        }
        \caption{
        {(\textbf{Theorem~\ref{theorem:sigma-linear-f0-Gaussian-varphi-nonlinear}:} $f_0$ is Gaussian, $\sigma(z)=z$, 
        $\varphi(z)=1/(1+e^{-z})$.)} 
        Log-log graphs depicting the empirical convergence rates of the MLE $(\hlambdan,\han,\hbn,\hnun)$ to the ground-truth values $(\lambdas,\as,\bs,\nus)$. 
        The blue lines display the parameter estimation errors, while the red dashed dotted lines are the fitted lines, highlighting the empirical MLE convergence rates. Figure~\ref{fig:thm6-fixed} and Figure~\ref{fig:thm6-var} illustrates results for the cases when $\lambdas = 0.5$ and $\lambdas = 0.5~n^{-1/4}$, respectively.
        }
        \label{fig:thm6_experiments}
    \end{figure*}

    
    \begin{figure*}[h]
        \centering
        \subfloat[\textbf{Case (i):} $\lambda^* = 0.5$.\label{fig:thm7-fixed}]{
            \includegraphics[width=0.9\textwidth]{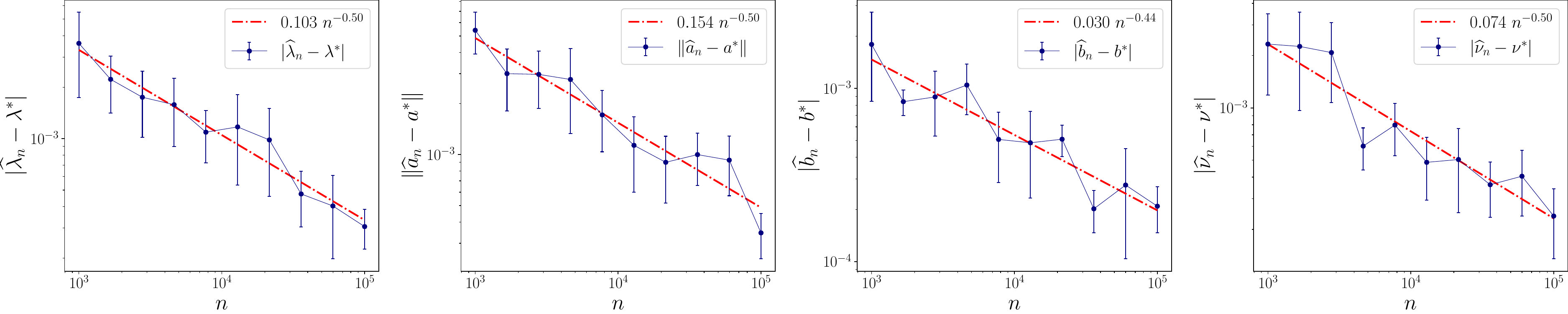}
        }
        
        \subfloat[\textbf{Case (ii):} $\lambda^* = 0.5~n^{-1/4}$.\label{fig:thm7-var}]{
            \includegraphics[width=0.9\textwidth]{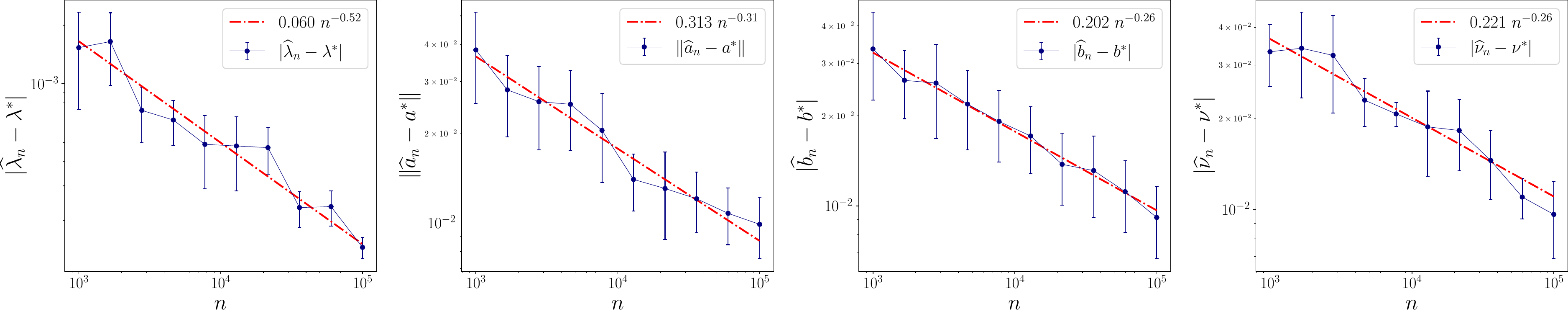}
        }
        \caption{{(\textbf{Theorem~\ref{theorem:sigma-nonlinear-f0-notGaussian}:} $f_0$ is Student's t-distribution, $\sigma(z)=1/(1+e^{-z})$, $\varphi(z)=z$.
        )} 
        Log-log graphs depicting the empirical convergence rates of the MLE $(\hlambdan,\han,\hbn,\hnun)$ to the ground-truth values $(\lambdas,\as,\bs,\nus)$. 
        The blue lines display the parameter estimation errors, while the red dashed dotted lines are the fitted lines, highlighting the empirical MLE convergence rates. Figure~\ref{fig:thm7-fixed} and Figure~\ref{fig:thm7-var} illustrates results for the cases when $\lambdas = 0.5$ and $\lambdas = 0.5~n^{-1/4}$, respectively.
        }
        \label{fig:thm7_experiments}
    \end{figure*}


    \begin{figure*}[h]
        \centering
        \subfloat[\textbf{Case (i):} $\lambda^* = 0.5$.\label{fig:thm8-fixed}]{
            \includegraphics[width=0.9\textwidth]{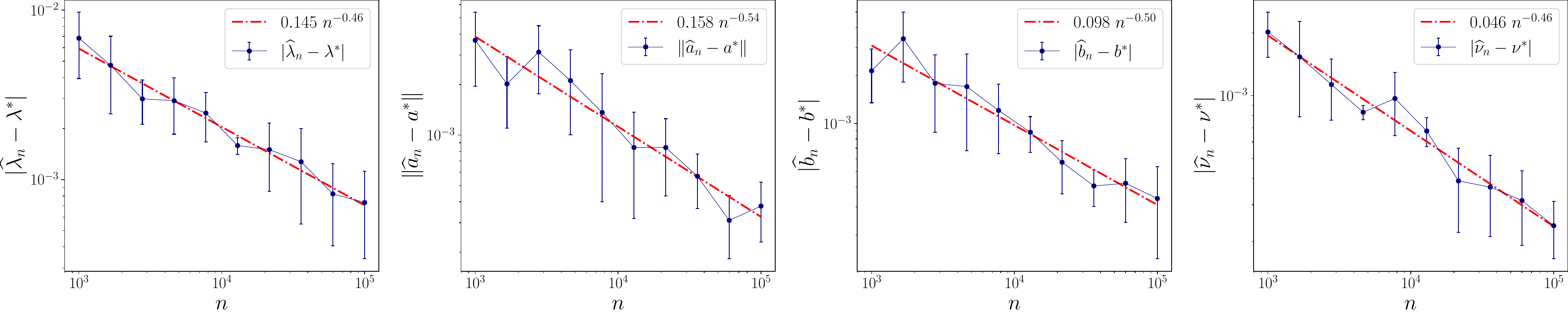}
        }
        
        \subfloat[\textbf{Case (ii):} $\lambda^* = 0.5~n^{-1/4}$.\label{fig:thm8-var}]{
            \includegraphics[width=0.9\textwidth]{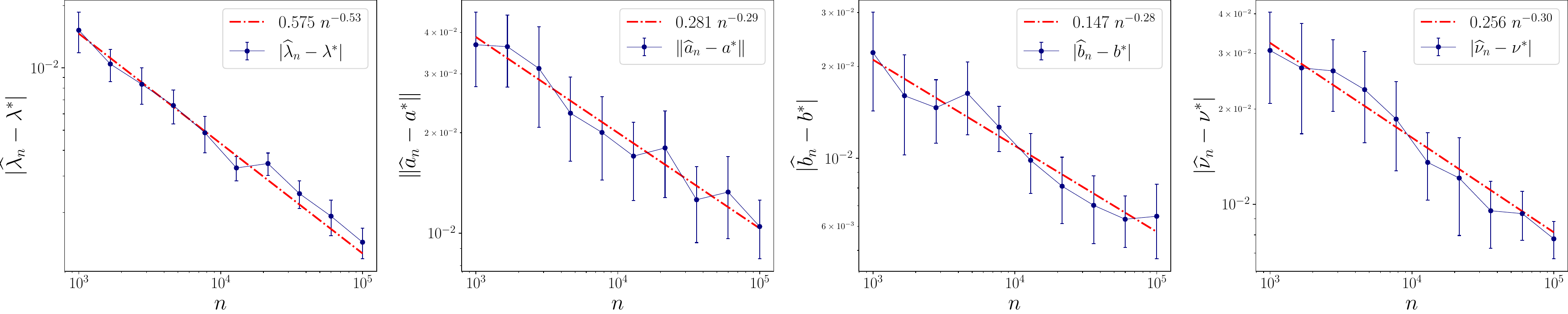}
        }
        \caption{{(\textbf{Theorem~\ref{theorem:sigma-nonlinear-f0-Gaussian-varphi-linear}:} $f_0$ is Gaussian, $\sigma(z)=1/(1+e^{-z})$,  $\varphi(z)=z$.
        )} Log-log graphs depicting the empirical convergence rates of the MLE $(\hlambdan,\han,\hbn,\hnun)$ to the ground-truth values $(\lambdas,\as,\bs,\nus)$. 
        The blue lines display the parameter estimation errors, while the red dashed dotted lines are the fitted lines, highlighting the empirical MLE convergence rates. Figure~\ref{fig:thm8-fixed} and Figure~\ref{fig:thm8-var} illustrates results for the cases when $\lambdas = 0.5$ and $\lambdas = 0.5~n^{-1/4}$, respectively.
        }
        \label{fig:thm8_experiments}
    \end{figure*}


    \begin{figure*}[h]
        \centering
        \subfloat[\textbf{Case (i):} $\lambdas = 0.5, \as = (1 + n^{-1/8})e_1, \bs = 0, \nus = 0.01 + n^{-1/8}$.
        \label{fig:thm4-case1}]{
            \includegraphics[width=0.9\textwidth]{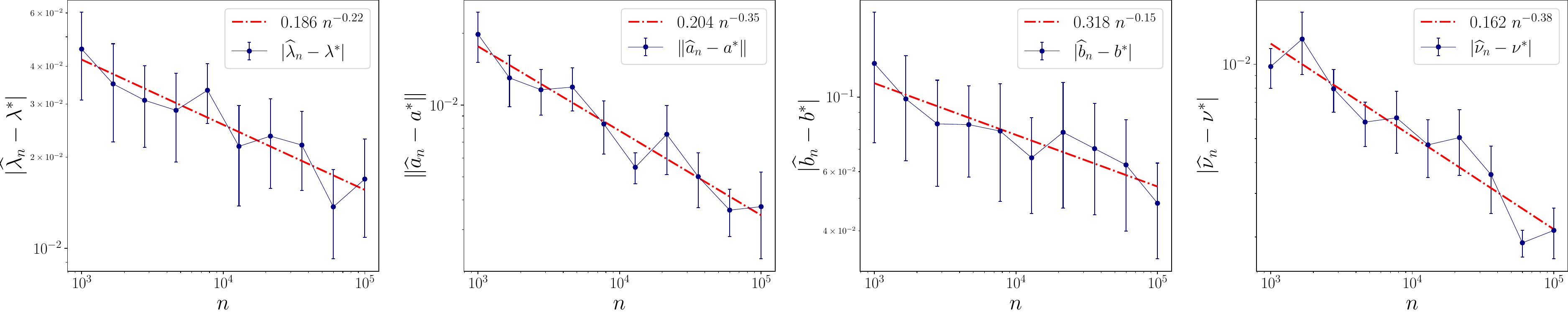}
        }
        
        \subfloat[\textbf{Case (ii):} $\lambdas = 0.5, \as = e_1, \bs = n^{-1/8}, \nus = 0.01$.\label{fig:thm4-case2}]{
            \includegraphics[width=0.9\textwidth]{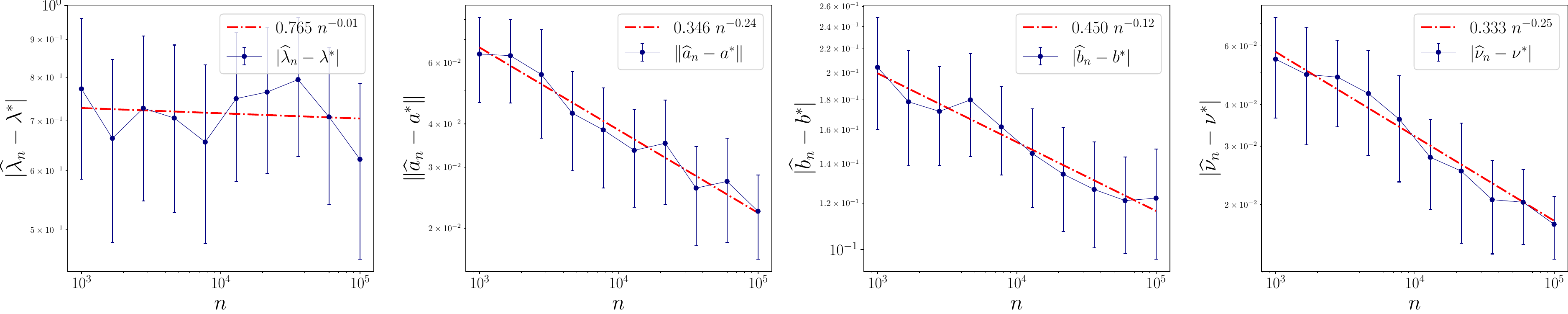}
        }
        \caption{
        {(\textbf{Theorem~\ref{theorem:sigma-linear-f0-Gaussian-varphi-linear}:} $f_0$ is Gaussian, $\sigma(z)=\varphi(z)=z$.)}
        Log-log graphs depicting the empirical convergence rates of the MLE $(\hlambdan,\han,\hbn,\hnun)$ to the ground-truth values $(\lambdas,\as,\bs,\nus)$. 
        The blue lines display the parameter estimation errors, while the red dashed dotted lines are the fitted lines, highlighting the empirical MLE convergence rates. Figure~\ref{fig:thm4-case1} and Figure~\ref{fig:thm4-case2} illustrates results for Case (i) and Case (ii), respectively.
        }
        \label{fig:thm4_experiments}
    \end{figure*}


    \begin{figure*}[h]
        \centering
        \subfloat[\textbf{Case (i):} $\lambdas = 0.5~n^{-1/4}, \as = e_1, \bs = 0, \nus = 0.01$.
        \label{fig:thm9-case1}]{
            \includegraphics[width=0.9\textwidth]{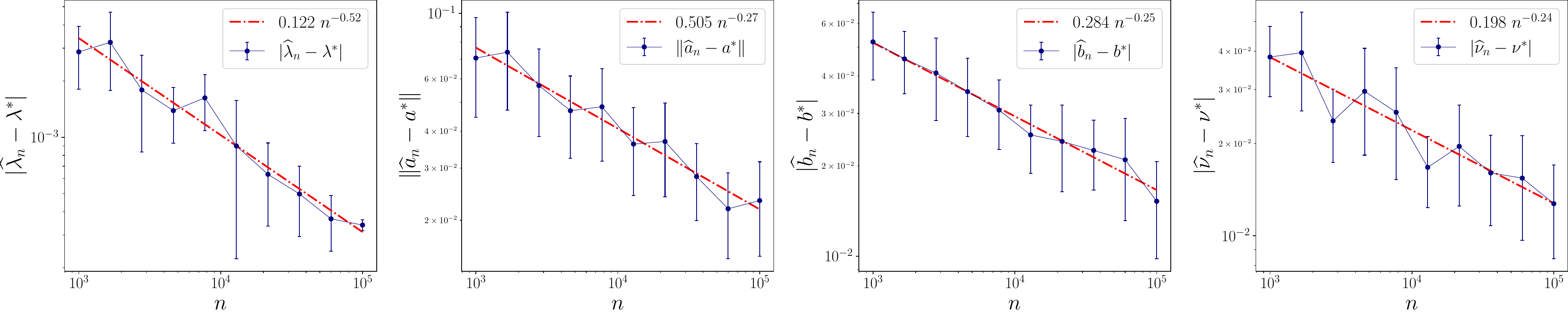}
        }
        
        \subfloat[\textbf{Case (ii):} $\lambdas = 0.5, \as = (1+n^{-1/4})~e_1, \bs = n^{-1/4}, \nus = 0.01 + n^{-1/4}$.\label{fig:thm9-case2}]{
            \includegraphics[width=0.9\textwidth]{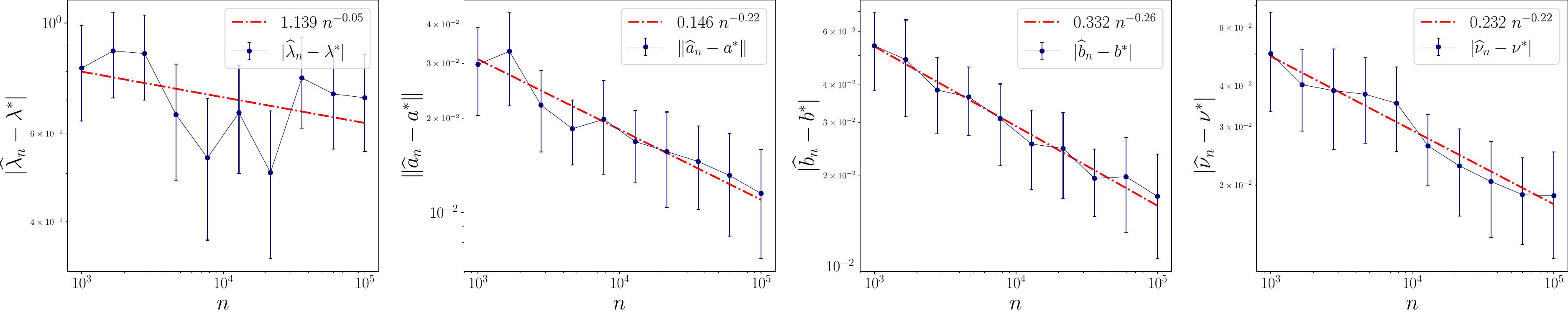}
        }
        \caption{{(\textbf{Theorem~\ref{theorem:sigma-nonlinear-f0-Gaussian-varphi-nonlinear}:} $f_0$ is Gaussian, $\sigma(z)=\varphi(z)=1/(1+e^{-z})$.
        )} Log-log graphs depicting the empirical convergence rates of the MLE $(\hlambdan,\han,\hbn,\hnun)$ to the ground-truth values $(\lambdas,\as,\bs,\nus)$. 
        The blue lines display the parameter estimation errors, while the red dashed dotted lines are the fitted lines, highlighting the empirical MLE convergence rates. Figure~\ref{fig:thm9-case1} and Figure~\ref{fig:thm9-case2} illustrates results for Case (i) and Case (ii), respectively.
        }
        \label{fig:thm9_experiments}
    \end{figure*}







\end{document}